\documentclass{article} 
\usepackage{iclr2027_conference,times}

\usepackage[utf8]{inputenc} 
\usepackage[T1]{fontenc}    
\usepackage{hyperref}       
\usepackage{url}            
\usepackage{booktabs}       
\usepackage{amsfonts}       
\usepackage{nicefrac}       
\usepackage{microtype}      

\usepackage{amsmath}
\usepackage{amssymb}
\usepackage{mathtools}
\usepackage{amsfonts}
\usepackage{amsthm}
\usepackage{dsfont}

\usepackage{multirow}
\usepackage{array}
\usepackage{makecell}
\usepackage[table,dvipsnames]{xcolor}
\usepackage{wrapfig}
\usepackage{colortbl}
\usepackage{subcaption}

\usepackage{graphicx}
\usepackage{graphbox}

\usepackage[inline]{enumitem}
\newlist{todolist}{itemize}{2}
\setlist[todolist]{label=$\square$}
\usepackage{pifont}
\usepackage{algorithm}
\usepackage{algpseudocode}

\usepackage{xcolor}

\usepackage{minitoc}

\newtheorem{theorem}{Theorem}[section]
\newtheorem{lemma}[theorem]{Lemma}
\newtheorem{corollary}[theorem]{Corollary}
\newtheorem{problem}[theorem]{Problem}
\theoremstyle{definition}
\newtheorem{defn}{Definition}[section]

\newcommand\numberthis{\addtocounter{equation}{1}\tag{\theequation}}

\title{CACFG: Curvature-Aware Classifier-Free Guidance and Optimal Control}

\author{Max Collins\\
    School of Physics, Maths and Computing\\
    The University of Western Australia\\
    Perth, WA 6009\\
    \texttt{max.collins@research.uwa.edu.au}\\
    \And
    Dasith de Silva Edirimuni\\
    School of Physics, Maths and Computing\\
    The University of Western Australia\\
    Perth, WA 6009\\
    \texttt{dasith.desilva@uwa.edu.au}\\
    \And
    Jordan Vice\\
    School of Physics, Maths and Computing\\
    The University of Western Australia\\
    Perth, WA 6009\\
    \texttt{jordan.vice@uwa.edu.au}\\
    \And
    Tim French\\
    School of Physics, Maths and Computing\\
    The University of Western Australia\\
    Perth, WA 6009\\
    \texttt{tim.french@uwa.edu.au}\\
    \And
    Ajmal Mian\\
    School of Physics, Maths and Computing\\
    The University of Western Australia\\
    Perth, WA 6009\\
    \texttt{ajmal.mian@uwa.edu.au}\\
}

\iclrfinalcopy 
\begin{document}
\doparttoc 
\faketableofcontents 

\part{} 

\maketitle

\begin{abstract}
    Diffusion models generate samples by learning to reverse a fixed corruption process, and classifier-free guidance (CFG) is the standard mechanism for conditioning this process on a desired class or prompt. CFG can be applied at varying guidance strengths, and while higher strengths improve image quality and conditional alignment, too high a guidance strength can degrade image quality and diversity. Furthermore, CFG violates principled diffusion sampling dynamics, and existing explanations for why it works despite the violation disagree on the underlying theory or do not extend to deterministic samplers used in practice. We address both these issues. We first frame CFG sampling as a continuous-time optimal control problem, treating the sampling trajectory as a sequence of controls chosen to maximise the probability of the desired condition. Solving the resulting Hamilton--Jacobi--Bellman equation shows that CFG is recovered under specific path costs when using an unconstrained control set. We argue this lack of constraint is responsible for CFG's failure at high guidance strengths, since it permits the sampling path to move arbitrarily far from the current image estimate. To fix this, we propose curvature-aware CFG (CACFG), which constrains the control set to a hypersphere informed by the Gaussian regularisation used when training variational autoencoders. We show that the control inputs produced by CFG sampling routinely violate this bound, and that across diffusion models, datasets, and guidance schedules, CACFG achieves superior generative quality at mid-to-high guidance strengths with a less severe quality-diversity tradeoff than regular CFG.
\end{abstract}

\section{Introduction}
Diffusion models generate samples by learning to iteratively reverse a fixed process that gradually corrupts data with noise \citep{DenoisingDiffusion_Ho2020, Song2021}. \emph{Classifier-free guidance (CFG)} \citep{ClassifierFreeGuidance_Ho2022} provides a mechanism to control the content of generated images by using a variable guidance strength parameter to target a tilted conditional distribution. However, this tilted distribution is not the marginal of any true diffusion process, and excessive CFG guidance can reduce image diversity and degrade image quality by introducing visual artefacts \citep{APG_Sadat2025, Karras2024}. Despite this, CFG is frequently used in practice, and practitioners select the largest guidance strengths that do not harm image quality \citep{ClassifierGuidance_Dhariwal2021, ClassifierFreeGuidance_Ho2022}.

Previous work has proposed methods of fixing CFG artefacts: \emph{CFG++} \citep{CFG++_Chung2025} and \emph{Interval-CFG (IntCFG)} \citep{IntCFG_Kynkaanniemi2024} each attribute CFG's artefacts to the sampling trajectory departing the image manifold, correcting for this by varying guidance strength across the trajectory. A separate line of work asks why CFG works \emph{despite} violating diffusion sampling dynamics. Hypotheses attribute its effectiveness to a shared bias in the conditional and unconditional denoisers \citep{Karras2024}, an implicit correction for the poorly-approximated tilted distribution \citep{RenyiCFG_Moufad2025}, or an equivalence with a predictor-corrector scheme in the continuous-time SDE limit \citep{CFGPredCorrect_Bradley2025}. However, these explanations disagree on the mechanism responsible, and they do not rigorously extend to commonly used deterministic samplers. 

We address this gap by reformulating CFG sampling as a continuous-time optimal control problem: rather than fixing the sampling trajectory as a function of stochastic transition kernels, we treat it as a sequence of controls chosen to maximise the terminal reward $p(\xi \mid \boldsymbol{x}_0)$, i.e., how probable it is that the sample is of the desired condition/class $\xi$. Solving the resulting \emph{Hamilton--Jacobi--Bellman (HJB) equation} shows that the optimal control recovers CFG when the control set is unconstrained and the traversal penalty is expressed in the same scale as the diffusion training loss. To the best of our knowledge, we are the first to frame CFG as a deterministic optimal control problem in this way. Additionally, existing guidance scheduling methods such as \emph{CFG++} and \emph{IntCFG} also emerge from this same optimal control formulation under alternative choices of the traversal penalty (Appendix~\ref{Apdx:Comparison Methods CFG}). By examining the path costs of this optimisation problem, we find that the CFG diffusion sampling trajectory is encouraged to: align the conditional and CFG scores at early-to-mid sampling times; converge towards a local maximum at late sampling times; and avoid traversing near a maximum too early. This provides significant insight into how CFG works, and why it exhibits certain pathologies.

\begin{figure}[t]
    \centering
    \begin{subfigure}{0.49\linewidth}
        \centering
        \resizebox{\linewidth}{!}{%
            \includegraphics[height=3cm]{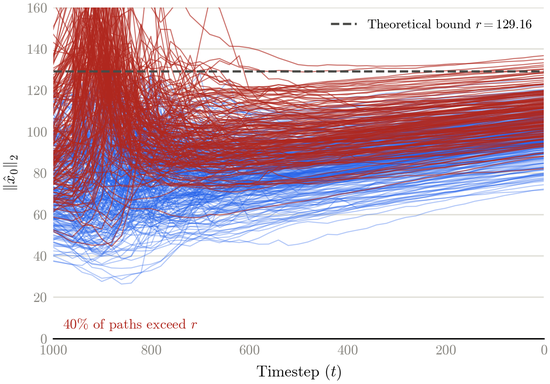}%
            \hspace{0.5pt}%
            \includegraphics[height=3cm]{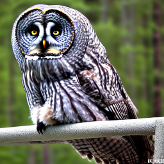}%
        }
        \caption{CFG}
    \end{subfigure}
    \hfill
    \begin{subfigure}{0.49\linewidth}
        \centering
        \resizebox{\linewidth}{!}{%
            \includegraphics[height=3cm]{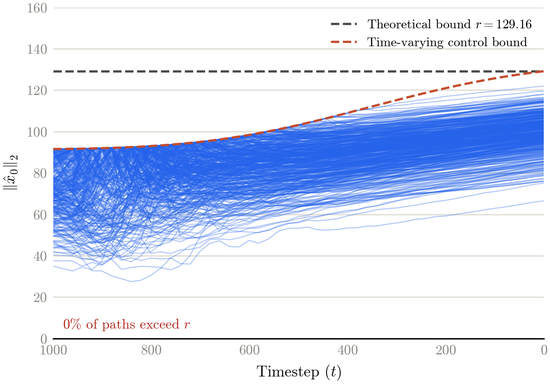}%
            \hspace{0.5pt}%
            \includegraphics[height=3cm]{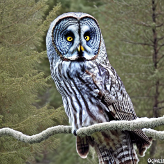}%
        }
        \caption{CACFG}
    \end{subfigure}
    \caption{Tweedie's magnitude ($\lVert \hat{\boldsymbol{x}}_0(t, \boldsymbol{x}_t, \xi) \rVert_2$) of diffusion sampling trajectories across time with $\omega = 7.5$ and using (a) CFG \citep{ClassifierFreeGuidance_Ho2022} and (b) CACFG (ours), alongside a representative sample generated by each method. Note that $40 \%$ of CFG sampling paths violate the theoretical $95 \%$ upper magnitude bound. CACFG restricts permissible controls to lie within a time-varying bounded control set.}
    \label{fig:Sampling magnitude plot}
    \vspace{-1em}
\end{figure}

Our optimal control formulation shows that traditional CFG imposes \emph{no} constraint on the control set. This permits the sampling path to move arbitrarily far from the current image estimate given sufficiently low traversal penalty. We argue this is directly responsible for CFG's failure modes in high-guidance regimes. To address this, we propose a constrained control set and corresponding curvature-aware CFG scheme (CACFG) that relies on a property latent diffusion models have by construction: the training objective of their variational autoencoder (VAE) is explicitly regularised towards an isotropic Gaussian. We hypothesise that this forces most encodings to lie within a hypersphere of an estimable radius, and we show empirically that the Tweedie estimates produced by traditional CFG sampling frequently violate this theoretical upper bound, while those of CACFG do not (Figure~\ref{fig:Sampling magnitude plot}). We find that across diffusion models, datasets, and guidance schedules, CACFG has superior generative quality at mid-to-high guidance strengths as measured by various quality metrics.

Our core contributions are as follows:
\begin{enumerate*}[label=(\roman*)]
    \item We show that diffusion guidance can be interpreted as a continuous-time control problem;
    \item We show that classifier-free guidance (CFG) is one solution to this control problem, and that existing guidance scheduling methods, such as CFG++ \citep{CFG++_Chung2025} and IntCFG \citep{IntCFG_Kynkaanniemi2024}, arise from the same framework under alternative choices of the traversal penalty;
    \item We propose a method of using the Gaussian regularisation term of the VAE training loss to infer the structure of the embedded image manifold;
    \item We propose a constrained control set for the continuous-time optimisation problem that results in a diffusion guidance (CACFG) that provides superior conditional generation with a less severe quality-diversity tradeoff than unconstrained guidance methods, at negligible additional computational cost.
\end{enumerate*}

\section{Preliminaries}
\noindent\textbf{Denoising diffusion models} generate data by learning to reverse a fixed diffusion (noising) process. We write $\boldsymbol{x}_t \in \mathbb{R}^m$ for $t \in [0,T]$ to denote the state acted on by this process. The diffusion itself is driven by the stochastic differential equation (SDE) \citep{DenoisingDiffusion_Ho2020, Song2021},
\begin{equation}
    d\boldsymbol{x}_t = f(t)\boldsymbol{x}_t\,dt + g(t)\,d\boldsymbol{B}_t, \quad \forall\, t \in [0,T], \label{eq:background forward sde}
\end{equation}
where $f(t), g(t) : \mathbb{R} \to \mathbb{R}$ are continuous and $\cdot d\boldsymbol{B}_t$ denotes an It\^{o} integral with respect to the $m$-dimensional Brownian motion process, $\boldsymbol{B}_t$. The form of $f$ and $g$ is chosen such that (\ref{eq:background forward sde}) gradually erases structure in $\boldsymbol{x}_0$ as $t \rightarrow T$, leaving a state that is approximately isotropic Gaussian at termination, $p(\boldsymbol{x}_T) \approx \mathcal{N}(\boldsymbol{0}, \sigma_T^2 I)$. Modern image diffusion models leverage both unconditional and conditional modes, where the initial measure of the state $X_0$ is conditioned on a measurable event $\xi$ of an auxiliary variable $Y$ that takes values in $\Omega_Y$, i.e. $P(\boldsymbol{X}_0|\xi)$. This allows for $\xi = \Omega_Y$ to denote unconditional generation, and $\xi \neq \Omega_Y$ to describe conditional generation with target $\xi$. We formally define the above diffusion as the \emph{canonical image diffusion} in Definition~\ref{Def:Canonical Diffusion}, and in Theorem~\ref{Thm:Conditional Forward} we show that it admits the transition kernel $p(\boldsymbol{x}_t \mid \boldsymbol{x}_0) = \mathcal{N} \left( \alpha_t \boldsymbol{x}_0, \ \beta_t^2 I^{(m \times m)} \right)$ with $\alpha_t, \beta_t : \mathbb{R} \to \mathbb{R}$ and where $I^{(m \times m)}$ denotes the $m$-dimensional identity.

Generating a sample amounts to reversing (\ref{eq:background forward sde}), which can be done stochastically via \citeauthor{Anderson1982}'s reverse-time SDE \citep{Anderson1982}, or (more commonly) deterministically via the probability-flow ODE \citep{Song2021},
\begin{equation}
    d\boldsymbol{x}_t / dt = f(t)\boldsymbol{x}_t - g(t)^2 \nabla_{\boldsymbol{x}_t}\log p(\boldsymbol{x}_t \mid \xi) / 2, \label{eq:background flow ode}
\end{equation}
both of which require an estimate of the score, $\nabla_{\boldsymbol{x}_t}\log p(\boldsymbol{x}_t \mid \xi)$. This is obtained by training a denoiser $\boldsymbol{\epsilon}_\theta : \mathbb{R}^m \times \mathbb{R} \times \mathcal{Y} \to \mathbb{R}^m$ to recover $\boldsymbol{x}_0$ from a noised observation using the objective,
\begin{equation}
    \min_\theta\ \mathbb{E}_{\boldsymbol{x}_0 \sim P(\boldsymbol{x}_0 \mid \xi),\ \boldsymbol{x}_t \sim p(\boldsymbol{x}_t \mid \boldsymbol{x}_0),\ t \sim p(t)} \left[ \left\lVert \boldsymbol{x}_0 - \hat{\boldsymbol{x}}_0 \right\rVert_2^2 \right], \quad \hat{\boldsymbol{x}}_0(t, \boldsymbol{x}_t, \xi; \theta) \triangleq \frac{\boldsymbol{x}_t - \beta_t \boldsymbol{\epsilon}_\theta(t, \boldsymbol{x}_t, \xi)}{\alpha_t}. \label{eq:background training objective}
\end{equation}
For a minimiser $\theta^\star$ of (\ref{eq:background training objective}), the score is proportional to the denoiser (see Theorem~\ref{Thm:Score-Model Link}), 
\begin{equation}
    \nabla_{\boldsymbol{x}_t} \log p(\boldsymbol{x}_t \mid \xi) = - \boldsymbol{\epsilon}_{\theta^\star}(t, \boldsymbol{x}_t, \xi) / \beta_t, \label{eq:score-model link maintext}
\end{equation}
so either reverse-process formulation (SDE or ODE) can be run with $\boldsymbol{\epsilon}_{\theta^\star}$ in place of the score. Additionally, by \emph{Tweedie's formula} \citep{TweediesFormula_Efron2011} the optimal denoiser's output is equal to the posterior mean (see Theorem~\ref{Thm:Tweedie}),
\begin{equation}
    \hat{\boldsymbol{x}}_0(t, \boldsymbol{x}_t, \xi) = (\boldsymbol{x}_t - \beta_t \boldsymbol{\epsilon}_{\theta^\star}(t, \boldsymbol{x}_t, \xi)) / \alpha_t = \mathbb{E}[\boldsymbol{X}_0 \mid \boldsymbol{x}_t, \xi], \label{eq:Tweedie maintext}
\end{equation}
and note that we are taking $\hat{\boldsymbol{x}}_0(t, \boldsymbol{x}_t, \xi) \triangleq \hat{\boldsymbol{x}}_0(t, \boldsymbol{x}_t, \xi; \theta^{\star})$ for brevity.

\noindent\textbf{Classifier-free guidance (CFG).} In many cases the conditional score alone is not sufficient to reconstruct high-quality images of the desired class. Although $\boldsymbol{\epsilon}_{\theta^\star}(t, \boldsymbol{x}_t, \xi)$ recovers $\nabla_{\boldsymbol{x}_t}\log p(\boldsymbol{x}_t \mid \xi)$ exactly at optimality (Theorem~\ref{Thm:Score-Model Link}), faithfully sampling $p(\boldsymbol{x}_0 \mid \xi)$ is not the same as adhering strongly to $\xi$: a specific conditioning event is typically supported by a small number of training samples compared to the entire training dataset. Thus, the model constructs a noisy estimate of $P(\boldsymbol{X}_0 \mid \xi)$, which places significant probability mass over samples weakly consistent with $\xi$ \citep{ClassifierGuidance_Dhariwal2021, ClassifierFreeGuidance_Ho2022}.

To strengthen class adherence, guidance methods target a tilted distribution in place of $p(\boldsymbol{x}_t \mid \xi)$,
\begin{equation}
    \bar{p}(\boldsymbol{x}_t \mid \xi) \propto p(\xi \mid \boldsymbol{x}_t)^\omega\, p(\boldsymbol{x}_t), \qquad \omega \in \mathbb{R}_{>0}. \label{eq:background tilted}
\end{equation}
Increasing $\omega$ concentrates probability mass around these maximal-density regions, whilst leaving its maximisers unchanged. Differentiating the log of (\ref{eq:background tilted}) and substituting (\ref{eq:score-model link maintext}) at $\xi = \Omega_Y$ gives
\begin{align}
    \nabla_{\boldsymbol{x}_t} \log \bar{p}(\boldsymbol{x}_t \mid \xi) &= \omega \nabla_{\boldsymbol{x}_t} \log p(\xi \mid \boldsymbol{x}_t) - \boldsymbol{\epsilon}_{\theta^\star}(t, \boldsymbol{x}_t, \Omega_Y) / \beta_t \label{eq:background classifier guidance} \\
    &= - \Bigl[\boldsymbol{\epsilon}_{\theta^\star}(t, \boldsymbol{x}_t, \Omega_Y) + \omega \bigl( \boldsymbol{\epsilon}_{\theta^\star}(t, \boldsymbol{x}_t, \xi) - \boldsymbol{\epsilon}_{\theta^\star}(t, \boldsymbol{x}_t, \Omega_Y) \bigr)\Bigr] / \beta_t, \label{eq:background cfg}
\end{align}
where the second line eliminates $\nabla_{\boldsymbol{x}_t}\log p(\xi \mid \boldsymbol{x}_t)$ via Bayes' rule and the score identity in (\ref{eq:score-model link maintext}). \emph{Classifier guidance} \citep{ClassifierGuidance_Dhariwal2021} evaluates (\ref{eq:background classifier guidance}) by training a separate model $p_\theta(\xi \mid \boldsymbol{x}_t)$; \emph{classifier-free guidance} \citep{ClassifierFreeGuidance_Ho2022} instead uses (\ref{eq:background cfg}), interpolating the class-conditional and unconditional noise models defined above. Note that this guidance does not represent a true diffusion process: $\bar{p}(\boldsymbol{x}_t \mid \xi)$ is not the marginal induced by diffusing $\boldsymbol{x}_0 \sim \bar{P}(\boldsymbol{X}_0 \mid \xi)$, so following this score abandons the reverse-time SDE/flow ODE dynamics of (\ref{eq:background flow ode}) \citep{Karras2024}; however, it empirically improves perceptual sample quality \citep{ClassifierGuidance_Dhariwal2021, ClassifierFreeGuidance_Ho2022}.

\noindent\textbf{Latent diffusion models.} Let $\boldsymbol{Z} \in \mathbb{R}^n$ be an $\mathbb{R}^n$-valued random variable representing a natural image in $n$-dimensional pixel-space, sharing the auxiliary (label) variable $Y$ defined above (and in Definition~\ref{Def:Canonical Diffusion}). As above, for $\xi \in \mathcal{Y}$ we write $P(\boldsymbol{Z} \mid \xi)$ for the conditional law, and $\boldsymbol{z} \sim P(\boldsymbol{Z} \mid \xi)$ for a realised image. Taking $\boldsymbol{X}_0 \triangleq \boldsymbol{Z}$ (so $m=n$) equates the diffusion state with the image directly, giving $P(\boldsymbol{X}_0 \mid \xi) = P(\boldsymbol{Z} \mid \xi)$ and recovering ordinary \emph{pixel-space diffusion}.

\emph{Latent diffusion models} \citep{StableDiffusion_Rombach2022} instead diffuse over a compressed representation of $\boldsymbol{z}$, obtained from a variational autoencoder (VAE) \citep{VAE_Kingma2014}. The encoder maps a realised image to a distribution of latents, $q_\phi(\boldsymbol{x}_0 \mid \boldsymbol{z}) \triangleq \mathcal{N}\left(\boldsymbol{x}_0;\, \boldsymbol{\mu}_\phi(\boldsymbol{z}),\, \mathrm{diag}\!\left(\boldsymbol{\sigma}_\phi(\boldsymbol{z})^2\right)\right)$,
and the decoder gives the reconstruction likelihood, $p_\psi(\boldsymbol{z} \mid \boldsymbol{x}_0) \triangleq \mathcal{N}\left(\boldsymbol{z};\, \boldsymbol{\nu}_\psi(\boldsymbol{x}_0),\, \mathrm{diag}\!\left(\boldsymbol{\gamma}_\psi(\boldsymbol{x}_0)^2\right)\right)$. The density of $\boldsymbol{X}_0 \mid \xi$ can then be obtained via $p(\boldsymbol{x}_0 \mid \xi) \triangleq \int_{\mathbb{R}^n} q_\phi(\boldsymbol{x}_0 \mid \boldsymbol{z}) \, dP(\boldsymbol{z} \mid \xi)$.

The encoder and decoder are trained jointly, over parameters $\phi, \psi$ respectively, to minimise the negative evidence lower bound
\begin{equation}
    \mathcal{L}_{\text{VAE}}(\phi, \psi) \triangleq \mathbb{E}_{\boldsymbol{z} \sim P(\boldsymbol{Z})} \Big[\, \underbrace{-\, \mathbb{E}_{q_\phi(\boldsymbol{x}_0 \mid \boldsymbol{z})}\!\left[\log p_\psi(\boldsymbol{z} \mid \boldsymbol{x}_0)\right]}_{\text{reconstruction}} \;+\; \underbrace{D_{\mathrm{KL}}\!\left(q_\phi(\boldsymbol{x}_0 \mid \boldsymbol{z}) \,\big\|\, \mathcal{N}(\boldsymbol{0}, I^{(m \times m)})\right)}_{\text{regularisation}} \Big], \label{eq:background vae loss}
\end{equation}
following \citet{VAE_Kingma2014}, where $P(\boldsymbol{Z}) \triangleq P(\boldsymbol{Z} \mid \Omega_Y)$ is the unconditional image law. Note that the KL term regularises the encoder distribution towards the same isotropic standard normal that the forward process of (\ref{eq:background forward sde}) itself converges to as $t \to T$--a fact we leverage later.

\vspace{-0.5em}
\section{Related Work}
\vspace{-0.5em}
\citet{CFG++_Chung2025} attributes CFG artefacts such as colour oversaturation and broken invertibility to an off-manifold effect: since $\omega>1$ extrapolates beyond the conditional and unconditional denoiser outputs, the resulting noise estimate can fall off the data manifold. By reformulating guidance as a manifold-constrained score-matching problem, \citet{CFG++_Chung2025} appeal to the linear manifold hypothesis \citep{ManifoldConstraint_Chung2022} to compute a \emph{manifold constrained gradient} to force the sample to remain on the theorised linear image manifold. Alternatively, \emph{Interval-CFG (IntCFG)} \citep{IntCFG_Kynkaanniemi2024} hypothesises that CFG artefacts can be attributed to \emph{when} guidance is applied: guidance at high noise levels collapses diversity by prematurely committing the trajectory to a small set of modes, while guidance at low noise levels has minimal effect. Thus, IntCFG disables guidance outside an interval of intermediate noise levels. Though motivated differently, both methods let the guidance scale of (\ref{eq:background cfg}) vary across the sampling trajectory, a fact we make concrete in Appendix~\ref{Apdx:Comparison Methods CFG}. Alternatively, techniques such as \emph{Adaptive Projected Guidance (APG)} \citep{APG_Sadat2025} reshape the CFG update itself. However, none of this work provides a rigorous explanation as to why CFG is able to generate conditional images, \emph{despite} violating diffusion sampling dynamics.

\citet{Karras2024} attributes CFG's effectiveness to a shared model bias: a capacity-limited denoiser over-covers sparsely populated regions of the data manifold, and because the unconditional model exhibits this bias more severely than the conditional one, their ratio acts as an attractive force that pulls samples towards the image manifold. \citet{RenyiCFG_Moufad2025} show that the tilted distribution of (\ref{eq:background tilted}) can be recovered by augmenting CFG with a corrective R\'enyi-divergence term, and they argue that CFG is just poorly approximating the true tilted distribution. \citet{CFGPredCorrect_Bradley2025} instead prove that in the continuous-time SDE limit, the stochastic (DDPM) variant \citep{DenoisingDiffusion_Ho2020} of CFG is identical to a predictor-corrector scheme that pairs DDIM \citep{Song2022} denoising with a Langevin-dynamics corrector on the tilted distribution. However, this does not account for deterministic (flow ODE) variants of CFG, and the explanations of \citet{Karras2024} and \citet{RenyiCFG_Moufad2025} disagree on the key mechanisms.

\vspace{-0.5em}
\section{Reframing CFG as an Optimal Control} \label{sec:CFG Optimal Control}
\vspace{-0.5em}
We first consider the tilted density of (\ref{eq:background tilted}), and note that raising $p(\xi \mid \boldsymbol{x}_t)$ to any power $\omega > 0$ leaves its maximiser (w.r.t. $\boldsymbol{x}_t$) unchanged. Thus, the CFG update of (\ref{eq:background cfg}) can be viewed as naively moving towards a local maximum of $p(\xi \mid \boldsymbol{x}_t)$ at each timestep; however, this does not guarantee maximisation with respect to the terminal state, $\boldsymbol{x}_0$. Treating CFG sampling as purely an optimisation problem leads to the natural question: \textit{does CFG sampling construct a path that maximises $p(\xi \mid \boldsymbol{x}_0)$ subject to some reasonable path constraints?} More concretely, we wish to understand whether CFG constructs trajectories that terminate at a state, $\boldsymbol{x}_0$, that \emph{maximise} $p(\xi \mid \boldsymbol{x}_0)$ subject to interpretable path costs that punish/reward movement from the initial position, $\boldsymbol{x}_T$. 

We treat the sequence of denoising targets used along the trajectory as free decisions rather than following a fixed set of transition kernels, and we introduce a time-dependent control variable $\boldsymbol{u}_t \in U(t) \subseteq \mathbb{R}^m$ taking values in a nonempty, closed, and \emph{convex} time-varying control set $U(t)$. Additionally, note that maximising $p(\xi|\boldsymbol{x}_0)$ w.r.t. $\boldsymbol{x}_0$ is equivalent to maximising $p(\boldsymbol{x}_0 \mid \xi) / p(\boldsymbol{x}_0)$ via Bayes' rule. Under this framework, diffusion guidance is the solution to 
\begin{equation}
    \max_{\boldsymbol{u}_{[t_l,T]}} \left\{ \int_{t_l}^{T} L(s, \boldsymbol{x}_s, \boldsymbol{u}_s) \, ds + \phi(t_l, \boldsymbol{x}_{t_l}) \right\}, \qquad \dot{\boldsymbol{x}}_s = \boldsymbol{b}(s, \boldsymbol{x}_s, \boldsymbol{u}_s), \label{eq:oc max path problem}
\end{equation}
with running reward $L : \mathbb{R} \times \mathbb{R}^m \times \mathbb{R}^p \rightarrow \mathbb{R}$, dynamics constraint $\boldsymbol{b} : \mathbb{R} \times \mathbb{R}^m \times \mathbb{R}^p \rightarrow \mathbb{R}^m$, and terminal reward $\phi(t_l, \boldsymbol{x}_{t_l})$, where $\phi(s, \boldsymbol{x}_s) \triangleq \log\left( \frac{p(\boldsymbol{x}_s \mid \xi)}{p(\boldsymbol{x}_s)} \right) \ \forall \ s \in [t_l, T]$. (\ref{eq:oc max path problem}) defines a reverse-time, continuous-time maximisation problem, and we write
\begin{equation}
    \hat{V}(t, \boldsymbol{x}_t) \triangleq \sup_{\boldsymbol{u}_{[t_l,t]}} \left\{ \int_{t_l}^{t} L(s, \boldsymbol{x}_s, \boldsymbol{u}_s) \, ds + \phi(t_l, \boldsymbol{x}_{t_l}) \right\},
\end{equation}
to define the best accumulated reward achievable by time $t$. Bellman's principle of optimality \citep{OriginalHJB_Bellman1957} shows that $\hat{V}$ must satisfy the reverse-time Hamilton--Jacobi--Bellman (HJB) equation\footnote{We state the reverse-time HJB equation in Lemma~\ref{Lma:Reverse HJB}, and provide a rigorous HJB derivation in Appendix~\ref{Apdx:HJB CFG Formulation} based on \citet{OriginalHJB_Bellman1957,CalculusofVariationsHJB_Liberzon2012,NonlinearControlHJB_Berkovitz2012}.}
\begin{equation}
    \hat{V}_t(t, \boldsymbol{x}_t) = \sup_{\boldsymbol{u}_t \in U(t)} \Big\{ L(t, \boldsymbol{x}_t, \boldsymbol{u}_t) - \big\langle \hat{V}_{\boldsymbol{x}_t}(t, \boldsymbol{x}_t), \, \boldsymbol{b}(t, \boldsymbol{x}_t, \boldsymbol{u}_t) \big\rangle \Big\}, \quad \forall\, t \in (t_l, T], \label{eq:oc reverse hjb restated}
\end{equation}
with boundary condition $\hat{V}(t_l, \boldsymbol{x}_{t_l}) = \phi(t_l, \boldsymbol{x}_{t_l})$. Throughout, we take the termination time $t_l \in (0,T)$, which keeps $\beta_s^2 > 0$ (and hence the relevant densities well-defined) over the whole horizon. Additionally, we require that $L$, $\hat{V}$, and $\boldsymbol{b}$ be regular\footnote{See Problem~\ref{Prob:OCP} for full conditions, and note that many are immediately satisfied due to the differentiability of the trained denoiser and Gaussian transition kernels.} and $C^{1}$ differentiable in $t$, $\boldsymbol{x}_t$, and $\boldsymbol{u}_t$.

We now choose $L$ and $\boldsymbol{b}$. For the running reward, we penalise the divergence of the control $\boldsymbol{u}_s$ from the \emph{unconditional} sample, multiplied by some time-dependent cost scale $r(s) > 0$, alongside an additional, as-yet-unspecified term $h$ capturing further path costs/rewards,
\begin{equation}
    L(s, \boldsymbol{x}_s, \boldsymbol{u}_s) \triangleq -r(s) \left\lVert \boldsymbol{u}_s - \hat{\boldsymbol{x}}_0(s, \boldsymbol{x}_s, \Omega_Y) \right\rVert_2^2 + h(s, \boldsymbol{x}_s, \boldsymbol{u}_s). \label{eq:oc running reward}
\end{equation}
The choice of $\hat{\boldsymbol{x}}_0(s, \boldsymbol{x}_s, \Omega_Y)$ as the anchor for the $\ell_2$ penalty is deliberate: by Tweedie's formula (\ref{eq:Tweedie maintext}), it is the unconditional posterior mean of $\boldsymbol{x}_s$, i.e.\ the most plausible denoising target available at that noise level. Penalising the $\ell_2$ deviation of $\boldsymbol{u}_s$ from this quantity therefore encodes a natural prior: a candidate target should only be preferred over the unconditional estimate if there is sufficient reward for doing so. Framing guidance this way turns the question of \emph{why} CFG works into a question about $h$ and $\boldsymbol{b}$: is there a choice of $h$ and $\boldsymbol{b}$ whose optimal control recovers the CFG update of (\ref{eq:background cfg})?

To connect this to diffusion sampling, we take the dynamics $\boldsymbol{b}$ to be the probability-flow ODE of (\ref{eq:background flow ode}), written in Tweedie form via (\ref{eq:score-model link maintext}) and (\ref{eq:Tweedie maintext}), with the denoising target replaced by the control $\boldsymbol{u}$ itself,
\begin{equation}
    \dot{\boldsymbol{x}}_s = \boldsymbol{b}(s, \boldsymbol{x}_s, \boldsymbol{u}_s) \triangleq f(s) \boldsymbol{x}_s - g(s)^2 \big( \alpha_s \boldsymbol{u}_s - \boldsymbol{x}_s \big) / (2 \beta_s^2). \label{eq:oc flow ode control}
\end{equation}
Setting $\boldsymbol{u}_s = \hat{\boldsymbol{x}}_0(s, \boldsymbol{x}_s, \Omega_Y)$ at every step recovers ordinary unconditional sampling, while any other choice of control defines an equally valid, but different, flow-ODE trajectory. What remains is to specify $h$, and here we make a significant assumption: we posit that the value function is equal to the quantity CFG attempts to maximise at each timestep:
\begin{equation}
    \hat{V}(t, \boldsymbol{x}_t) \triangleq \phi(t, \boldsymbol{x}_t) = \log\left( p(\boldsymbol{x}_t \mid \xi) / p(\boldsymbol{x}_t) \right). \label{eq:oc value function is cfg ratio}
\end{equation}
We argue that this is a reasonable assertion: CFG is already effective at conditioning the diffusion trajectory--so choosing a value function whose gradients align with the CFG score incentivises movement in that direction, and this assumption immediately satisfies the boundary condition of (\ref{eq:oc reverse hjb restated}). We now ask what constraints (what $h$) has $\phi$ as its value function under (\ref{eq:oc reverse hjb restated}).

Substituting (\ref{eq:oc flow ode control}) and (\ref{eq:oc value function is cfg ratio}) into the reverse-time HJB equation (\ref{eq:oc reverse hjb restated}) and using the Kolmogorov forward equation to derive $\hat{V}_t$ allows us to solve for a unique $h$. We show in Theorem~\ref{Thm:CFG HJB} that $\hat{V} = \phi$ solves (\ref{eq:oc reverse hjb restated}) for every $s \in [t_l, T]$ when
\begin{align}
    h(s, \boldsymbol{x}_s) =& \underbrace{\frac{1}{2} g(s)^2 \big\langle \nabla_{\boldsymbol{x}_s} \phi(s, \boldsymbol{x}_s), \nabla_{\boldsymbol{x}_s} \log p(\boldsymbol{x}_s \mid \xi) \big\rangle}_{\text{conditional score alignment}} + \underbrace{r(s) \left\lVert \Upsilon(s, \boldsymbol{x}_s) \right\rVert_2^2}_{\text{CFG score magnitude}} \nonumber \\
    & \underbrace{- \frac{1}{2} \frac{g(s)^2 \alpha_s}{\beta_s^2} \big\langle \nabla_{\boldsymbol{x}_s} \phi(s, \boldsymbol{x}_s), \Upsilon(s, \boldsymbol{x}_s) \big\rangle}_{\text{CFG score regularisation}} + \underbrace{\frac{1}{2} g(s)^2 \Delta_{\boldsymbol{x}_s} \phi(s, \boldsymbol{x}_s)}_{\text{early-convergence penalty}}, \quad \text{where} \label{eq:oc h} \\
    \Upsilon(s, \boldsymbol{x}_s) \triangleq& \text{proj}_{U(s) - \hat{\boldsymbol{x}}_0(s, \boldsymbol{x}_s, \Omega_Y)} \left\{ \frac{1}{4} \frac{g(s)^2 \alpha_s}{r(s) \beta_s^2} \nabla_{\boldsymbol{x}_s} \phi(s, \boldsymbol{x}_s) \right\}. \label{eq:oc upsilon}
\end{align}
The corresponding optimal control is then given by
\begin{equation}
    \boldsymbol{u}_t^\star \triangleq \hat{\boldsymbol{x}}_0(t, \boldsymbol{x}_t, \Omega_Y) + \Upsilon(t, \boldsymbol{x}_t). \label{eq:oc optimal control}
\end{equation}
\emph{CFG-like guidance} can then be recovered by reversing Tweedie's formula (\ref{eq:Tweedie maintext}), and using (\ref{eq:score-model link maintext}) to solve for the score imposed by the optimal control:
\begin{equation}
    \nabla_{\boldsymbol{x}_t} \log \bar{p}(\boldsymbol{x}_t \mid \xi) = - \frac{\boldsymbol{x}_t - \alpha_t \boldsymbol{u}_t^\star}{\beta_t^2} = -\frac{1}{\beta_t}\Bigl[\boldsymbol{\epsilon}_{\theta^\star}(t, \boldsymbol{x}_t, \Omega_Y) - \frac{\alpha_t}{\beta_t} \Upsilon \Bigr]. \label{eq:oc score}
\end{equation}

\noindent\textbf{Traditional CFG, scheduled variants, and Kingma's SNR.} CFG and other scheduled guidance variants with a positive guidance-weight function $\omega(t) > 0$ can be recovered by setting
\begin{equation}
    r(t) = \frac{g(t)^2 \alpha_t^2}{\beta_t^4} \cdot \frac{1}{4 \omega(t)}, \quad U(t) = \mathbb{R}^m \ \forall \ t \in (0, T], \label{eq:oc CFG form}
\end{equation}
and substituting (\ref{eq:oc upsilon}) and (\ref{eq:oc CFG form}) into (\ref{eq:oc score}). \emph{Traditional CFG} \citep{ClassifierFreeGuidance_Ho2022} follows when $\omega(t) = w > 0$, and we cover CFG++ \citep{CFG++_Chung2025} and IntCFG \citep{IntCFG_Kynkaanniemi2024} in Appendix~\ref{Apdx:Comparison Methods CFG}. Note that for CFG, CFG++, and IntCFG, the $r(t)$ of (\ref{eq:oc CFG form}) is proportional to $-d\,\mathrm{SNR}(t)/dt = g(t)^2\alpha_t^2/\beta_t^4$, with $\mathrm{SNR}(t) \triangleq \alpha_t^2/\beta_t^2$. Let $\lambda_t \triangleq \log \mathrm{SNR}(t)$ (Kingma's log-SNR), then \citet{DiffusionSNRCorrection_Kingma2023} show $-d\lambda_t/dt$ is the weighting that converts the uniform noise-matching ELBO objective into an equivalent image-matching objective, and $\lVert \boldsymbol{\epsilon} - \hat{\boldsymbol{\epsilon}} \rVert_2^2 = \mathrm{SNR}(t) \lVert \boldsymbol{x}_0 - \hat{\boldsymbol{x}}_0 \rVert_2^2$. Thus, minimising $\mathbb{E}\lVert \boldsymbol{x}_0 - \hat{\boldsymbol{x}}_0 \rVert_2^2$ weighted by $g(t)^2\alpha_t^2/\beta_t^4$ is the same training objective as the (unweighted) noise-prediction ELBO. In other words, CFG, CFG++, and IntCFG are optimal controls with an $\ell_2$ penalty that is specified with the same scale as the denoising training objective.

We now ask: \emph{what do the constraints of (\ref{eq:oc h}) imply about the paths constructed by CFG-like guidance?} The \emph{conditional score alignment} term in (\ref{eq:oc h}) rewards paths that align the conditional and CFG scores, which is reasonable, given we seek trajectories that converge to high-density conditional regions. This term is scaled by $g(s)^2$, which decreases as $s \rightarrow 0$ under most diffusion parameterisations \citep{Karras2022}, suggesting that the conditional alignment matters most in early-to-mid sampling times. The CFG score projection, $\Upsilon$, can be thought of as the \emph{maximum allowable CFG movement} given the current Tweedie estimate, $\hat{\boldsymbol{x}}_0(s, \boldsymbol{x}_s, \Omega_Y)$, at time $s$. Assuming that $\hat{\boldsymbol{x}}_0 \in U(s)$ (the control set should contain all Tweedie estimates), then $\Upsilon$ is either $\boldsymbol{0}$ or aligned with the CFG score, $\nabla_{\boldsymbol{x}_s} \phi$, in the sense that $\langle \nabla_{\boldsymbol{x}_s} \phi(s, \boldsymbol{x}_s), \Upsilon(s, \boldsymbol{x}_s) \rangle \geq 0$ (by \citealp[Prop.~1.1.9]{ConvexOptimizationTheory_Bertsekas2009}), and thus, the \emph{CFG score regularisation} term is $\leq 0$. This term is negligible at early sampling times, $\alpha_s \rightarrow 0$ as $s \rightarrow T$, but grows as $\beta_s \rightarrow 0$ as $s \rightarrow 0$, and the only way to avoid this penalty is to find a local optimum ($\nabla \phi = 0$) or enter a region that no longer has viable movement in the direction of the CFG score, $\Upsilon = 0$. Thus, we can view this term as a penalty for not finding a local maximum, or not maximising allowable growth of $\phi$. The \emph{CFG score magnitude} term blindly rewards large CFG score magnitudes regardless of sign; however, this is entirely negated by the $\ell_2$ penalty of (\ref{eq:oc running reward}) at the optimal control, $\boldsymbol{u}^\star_t$. The \emph{early-convergence penalty} term is negative at a mode of $\phi$  ($\Delta_{\boldsymbol{x}_s}\phi < 0$) and is scaled by $g(s)^2$ (which shrinks as $s \rightarrow 0$), so this term penalises trajectories that sit near a sharp maximum at early sampling times, discouraging them from settling into a mode too early (reducing mode collapse). 

Thus, CFG-like guidance is an optimal control specified in the same loss scale as the denoising training objective, and it finds paths that: 
\begin{enumerate*}[label=(\roman*)]
    \item are encouraged to align the conditional and CFG scores at early-to-mid sampling times;
    \item are encouraged to converge to a local maximum at late sampling times; and
    \item are penalised for nearing local maximums at early sampling times.
\end{enumerate*}

\section{Curvature-Aware CFG (CACFG)} \label{sec:cacfg method}
Section~\ref{sec:CFG Optimal Control} shows that CFG-style guidance solves an optimal control problem, but says nothing about whether the resulting control is a \emph{safe} one to take. Compensating for weak class adherence with a large guidance scale $\omega$ (or small $r(s)$) may move the sampling trajectory into regions with few training samples, and thus, require the model to extrapolate. This is a known failure mode of neural networks--adversarial regions populate extrapolated sections of the learned manifold, and they cause networks to behave erratically \citep{Szegedy2014, Goodfellow2015, Madry2019}.

\begin{wrapfigure}{R}{0.40\textwidth}
    \vspace{-2.4em}
    \begin{minipage}{0.40\textwidth}
    \begin{algorithm}[H]
        \scriptsize
        \caption{Curvature-Aware Guidance} \label{alg:CACFG}
        \begin{algorithmic}
            \Require diffusion model: $\boldsymbol{\epsilon}_{\theta^\star}$; target $y \in \Omega_Y$; sampling steps with $t_0 = t_l \approx 0 < \dots < t_n = T$; noise schedules $\alpha_t$, $\beta_t$; guidance weight schedule $\omega(t)$; control set $U(t)$; upper radius $r_u$; interpolation parameter $\gamma$.\\

            \State $\boldsymbol{x}_{t_n} \sim \mathcal{N}(0, I)$
            \For{$i = n, \dots, 1$}
                \State $\boldsymbol{\epsilon}_{\Omega} \gets \boldsymbol{\epsilon}_{\theta^\star}(t_i, \boldsymbol{x}_{t_i}, \Omega_Y)$
                \State $\boldsymbol{\epsilon}_\xi \gets \boldsymbol{\epsilon}_{\theta^\star}(t_i, \boldsymbol{x}_{t_i}, \xi)$
                \State $\tilde{\boldsymbol{\epsilon}} \gets \boldsymbol{\epsilon}_\Omega + \omega(t_i)(\boldsymbol{\epsilon}_\xi - \boldsymbol{\epsilon}_\Omega)$
                \State $\hat{\boldsymbol{x}}_0^{\mathrm{cfg}} \gets (\boldsymbol{x}_{t_i} - \beta_{t_i}\tilde{\boldsymbol{\epsilon}})/\alpha_{t_i}$
                \State $\rho_i^2 \gets (1-\gamma)\dfrac{\alpha_{t_i}^2}{\alpha_{t_i}^2+\beta_{t_i}^2}\,r_u^2 + \gamma\, r_u^2$
                \State $\boldsymbol{u}_i^\star \gets \mathrm{proj}_{U_\gamma(t_i)}\!\big(\hat{\boldsymbol{x}}_0^{\mathrm{cfg}}\big)$ \Comment{clip to $\|\cdot\|_2 \le \rho_i$}
                \State $\boldsymbol{\epsilon}_i^\star \gets (\boldsymbol{x}_{t_i} - \alpha_{t_i}\boldsymbol{u}_i^\star)/\beta_{t_i}$
                \State $\boldsymbol{x}_{t_{i-1}} \gets \text{Reverse flow-ode step with } \boldsymbol{\epsilon}_i^\star, \ \boldsymbol{x}_{t_i}$
            \EndFor
            \State \textbf{return} $\boldsymbol{x}_{t_0}$
        \end{algorithmic}
    \end{algorithm}
    \end{minipage}
    \vspace{-2em}
\end{wrapfigure}

Existing corrections to CFG address this by assuming the underlying image manifold has some known geometry \citep{ManifoldConstraint_Chung2022, CFG++_Chung2025}. Such hypotheses are reasonable when $\boldsymbol{x}_0$ is itself a natural image, but for the latent diffusion models used in practice \citep{StableDiffusion_Rombach2022}, $\boldsymbol{x}_0$ is instead a VAE latent code, and the manifold it occupies is whatever geometry that particular VAE has learned. \emph{A manifold hypothesis tuned to one model's VAE carries no guarantee of transferring to another's.}

We instead propose a correction built on a property the VAE is trained to have: the regularisation term of the VAE objective in (\ref{eq:background vae loss}) penalises the encoder distribution for departing from an isotropic standard normal, $q_\phi(\boldsymbol{x}_0 \mid \boldsymbol{z}) \approx \mathcal{N}(\boldsymbol{0}, I^{(m\times m)})$ \citep{VAE_Kingma2014}. This term rewards encodings that are close to the origin, and we hypothesise that valid encodings become more sparse as distance from the origin increases (Appendix~\ref{Apdx:True VAE latents}), which means that these regions are likely adversarial from the perspective of the diffusion denoiser. We propose restricting the control set from Section~\ref{sec:CFG Optimal Control} to latents contained within a hypersphere with radius $r_u$. That is, we consider a situation where $\lVert \boldsymbol{X}_0 \rVert_2^2 \le r_u^2$, and then by Lemma~\ref{Lma:Tweedie Bound}, the magnitude of Tweedie's estimate is also upper bounded by $r_u$\footnote{There exists no measure on the truncated distribution that would permit $\lVert \hat{x}_0(t, \boldsymbol{x}_t, \xi) \rVert_2^2 > r_u^2$.}: $\left\lVert \hat{\boldsymbol{x}}_0(t, \boldsymbol{x}_t, \xi) \right\rVert_2^2 \leq r_u^2$. This limits the ability of large control inputs to direct the sample into these hypothesised adversarial regions.

Unfortunately, restricting the generative domain in this way would require knowledge of the learned VAE latent space. To circumvent this we again leverage the Gaussian regularisation of the VAE loss and assume $\boldsymbol{X}_0 \sim \mathcal{N}(\boldsymbol{0}, I^{(m\times m)})$, under which $\lVert \boldsymbol{X}_0 \rVert_2^2$ follows a chi-squared distribution with $m$ degrees of freedom. We then define a confidence level $p$ and a corresponding radius $r_u$ such that:
\begin{equation}
    P\!\left( \lVert \boldsymbol{X}_0 \rVert_2^2 \le r_u^2 \right) = p. \label{eq:cacfg ru}
\end{equation}
Empirically, we find that the Tweedie estimates (i.e.\ the unclipped control input) of standard CFG sampling routinely violate the $95.0\%$ bound imposed by this assumption (Figure~\ref{fig:Sampling magnitude plot}). Furthermore, under $\boldsymbol{X}_0 \sim \mathcal{N}(\boldsymbol{0}, I^{(m\times m)})$ and the transition kernel of Theorem~\ref{Thm:Conditional Forward}, $\boldsymbol{X}_0$ and $\boldsymbol{x}_t$ are jointly Gaussian, so the Tweedie estimate $\hat{\boldsymbol{x}}_0(t, \boldsymbol{x}_t, \xi)$ is itself Gaussian, and thus,
\begin{equation}
    P\!\left( \left\lVert \hat{\boldsymbol{x}}_0(t, \boldsymbol{x}_t, \xi) \right\rVert_2^2 \leq \frac{\alpha_t^2}{\alpha_t^2 + \beta_t^2}\, r_u^2 \right) = p \quad \forall \ t \in [t_l, T]. \label{eq:cacfg shrinkage bound}
\end{equation}
That is, for $t>0$ the Tweedie estimate is expected to lie \emph{closer} to the origin than the raw bound $r_u$.

\noindent\textbf{A curvature-aware control set.} The two bounds above represent the two extremes of the Gaussian hypothesis: (\ref{eq:cacfg ru}) assumes only the raw radius $r_u$, while (\ref{eq:cacfg shrinkage bound}) assumes the hypothesis holds exactly. We interpolate between the two with a single tuneable parameter $\gamma \in [0,1]$, and take the interpolated radius to define a time-varying \emph{convex} control set within the optimal control framework of Section~\ref{sec:CFG Optimal Control},
\begin{equation}
    U_\gamma(t) \triangleq \left\{ \boldsymbol{u}_t \in \mathbb{R}^m : \lVert \boldsymbol{u}_t \rVert_2^2 \le \frac{\alpha_t^2}{\alpha_t^2+\beta_t^2}\,r_u^2 + \gamma \left( r_u^2 - \frac{\alpha_t^2}{\alpha_t^2+\beta_t^2}\,r_u^2 \right) \right\} \quad \forall \ t \in [t_l, T]. \label{eq:cacfg control set}
\end{equation}
At $\gamma=0$ we assume the full Gaussian hypothesis, while $\gamma = 1$ requires only the upper bound to hold. We substitute $U_\gamma(t)$ into (\ref{eq:oc CFG form}), and refer to this guidance scheme as curvature-aware CFG (CACFG).

\section{CACFG Experiments}
\vspace{-0.5em}
\noindent\textbf{Diffusion models.} We evaluate CACFG using two latent diffusion models of varying sizes: Stable Diffusion 1.5 (SD1.5) \citep{StableDiffusion_Rombach2022} and Stable Diffusion XL (SDXL) \citep{SDXL_Podell2024}. Both are text-conditioned latent diffusion models that run the diffusion of (\ref{eq:background forward sde}) on a VAE latent space trained using the same objective of (\ref{eq:background vae loss}). SD1.5 and SDXL have differing generative resolutions of 512$\times$512 and 1024$\times$1024, which have substantially different compute costs. Thus, for SD1.5 experiments we generate 10,000 samples for each guidance-dataset-strength setting, while for SDXL we generate 2,000. We run both models on a NVIDIA RTX4090, and the curvature-aware projection does not incur substantive compute overhead--all models run at their originally reported speeds.

\noindent\textbf{Datasets.} We generate images conditioned on two datasets: ImageNet \citep{ImageNetBenchmark_Russakovsky2015}, which provides class-conditional generation across 1000 object categories, and COCO \citep{COCO_Lin2014}, which provides caption-conditioned generation--a more sophisticated generative task.

\noindent\textbf{Guidance schedules.} We compare three guidance-weight schedules $\omega(t)$: the constant-weight CFG of (\ref{eq:background cfg}) \citep{ClassifierFreeGuidance_Ho2022}, and the two scheduled alternatives CFG++ \citep{CFG++_Chung2025} and IntCFG \citep{IntCFG_Kynkaanniemi2024}, whose exact forms are given in Appendix~\ref{Apdx:Comparison Methods CFG}. Each schedule is evaluated both on its own and combined with the curvature-aware control set of (\ref{eq:cacfg control set}), giving the CACFG, CACFG++, and IntCACFG guidance schemes respectively. Across all curvature-aware variants we fix the confidence level of (\ref{eq:cacfg ru}) at $p=0.95$, which corresponds to $r_u = 129.16$ and $r_u = 257.16$ for SD1.5 and SDXL, respectively. Based on initial ablations (Appendix~\ref{Apdx:Blending parameter}), we use the blending parameter $\gamma = 0.5$ and $\gamma = 1.0$ for SD1.5 and SDXL, respectively. Additionally, we use the DDIM ODE solver \citep{Song2022} for SD1.5 and SDXL, as it is the most common deterministic solver.

\noindent\textbf{Guidance weights.} We report the guidance weight of CFG and compute an approximate equivalent weight under each other schedule. That is, we find schedule parameters, $\theta$, such that $\int_0^T \omega_{\text{CFG}} dt = \int_0^T \omega(t; \theta) dt$. This compares schedules with equal approximate total path perturbations. For SD1.5 we use CFG weights $\omega_{\text{CFG}} = 5, 7.5, 10, 15$, and for SDXL we use $\omega_{\text{CFG}} = 10, 15, 20, 25$. This discrepancy arises because SDXL is significantly more resistant to excessive CFG guidance, making comparisons at lower guidance strengths less meaningful.

\noindent\textbf{Metrics.} We report nine metrics measuring image quality across varying axes. Fr\'echet Inception Distance (FID) \citep{Heusel2018} measures the similarity of real and generated image distributions. Density and Coverage \citep{DensityCoverage_Naeem2020} report the fidelity and diversity of the generated images by comparing manifolds fitted to real and generated images. NIQE \citep{NIQE_Mittal2013}, BRISQUE \citep{BRISQUE_Mittal2012}, and TReS \citep{TReS_Alireza2022} are no-reference perceptual quality metrics, scoring an image's naturalness based on its pixel statistics. CLIP score \citep{CLIP_Radford2021} measures prompt alignment via the cosine similarity between a generated image's CLIP embedding and its conditioning. Human Preference Score v2 (HPSv2) \citep{HPS_Wu2023} is a model trained to mimic human rankings of image quality. We note that HPSv2 and CLIP score are known to artificially improve at large CFG strengths despite poorer image quality and increasing saturation \citep{GuidanceMatters_Xie2026}. We therefore additionally report Saturation \citep{APG_Sadat2025}, computed as the mean of the saturation channel after converting each image from RGB to HSV, to measure this effect.

\subsection{Results} \label{sec:results}
\noindent\textbf{Distribution quality and diversity.} When using SD1.5 we find FID, Density, and Coverage are improved or equivalent under the curvature-aware (CA) correction for all three guidance schedules and both datasets, with the gap widening substantially at higher guidance strengths (Figures~\ref{fig:Metric plot coco sd1.5} and \ref{fig:Metric plot sd1.5}). Since Density and Coverage improve alongside FID rather than trading off against it, the corrected samples are more accurately contained within the real-image manifold and also cover a wider region of it, indicating an improvement in both quality and diversity. 

\begin{wrapfigure}{R}{0.4\textwidth}
    \vspace{-1em}
    \centering
    \includegraphics[width=\linewidth]{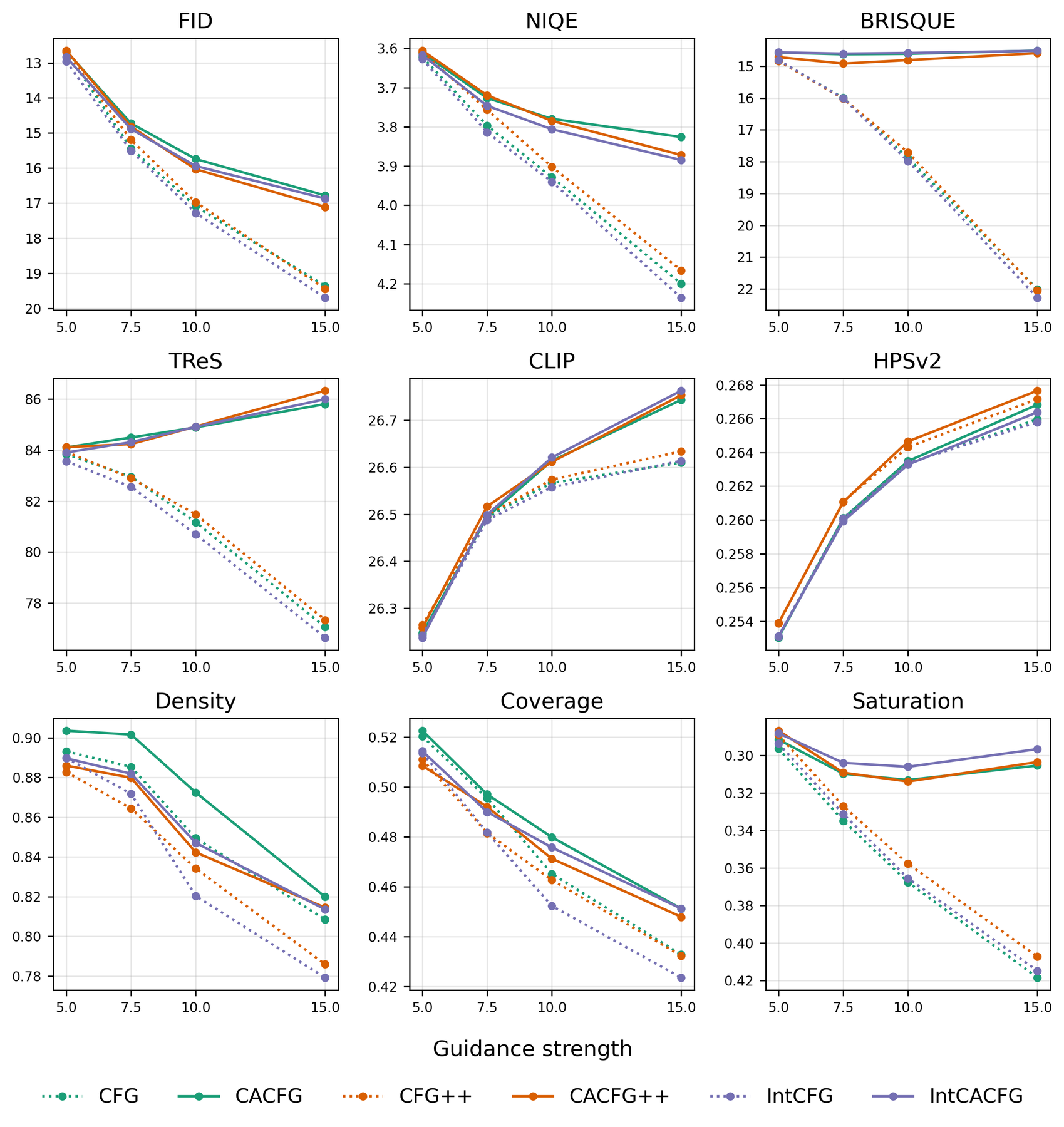}
    \caption{Performance of CFG, CFG++, and IntCFG with (solid line) and without (dotted line) our curvature-aware correction using SD1.5 on COCO. Some metric axes have been flipped so that the improvement direction is vertically up the page. Zoom in for easier viewing.}
    \label{fig:Metric plot coco sd1.5}
    \vspace{-1em}
\end{wrapfigure}

The same pattern holds for SDXL (Figure~\ref{fig:Metric plot sdxl}); however, Density and Coverage also depend on guidance strength in a way they do not for SD1.5. At the lowest guidance strength tested, we find both Density and Coverage are marginally degraded by the CA correction, before favouring the correction from moderate guidance strengths onward, alongside a FID advantage that similarly grows across the same range. We hypothesise two explanations for this phenomenon. First, SDXL is a higher-capacity model and is more robust to the manifold departures, thus, the CA control set has less to correct at low guidance, and its benefit only becomes apparent once uncorrected guidance has sufficiently corrupted the sampling trajectory. Second, at low guidance strengths the more restrictive control constraint, $U_\gamma(t)$, limits the sampler's ability to search the latent space, a cost that is outweighed once the guidance strength (effective search budget) is high enough. Thus, across both SD1.5 and SDXL, the CA correction provides superior conditional generation at moderate-to-high guidance strengths, with a less severe quality-diversity tradeoff than that exhibited by unconstrained guidance methods. These findings are supported qualitatively by the samples in Figure~\ref{fig:SD1.5_CACFG_COCO} and Appendix~\ref{Apdx:Additional samples}.

\begin{wrapfigure}{R}{0.4\textwidth}
    \vspace{-1.3em}
    \centering
    \resizebox{\linewidth}{!}{%
    \begin{tabular}{@{}c@{\hspace{2pt}}c@{}c@{}c@{}}
        \raisebox{-0.5\height}{\rotatebox{90}{\small CFG}} &
        \raisebox{-0.5\height}{\includegraphics[width=0.3\linewidth]{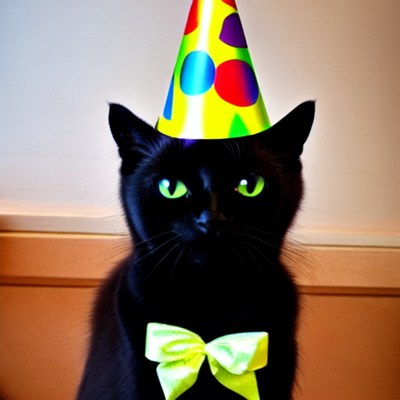}} &
        \raisebox{-0.5\height}{\includegraphics[width=0.3\linewidth]{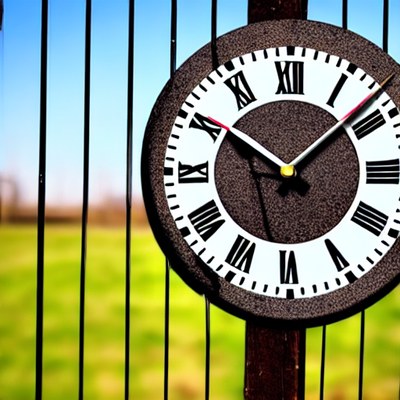}} &
        \raisebox{-0.5\height}{\includegraphics[width=0.3\linewidth]{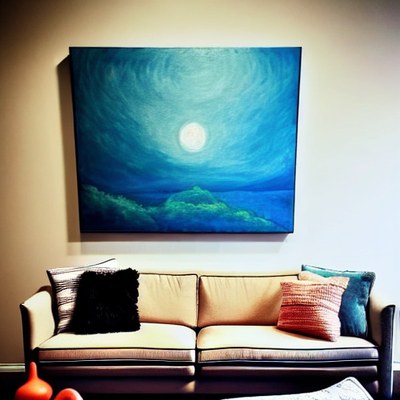}} \\
        \raisebox{-0.5\height}{\rotatebox{90}{\small CACFG}} &
        \raisebox{-0.5\height}{\includegraphics[width=0.3\linewidth]{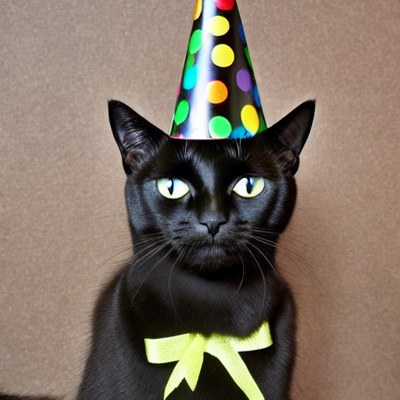}} &
        \raisebox{-0.5\height}{\includegraphics[width=0.3\linewidth]{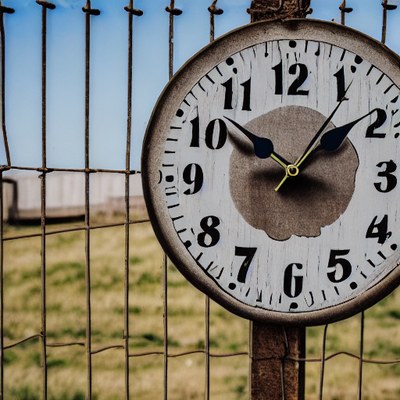}} &
        \raisebox{-0.5\height}{\includegraphics[width=0.3\linewidth]{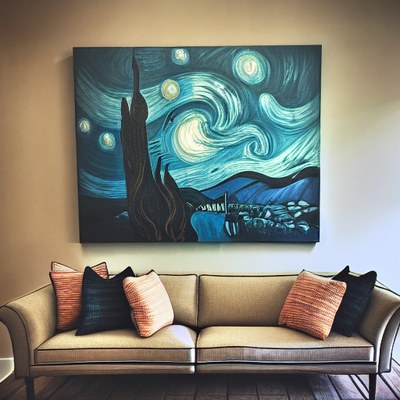}}
    \end{tabular}
    }
    \caption{Comparison of CFG \citep{ClassifierFreeGuidance_Ho2022} and CACFG (ours) at guidance strength $\omega = 15$ on the COCO \citep{COCO_Lin2014} dataset. CACFG better preserves fine detail: the middle image has a correctly numbered clock face, while the rightmost image reconstructs Van Gogh's \textit{The Starry Night}.}
    \label{fig:SD1.5_CACFG_COCO}
    \vspace{-1.5em}
\end{wrapfigure}

\noindent\textbf{Perceptual quality.} NIQE, BRISQUE, and TReS all favour the curvature-aware variants at every guidance strength for both SD1.5 and SDXL, with the largest gains at high guidance. As these are no-reference measures of image pixel statistics, this indicates that curvature-aware guidance produces images with more natural low-level statistics, independently of any comparison to the real data distribution, and that this benefit is consistent across model scale, guidance schedule, and diffusion model.

\noindent\textbf{Prompt alignment.} CLIP score improves only marginally under ImageNet but more substantially under COCO, for both SD1.5 and SDXL. We hypothesise that ImageNet's single-object class prompts are simple enough that standard guidance can satisfy them at these guidance strengths, leaving no room for the CA correction to help, whereas COCO's complex caption prompts are harder to satisfy: by keeping the sampling trajectory closer to the latent image manifold, curvature-aware guidance allows higher guidance strengths to be spent productively toward matching the full caption, rather than being wasted on manifold-departing artefacts.

\noindent\textbf{Human-preference metrics and saturation.} HPSv2 is the one metric where the curvature-aware variants are outperformed by their uncorrected counterparts, an effect that is mild for SD1.5 but more pronounced for SDXL. However, the Saturation values of both Figures~\ref{fig:Metric plot sd1.5} and \ref{fig:Metric plot sdxl} show that this coincides with increasing saturation levels, while curvature-aware guidance keeps saturation roughly constant. Given that \citet{GuidanceMatters_Xie2026} have shown HPSv2 and CLIP scores artificially increase alongside saturation, and that every other image-quality metric favours the curvature-aware correction for both models, we interpret the higher HPSv2 of uncorrected guidance as an instance of this failure mode rather than a quality advantage.

\section{Conclusion}
We show via the HJB equation that CFG and other guidance schedules can be derived as the optimal control for a continuous-time optimisation problem. Existing CFG corrections assume the underlying image manifold has some known geometry, but this is not guaranteed for latent diffusion models, whose latent image manifold is governed by what the VAE has learned. We use the Gaussian regularisation term of the VAE loss to infer that most image encodings lie within a hypersphere of estimable radius. We then propose a curvature-aware CFG correction that constrains the control set of our optimisation problem to a time-varying hypersphere whose radius is informed by the estimated radius. We find that at high guidance scales this generates images with substantially better quality, as measured by FID, NIQE, BRISQUE, TReS, Density, and Coverage. Thus, our constrained guidance provides superior conditional generation with a less severe quality-diversity tradeoff than that exhibited by unconstrained guidance methods, and functionally no additional compute cost.

\subsubsection*{Acknowledgments}
Max Collins is the recipient of the Research Training Program scholarship funded by the Australian Government. This research was partially supported by National Intelligence and Security Discovery Research Grants (project number NS220100007), funded by the Department of Defence, Australia.

\bibliographystyle{iclr2027_conference}
\bibliography{../references}

\appendix
\addcontentsline{toc}{section}{Appendix} 
\part{Appendix} 
\parttoc 

\newpage
\section{Statements}
\subsection{Limitations}
We note three limitations of our work, each of which we view as a promising direction for future work rather than a fundamental shortcoming of the proposed optimal control framework. First, the control set $U_\gamma(t)$ we optimise over is time-varying but not state-varying; extending it to depend jointly on sampling time and the current state $\boldsymbol{x}_t$ would more directly incentivise training a model to perform the image-manifold projection. Second, our optimal control perspective does not currently explain guidance techniques such as APG \citep{APG_Sadat2025}, which act by manipulating the guidance gradient itself, rather than by reweighting it through the guidance schedule. Finally, our formulation is posed for deterministic sampling processes, and does not immediately extend to stochastic samplers. However, these extensions are unnecessary for understanding and improving traditional classifier-free guidance on deterministic samplers, which was the primary objective of this work.

\subsection{Reproducibility}
We ensure reproducibility by providing a detailed description of our algorithm in the main text (Algorithm~\ref{alg:CACFG}). Additionally, our code and the configuration files necessary to replicate our experiments are made publicly available. Comparison methods are implemented using default parameters or with parameters selected based on the respective authors' tuning algorithms.

\subsection{AI use}
We used large language models as a writing and verification aid during the preparation of this paper. Specifically, AI assistance was used to check and typeset proofs that we had already derived by hand, to edit and improve the clarity of rough prose drafts, and to assist with LaTeX formatting of figures. AI assistance was not used to write experiment code, nor to generate the ideas, methods, or theoretical contributions presented in this work; all such content originates from the authors.

\section{Comparison Methods CFG Form} \label{Apdx:Comparison Methods CFG}
Both CFG++ \citep{CFG++_Chung2025} and Interval-CFG (IntCFG) \citep{IntCFG_Kynkaanniemi2024} propose guidance scheduling schemes to modulate the strength of CFG-like guidance across diffusion sampling paths. Similarly to Section~\ref{sec:CFG Optimal Control}, we let the guidance scale of (\ref{eq:background cfg}) vary with time, $\omega(t)$. We consider a discretisation $0 = t_0 < t_1 < \dots < t_k = T$ of the canonical image diffusion of Definition~\ref{Def:Canonical Diffusion}, with sampling proceeding from $t_k$ down to $t_0$, i.e.\ each step moves $\boldsymbol{x}_{t_i}$ to the less-noisy $\boldsymbol{x}_{t_{i-1}}$. The guidance schedules of CFG++ and IntCFG are as follows.

\noindent\textbf{CFG++.} \citet{CFG++_Chung2025} use a manifold-constrained gradient to force the diffusion sampling trajectory to remain on the theorised linear image manifold \citep{ManifoldConstraint_Chung2022, DiffusionDDS_Chung2024}. Under common diffusion scheduling schemes this culminates in a time-varying CFG guidance strength, $\omega(t)$. Given the discretisation above and using a DDIM diffusion schedule, \citeauthor{CFG++_Chung2025}'s CFG++ guidance scale is given by
\begin{equation}
    \omega(t_i) = w \cdot \frac{\alpha_{t_{i-1}} \beta_{t_i}}{\alpha_{t_{i-1}} \beta_{t_i} - \alpha_{t_i} \beta_{t_{i-1}}}, \label{eq:CFG++}
\end{equation}
where $w > 0$ is a hyperparameter that governs the strength of the CFG++ guidance. Using Euler derivative approximations: $\dot{\alpha}_{t_i} \approx \frac{\alpha_{t_i} - \alpha_{t_{i-1}}}{\Delta t} \implies \alpha_{t_{i-1}} \approx \alpha_{t_{i}} - \Delta t \cdot \dot{\alpha}_{t_i}$ and $\dot{\beta}_{t_i} \approx \frac{\beta_{t_{i}} - \beta_{t_{i-1}}}{\Delta t} \implies \beta_{t_{i-1}} \approx \beta_{t_{i}} - \Delta t \cdot \dot{\beta}_{t_i}$, then (\ref{eq:CFG++}) can be written as 
\begin{equation}
    \omega(t) \approx w \cdot \left[ \frac{1}{\Delta t} \cdot \frac{\alpha_{t} \beta_{t}}{\alpha_t \dot{\beta}_t - \dot{\alpha}_t \beta_t} - \frac{\dot{\alpha}_{t} \beta_{t}}{\alpha_t \dot{\beta}_t - \dot{\alpha}_t \beta_t} \right], \label{eq:CFG++ derivative form}
\end{equation}
where $\Delta t$ functions as a parameter describing how coarse the diffusion discretisation is. Notably, from (\ref{eq:CFG++ derivative form}), it is clear that $\omega(t)$ diverges as $\Delta t \rightarrow 0$. This suggests that CFG++ is sensitive to the number of sampling steps, with an exploding guidance scale as the granularity of the ODE solver increases.

\noindent\textbf{Interval CFG (IntCFG).} \citet{IntCFG_Kynkaanniemi2024} hypothesise that guidance at high noise levels reduces sample diversity, and is redundant at low noise levels, where the trajectory has already converged onto the correct mode. Thus, IntCFG confines guidance to an interval of intermediate noise levels. In our notation this is the piecewise-constant schedule
\begin{equation}
    \omega(t) = \begin{cases} w, & \frac{\beta_t}{\alpha_t} \in(\sigma_{\mathrm{lo}}, \sigma_{\mathrm{hi}}], \\ 1, & \text{otherwise}, \end{cases} \label{eq:intcfg schedule}
\end{equation}
for hyperparameters $0<\sigma_{\mathrm{lo}}<\sigma_{\mathrm{hi}}$ and guidance weight $w>1$, selected by grid search against a downstream metric (e.g.\ FID) rather than derived. Note that technically IntCFG results in a discontinuous $\omega(t)$ scale, which conflicts with HJB assumptions made in Section~\ref{sec:CFG Optimal Control}. To circumvent this, a smooth sigmoid transition could be used at each cutoff point $\sigma_{\mathrm{lo}}$ and $\sigma_{\mathrm{hi}}$ to ensure continuity. However, in this paper we use the standard IntCFG proposed by \citet{IntCFG_Kynkaanniemi2024} and specified in (\ref{eq:intcfg schedule}).

\section{Ablations}
\subsection{SD1.5 theoretical Tweedie upper bound}
Figure~\ref{fig:MagStats SD1.5} extends Figure~\ref{fig:Sampling magnitude plot} to the full range of guidance strengths tested in the main paper, showing Tweedie's magnitude ($\lVert \hat{\boldsymbol{x}}_0(t, \boldsymbol{x}_t, \xi) \rVert_2$) of SD1.5 CFG and CACFG sampling trajectories across $\omega \in \{5, 7.5, 10, 15\}$ on (a) ImageNet \citep{ImageNetBenchmark_Russakovsky2015} and (b) COCO \citep{COCO_Lin2014}. Notably, it highlights that increasing guidance strength dramatically increases the number of paths that violate the theoretical $95 \%$ upper limit, $r_u = 129.16$.

\begin{figure}[h]
    \centering
    \begin{subfigure}{\linewidth}
        \centering
        \resizebox{\linewidth}{!}{%
        \begin{tabular}{@{}c@{\hspace{2pt}}c@{}c@{}c@{}c@{}}
            & \makebox[0.18\linewidth][c]{\scriptsize $\omega = 5$} & \makebox[0.18\linewidth][c]{\scriptsize $\omega = 7.5$} & \makebox[0.18\linewidth][c]{\scriptsize $\omega = 10$} & \makebox[0.18\linewidth][c]{\scriptsize $\omega = 15$} \\
            \noalign{\vskip 2pt}
            \raisebox{-0.5\height}{\rotatebox{90}{\scriptsize CFG}} & \raisebox{-0.5\height}{\includegraphics[width=0.18\linewidth]{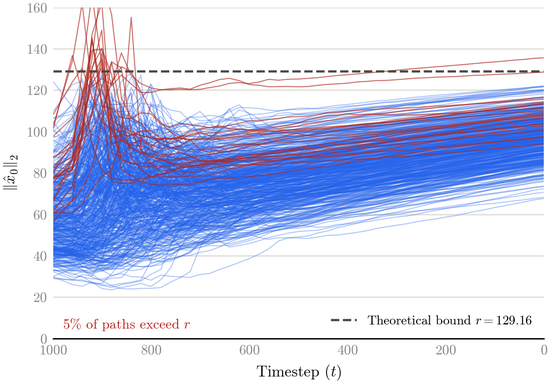}} & \raisebox{-0.5\height}{\includegraphics[width=0.18\linewidth]{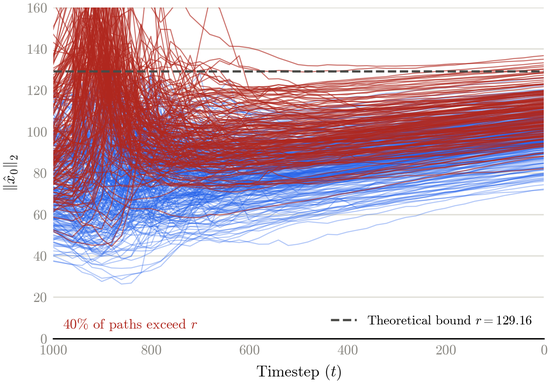}} & \raisebox{-0.5\height}{\includegraphics[width=0.18\linewidth]{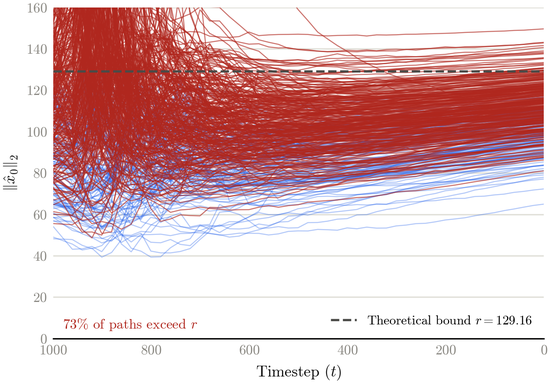}} & \raisebox{-0.5\height}{\includegraphics[width=0.18\linewidth]{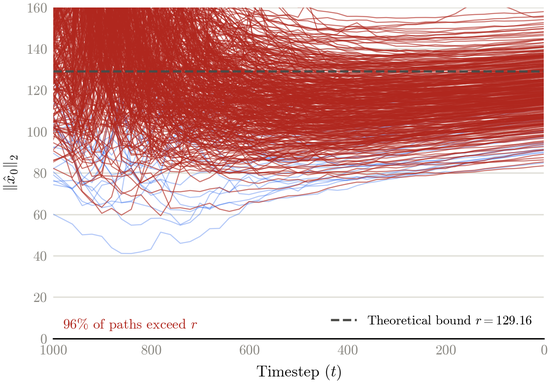}} \\
            \raisebox{-0.5\height}{\rotatebox{90}{\scriptsize CACFG}} & \raisebox{-0.5\height}{\includegraphics[width=0.18\linewidth]{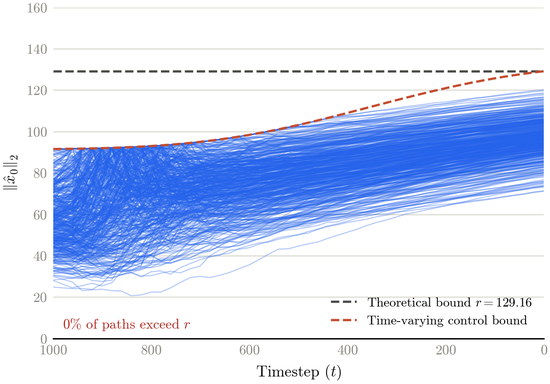}} & \raisebox{-0.5\height}{\includegraphics[width=0.18\linewidth]{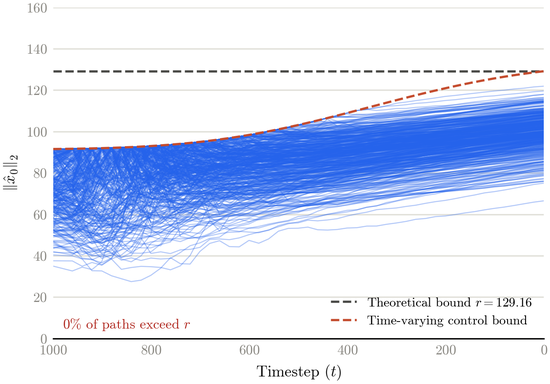}} & \raisebox{-0.5\height}{\includegraphics[width=0.18\linewidth]{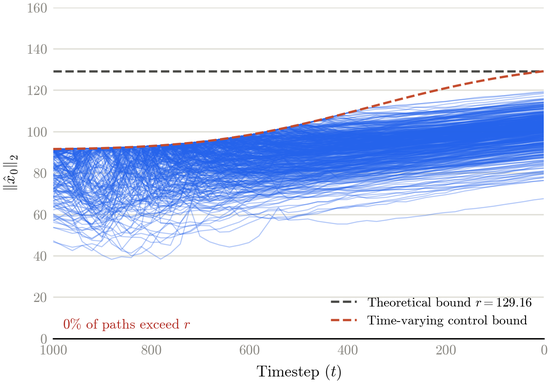}} & \raisebox{-0.5\height}{\includegraphics[width=0.18\linewidth]{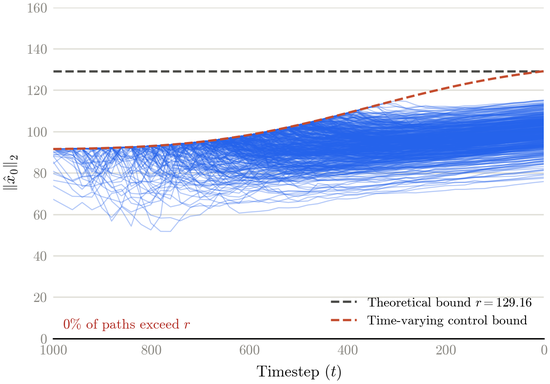}}
        \end{tabular}
        }
        \caption{ImageNet}
    \end{subfigure}

    \vspace{1em}

    \begin{subfigure}{\linewidth}
        \centering
        \resizebox{\linewidth}{!}{%
        \begin{tabular}{@{}c@{\hspace{2pt}}c@{}c@{}c@{}c@{}}
            & \makebox[0.18\linewidth][c]{\scriptsize $\omega = 5$} & \makebox[0.18\linewidth][c]{\scriptsize $\omega = 7.5$} & \makebox[0.18\linewidth][c]{\scriptsize $\omega = 10$} & \makebox[0.18\linewidth][c]{\scriptsize $\omega = 15$} \\
            \noalign{\vskip 2pt}
            \raisebox{-0.5\height}{\rotatebox{90}{\scriptsize CFG}} & \raisebox{-0.5\height}{\includegraphics[width=0.18\linewidth]{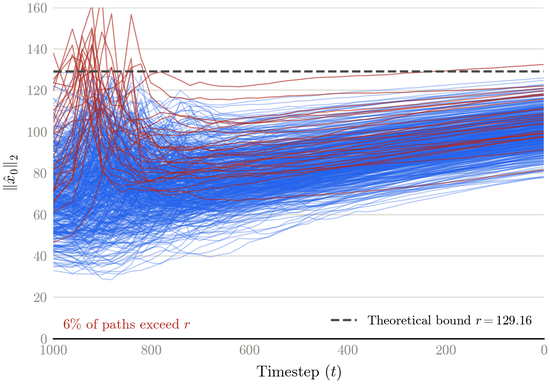}} & \raisebox{-0.5\height}{\includegraphics[width=0.18\linewidth]{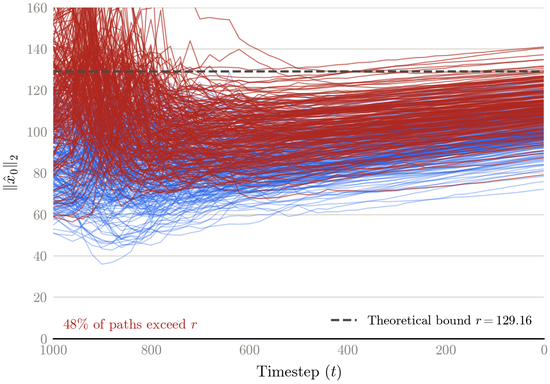}} & \raisebox{-0.5\height}{\includegraphics[width=0.18\linewidth]{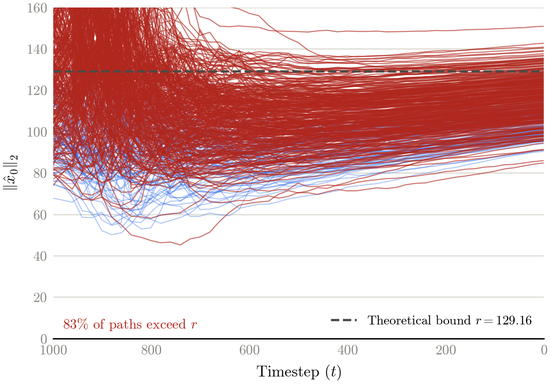}} & \raisebox{-0.5\height}{\includegraphics[width=0.18\linewidth]{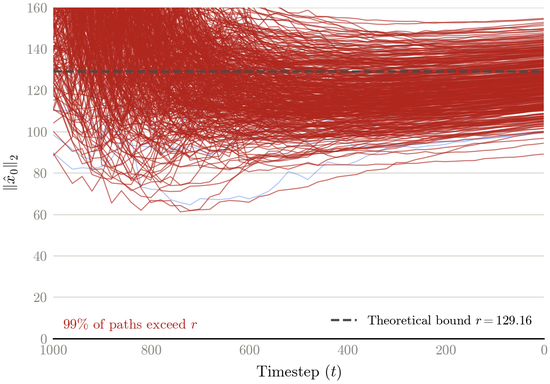}} \\
            \raisebox{-0.5\height}{\rotatebox{90}{\scriptsize CACFG}} & \raisebox{-0.5\height}{\includegraphics[width=0.18\linewidth]{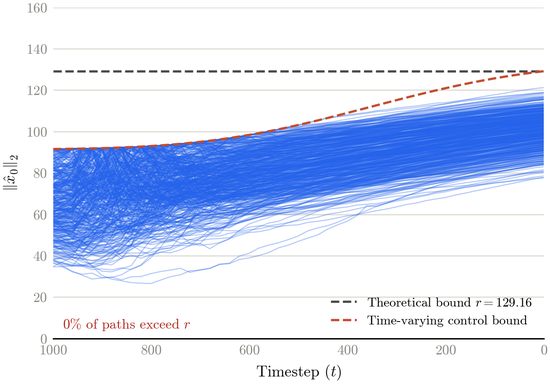}} & \raisebox{-0.5\height}{\includegraphics[width=0.18\linewidth]{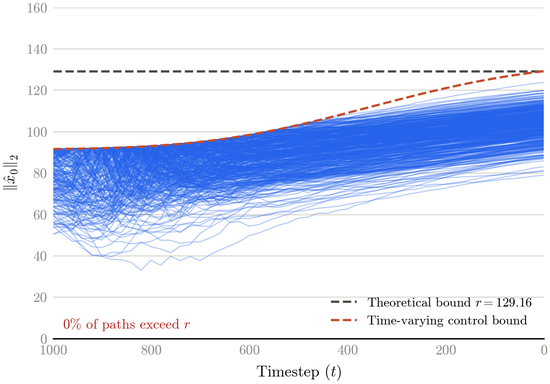}} & \raisebox{-0.5\height}{\includegraphics[width=0.18\linewidth]{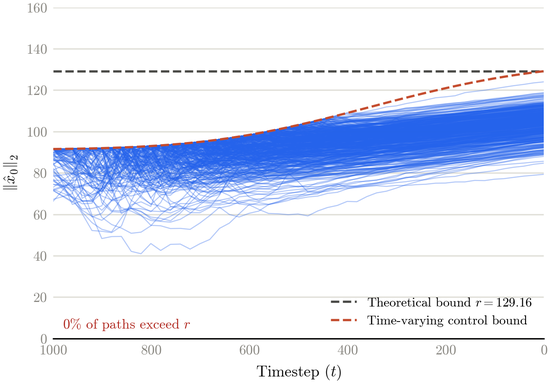}} & \raisebox{-0.5\height}{\includegraphics[width=0.18\linewidth]{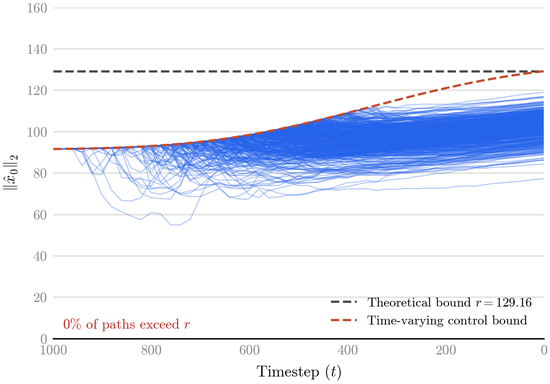}}
        \end{tabular}
        }
        \caption{COCO}
    \end{subfigure}

    \caption{Tweedie's magnitude ($\lVert \hat{\boldsymbol{x}}_0(t, \boldsymbol{x}_t, \xi) \rVert_2$) of SD1.5 CFG and CACFG (ours) sampling trajectories at guidance strengths $\omega \in \{5, 7.5, 10, 15\}$ on (a) ImageNet \citep{ImageNetBenchmark_Russakovsky2015} and (b) COCO \citep{COCO_Lin2014}. Red paths violate the theoretical $95\%$ upper magnitude bound $r_u = 129.16$ (black dashed line); the violating fraction is annotated in each panel. The proportion of violating paths grows sharply with guidance strength for CFG, and this growth is suppressed by our curvature-aware correction, with its time-varying upper bound (red dashed line). We recommend zooming in for easier viewing.}
    \label{fig:MagStats SD1.5}
\end{figure}

\clearpage
\subsection{True VAE latent distribution} \label{Apdx:True VAE latents}
Figure~\ref{fig:VAE Encoding Histogram} shows the empirical distribution of SD1.5 VAE encodings ($\lVert \boldsymbol{x}_0 \rVert_2$) of the ImageNet validation set, alongside the theoretical chi density implied by the isotropic Gaussian hypothesis proposed in Section~\ref{sec:cacfg method}. The two do not agree because the VAE is trained with a regularisation term that encourages encodings towards the theoretical density, but does not enforce their adherence. In practice we find that the high-dimensional theoretical density appears as a narrow spike, whereas the true distribution of encodings is visibly broader and right-skewed, with a tail extending up to an empirical maximum of $211.19$. Despite this, $r_u=129.16$ falls just past the mode of the empirical distribution, where the bulk of the density has already accumulated, so the bound continues to capture a large proportion of real encodings even though the theoretical density does not capture the distribution's true shape.

\begin{figure}[h]
    \centering
    \includegraphics[width=0.35\textwidth]{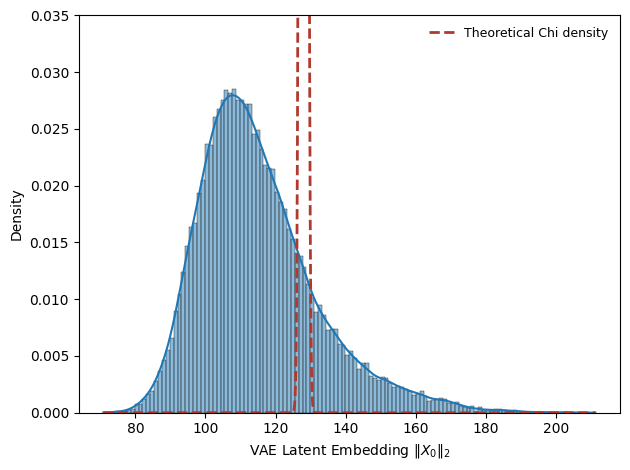}
    \caption{Empirical distribution of SD1.5 VAE encodings ($\lVert \boldsymbol{x}_0 \rVert_2$) of the ImageNet validation set, i.e. the distribution of $\lVert \boldsymbol{x}_0 \rVert_2$. Red line denotes the theoretical chi density implied by the Gaussian hypothesis in Section~\ref{sec:cacfg method}.}
    \label{fig:VAE Encoding Histogram}
\end{figure}

To test whether this mismatch matters in practice, we compare CACFG against \emph{CACFG*}, an otherwise identical scheme that instead sets $r_u=211.19$, i.e. the maximum SD1.5 VAE encoding magnitude observed on the ImageNet validation set. From Table~\ref{tab:SD1.5 vaeablation table} we see that CACFG* is worse than CACFG on BRISQUE, TReS, and Saturation at every guidance strength tested. FID, NIQE, Density, and Coverage follow the same pattern from strength $7.5$ onward. CLIP favours CACFG* at strengths 5 and 7.5, ties at strength 10, and favours CACFG at 15. HPSv2 favours CACFG* across all guidance strengths. However, as we discuss in the main paper, \citet{GuidanceMatters_Xie2026} tie HPSv2 to oversaturation, and CACFG*'s saturation exceeds CACFG's throughout, and all other metrics favour CACFG at high guidance strengths. Thus, we interpret the higher HPSv2 of CACFG* as an instance of the saturation-induced failure mode rather than a quality advantage.

We hypothesise that CACFG outperforms CACFG* because the tighter, theoretically-derived radius encloses a denser region of the encoding manifold: Figure~\ref{fig:VAE Encoding Histogram} shows valid encodings are more sparse with distance from the origin, so the smaller radius keeps sampling trajectories within the well-populated region the denoiser was trained on, rather than permitting convergence to low-confidence regions \citep{Szegedy2014, Goodfellow2015, Madry2019}.

\begin{table}[h]
\centering
\caption{Image quality for SD1.5 samples generated for ImageNet using CFG \citep{ClassifierFreeGuidance_Ho2022} and CACFG with varying upper bound limits, $r_u$. CACFG denotes the $r_u = 129.16$ bound used in the main paper, while CACFG* is the $r_u = 211.19$ bound derived from the maximum SD1.5 VAE encoding of the ImageNet validation dataset. The best result for each dataset and guidance strength is \textcolor{ForestGreen}{\textbf{bold}}.}
\label{tab:SD1.5 vaeablation table}
\resizebox{\textwidth}{!}{%
\small
\begin{tabular}{l l ccccccc cc}
\toprule
\textbf{Strength} & \textbf{Guidance} & \textbf{FID} $\downarrow$ & \textbf{NIQE} $\downarrow$ & \textbf{BRISQUE} $\downarrow$ & \textbf{TReS} $\uparrow$ & \textbf{CLIP} $\uparrow$ & \textbf{HPSv2} $\uparrow$ & \textbf{Density} $\uparrow$ & \textbf{Coverage} $\uparrow$ & \textbf{Saturation} $\downarrow$ \\
\midrule

\multirow[c]{3}{*}{5}
 & CFG      & $\textcolor{ForestGreen}{\mathbf{13.67}}$ & $4.06$ & $14.98$ & $83.24$ & $24.851$ & $25.076$ & $\textcolor{ForestGreen}{\mathbf{1.128}}$ & $0.453$ & $0.299$ \\
 \noalign{\vskip -1.5pt}
 \cmidrule[0.1pt](lr){2-11}
 \noalign{\vskip -1.5pt}
 & CACFG    & $13.70$ & $\textcolor{ForestGreen}{\mathbf{4.05}}$ & $\textcolor{ForestGreen}{\mathbf{14.83}}$ & $\textcolor{ForestGreen}{\mathbf{83.47}}$ & $24.849$ & $25.070$ & $1.124$ & $\textcolor{ForestGreen}{\mathbf{0.454}}$ & $\textcolor{ForestGreen}{\mathbf{0.295}}$ \\
 & CACFG*   & $\textcolor{ForestGreen}{\mathbf{13.67}}$ & $\textcolor{ForestGreen}{\mathbf{4.05}}$ & $14.98$ & $83.27$ & $\textcolor{ForestGreen}{\mathbf{24.856}}$ & $\textcolor{ForestGreen}{\mathbf{25.081}}$ & $1.123$ & $\textcolor{ForestGreen}{\mathbf{0.454}}$ & $0.299$ \\
 \cmidrule(lr){1-11}
\multirow[c]{3}{*}{7.5}
 & CFG      & $15.12$ & $4.27$ & $15.47$ & $82.80$ & $\textcolor{ForestGreen}{\mathbf{25.251}}$ & $25.680$ & $1.139$ & $0.443$ & $0.338$ \\
 \noalign{\vskip -1.5pt}
 \cmidrule[0.1pt](lr){2-11}
 \noalign{\vskip -1.5pt}
 & CACFG    & $\textcolor{ForestGreen}{\mathbf{14.97}}$ & $\textcolor{ForestGreen}{\mathbf{4.22}}$ & $\textcolor{ForestGreen}{\mathbf{14.59}}$ & $\textcolor{ForestGreen}{\mathbf{83.80}}$ & $25.243$ & $25.653$ & $\textcolor{ForestGreen}{\mathbf{1.140}}$ & $\textcolor{ForestGreen}{\mathbf{0.445}}$ & $\textcolor{ForestGreen}{\mathbf{0.316}}$ \\
 & CACFG*   & $15.11$ & $4.27$ & $15.33$ & $82.85$ & $25.249$ & $\textcolor{ForestGreen}{\mathbf{25.687}}$ & $1.132$ & $0.444$ & $0.337$ \\
\cmidrule(lr){1-11}
\multirow[c]{3}{*}{10}
 & CFG      & $16.22$ & $4.44$ & $17.13$ & $81.41$ & $25.391$ & $25.967$ & $1.084$ & $0.424$ & $0.371$ \\
 \noalign{\vskip -1.5pt}
 \cmidrule[0.1pt](lr){2-11}
 \noalign{\vskip -1.5pt}
 & CACFG    & $\textcolor{ForestGreen}{\mathbf{15.73}}$ & $\textcolor{ForestGreen}{\mathbf{4.32}}$ & $\textcolor{ForestGreen}{\mathbf{14.66}}$ & $\textcolor{ForestGreen}{\mathbf{84.10}}$ & $\textcolor{ForestGreen}{\mathbf{25.395}}$ & $25.933$ & $\textcolor{ForestGreen}{\mathbf{1.119}}$ & $\textcolor{ForestGreen}{\mathbf{0.433}}$ & $\textcolor{ForestGreen}{\mathbf{0.321}}$ \\
 & CACFG*   & $16.16$ & $4.40$ & $16.57$ & $81.79$ & $\textcolor{ForestGreen}{\mathbf{25.395}}$ & $\textcolor{ForestGreen}{\mathbf{25.979}}$ & $1.094$ & $0.426$ & $0.365$ \\
\cmidrule(lr){1-11}
\multirow[c]{3}{*}{15}
 & CFG      & $17.78$ & $4.75$ & $21.95$ & $78.14$ & $25.491$ & $26.214$ & $0.994$ & $0.391$ & $0.424$ \\
 \noalign{\vskip -1.5pt}
 \cmidrule[0.1pt](lr){2-11}
 \noalign{\vskip -1.5pt}
 & CACFG    & $\textcolor{ForestGreen}{\mathbf{16.52}}$ & $\textcolor{ForestGreen}{\mathbf{4.42}}$ & $\textcolor{ForestGreen}{\mathbf{14.80}}$ & $\textcolor{ForestGreen}{\mathbf{84.60}}$ & $\textcolor{ForestGreen}{\mathbf{25.540}}$ & $26.175$ & $\textcolor{ForestGreen}{\mathbf{1.063}}$ & $\textcolor{ForestGreen}{\mathbf{0.413}}$ & $\textcolor{ForestGreen}{\mathbf{0.317}}$ \\
 & CACFG*   & $17.55$ & $4.57$ & $19.06$ & $80.02$ & $25.505$ & $\textcolor{ForestGreen}{\mathbf{26.264}}$ & $1.023$ & $0.400$ & $0.397$ \\
\bottomrule
\end{tabular}%
}
\end{table}

\subsection{Curvature-aware upper bound blending parameter} \label{Apdx:Blending parameter}
The blending parameter $\gamma$ interpolates between the two extremes of the Gaussian hypothesis underlying $U_\gamma(t)$: at $\gamma = 0$ the control set assumes the hypothesis holds exactly, while at $\gamma = 1$ it retains only the raw upper bound $r_u$ (\ref{eq:cacfg control set}). We sweep $\gamma \in \{0.0, 0.1, \dots, 1.0\}$ for CACFG on ImageNet, for both SD1.5 (Table~\ref{tab:SD1.5 blending}) and SDXL (Table~\ref{tab:SDXL blending}). A consistent pattern emerges across both models: loosening the control set steadily improves prompt alignment and human-preference scores (CLIP, HPSv2) and steadily increases saturation, while the remaining fidelity metrics (FID, NIQE, BRISQUE, Density, Coverage) improve sharply over the first few steps of the sweep and then plateau. No single value of $\gamma$ is simultaneously optimal across every metric, so selecting $\gamma$ requires an explicit choice about how to balance prompt alignment against the residual fidelity cost of moving away from $\gamma=0.0$; we describe that choice, and its cost, for each model below.

For SD1.5, we select $\gamma = 0.5$. At this value, six of the nine metrics (NIQE, BRISQUE, Density, Coverage, CLIP, and HPSv2) sit within $1\%$ of the best value achieved anywhere in the sweep, and FID ($14.97$) is close to its sweep-best ($14.53$) and far below the $\gamma=0.0$ value ($16.95$). Beyond $\gamma=0.5$, prompt-alignment and preference gains diminish sharply: CLIP and HPSv2 each move by less than $0.15\%$ between $\gamma=0.5$ and $\gamma=1.0$. Over the same range, the two metrics that continue to degrade past $\gamma=0.5$ worsen further: TReS, which peaks early at $\gamma=0.1$ and declines steadily thereafter, falling below even its $\gamma=0.0$ value; and saturation, which increases monotonically with $\gamma$ throughout the sweep. We therefore treat $\gamma=0.5$ as close to the point of diminishing returns for SD1.5: it captures nearly all of the available prompt-alignment gain without paying the additional TReS and saturation cost of continuing to $\gamma=1.0$.

For SDXL, we instead select $\gamma = 1.0$. Unlike SD1.5, fidelity metrics for SDXL plateau from roughly $\gamma=0.3$--$0.4$ onward and remain within a small margin of their sweep-best all the way to $\gamma=1.0$. FID, TReS, Density, and Coverage are each within $2\%$ of their respective optima at $\gamma=1.0$, while CLIP and HPSv2 are both maximised at $\gamma=1.0$. This is consistent with our hypothesis that SDXL, as a higher-capacity model, is more robust to the manifold departure the curvature-aware correction addresses, and so tolerates a looser control set with comparatively little fidelity cost. The corresponding cost is a mild increase in BRISQUE ($2.7\%$ above its sweep-best) and saturation, which reaches its highest value in the sweep at $\gamma=1.0$; we accept this cost given the consistent, and in several cases maximal, gains elsewhere.

\begin{table}[h]
\centering
\caption{Image quality for SD1.5 \citep{StableDiffusion_Rombach2022} samples generated for ImageNet \citep{ImageNetBenchmark_Russakovsky2015} dataset using CACFG across varying curvature-aware blending parameter values.}
\label{tab:SD1.5 blending}
\resizebox{\textwidth}{!}{%
\small
\begin{tabular}{cccccccc cc}
\toprule
$\boldsymbol{\gamma}$ & \textbf{FID} $\downarrow$ & \textbf{NIQE} $\downarrow$ & \textbf{BRISQUE} $\downarrow$ & \textbf{TReS} $\uparrow$ & \textbf{CLIP} $\uparrow$ & \textbf{HPSv2} $\uparrow$ & \textbf{Density} $\uparrow$ & \textbf{Coverage} $\uparrow$ & \textbf{Saturation} $\downarrow$ \\
\midrule
   $0.0$ & $16.95$ & $5.44$ & $23.59$ & $85.15$ & $24.835$ & $23.844$ & $0.938$ & $0.395$ & $0.194$ \\
   $0.1$ & $14.53$ & $4.56$ & $17.36$ & $86.28$ & $25.159$ & $25.069$ & $1.111$ & $0.444$ & $0.236$ \\
   $0.2$ & $14.54$ & $4.32$ & $15.28$ & $85.65$ & $25.199$ & $25.403$ & $1.142$ & $0.446$ & $0.266$ \\
   $0.3$ & $14.72$ & $4.24$ & $14.64$ & $84.96$ & $25.233$ & $25.550$ & $1.140$ & $0.448$ & $0.289$ \\
   $0.4$ & $14.89$ & $4.22$ & $14.50$ & $84.29$ & $25.242$ & $25.621$ & $1.140$ & $0.447$ & $0.305$ \\
   $0.5$ & $14.97$ & $4.22$ & $14.59$ & $83.80$ & $25.243$ & $25.653$ & $1.140$ & $0.445$ & $0.316$ \\
   $0.6$ & $15.03$ & $4.23$ & $14.75$ & $83.48$ & $25.244$ & $25.669$ & $1.145$ & $0.447$ & $0.323$ \\
   $0.7$ & $15.11$ & $4.24$ & $14.93$ & $83.28$ & $25.244$ & $25.672$ & $1.138$ & $0.445$ & $0.328$ \\
   $0.8$ & $15.10$ & $4.25$ & $15.04$ & $83.12$ & $25.247$ & $25.673$ & $1.137$ & $0.444$ & $0.331$ \\
   $0.9$ & $15.14$ & $4.26$ & $15.14$ & $83.04$ & $25.255$ & $25.681$ & $1.137$ & $0.445$ & $0.333$ \\
   $1.0$ & $15.09$ & $4.26$ & $15.20$ & $82.97$ & $25.255$ & $25.690$ & $1.133$ & $0.445$ & $0.335$ \\
\bottomrule
\end{tabular}%
}
\end{table}

\begin{table}[h]
\centering
\caption{Image quality for SDXL \citep{SDXL_Podell2024} samples generated for ImageNet \citep{ImageNetBenchmark_Russakovsky2015} dataset using CACFG across varying curvature-aware blending parameter values.}
\label{tab:SDXL blending}
\resizebox{\textwidth}{!}{%
\small
\begin{tabular}{cccccccc cc}
\toprule
$\boldsymbol{\gamma}$ & \textbf{FID} $\downarrow$ & \textbf{NIQE} $\downarrow$ & \textbf{BRISQUE} $\downarrow$ & \textbf{TReS} $\uparrow$ & \textbf{CLIP} $\uparrow$ & \textbf{HPSv2} $\uparrow$ & \textbf{Density} $\uparrow$ & \textbf{Coverage} $\uparrow$ & \textbf{Saturation} $\downarrow$ \\
\midrule
   $0.0$ & $35.84$ & $4.96$ & $26.93$ & $78.06$ & $24.751$ & $23.531$ & $0.863$ & $0.143$ & $0.139$ \\
   $0.1$ & $29.85$ & $4.48$ & $26.04$ & $81.82$ & $25.213$ & $25.409$ & $1.073$ & $0.174$ & $0.165$ \\
   $0.2$ & $29.22$ & $4.42$ & $26.05$ & $83.37$ & $25.274$ & $25.924$ & $1.133$ & $0.183$ & $0.182$ \\
   $0.3$ & $29.09$ & $4.43$ & $26.21$ & $83.96$ & $25.291$ & $26.114$ & $1.147$ & $0.184$ & $0.193$ \\
   $0.4$ & $29.02$ & $4.44$ & $26.45$ & $84.22$ & $25.322$ & $26.234$ & $1.117$ & $0.181$ & $0.201$ \\
   $0.5$ & $29.02$ & $4.47$ & $26.53$ & $84.31$ & $25.361$ & $26.294$ & $1.140$ & $0.184$ & $0.206$ \\
   $0.6$ & $29.06$ & $4.48$ & $26.59$ & $84.38$ & $25.342$ & $26.319$ & $1.157$ & $0.187$ & $0.209$ \\
   $0.7$ & $29.09$ & $4.48$ & $26.70$ & $84.41$ & $25.350$ & $26.339$ & $1.159$ & $0.186$ & $0.211$ \\
   $0.8$ & $29.10$ & $4.49$ & $26.66$ & $84.43$ & $25.356$ & $26.353$ & $1.159$ & $0.185$ & $0.213$ \\
   $0.9$ & $29.10$ & $4.50$ & $26.76$ & $84.40$ & $25.358$ & $26.370$ & $1.159$ & $0.187$ & $0.214$ \\
   $1.0$ & $29.12$ & $4.49$ & $26.75$ & $84.39$ & $25.369$ & $26.380$ & $1.158$ & $0.186$ & $0.215$ \\
\bottomrule
\end{tabular}%
}
\end{table}

\clearpage
\section{Curvature-Aware Experiments}
\subsection{Metric plots} \label{Apdx:Metric plots}
Here we provide the metric plots of SD1.5 \citep{StableDiffusion_Rombach2022} and SDXL \citep{SDXL_Podell2024} diffusion models on the ImageNet \citep{ImageNetBenchmark_Russakovsky2015} and COCO \citep{COCO_Lin2014} generation tasks. We discuss these results in Section~\ref{sec:results}.

\begin{figure}[h]
    \centering
    \begin{subfigure}{0.53\linewidth}
        \centering
        \includegraphics[width=\linewidth]{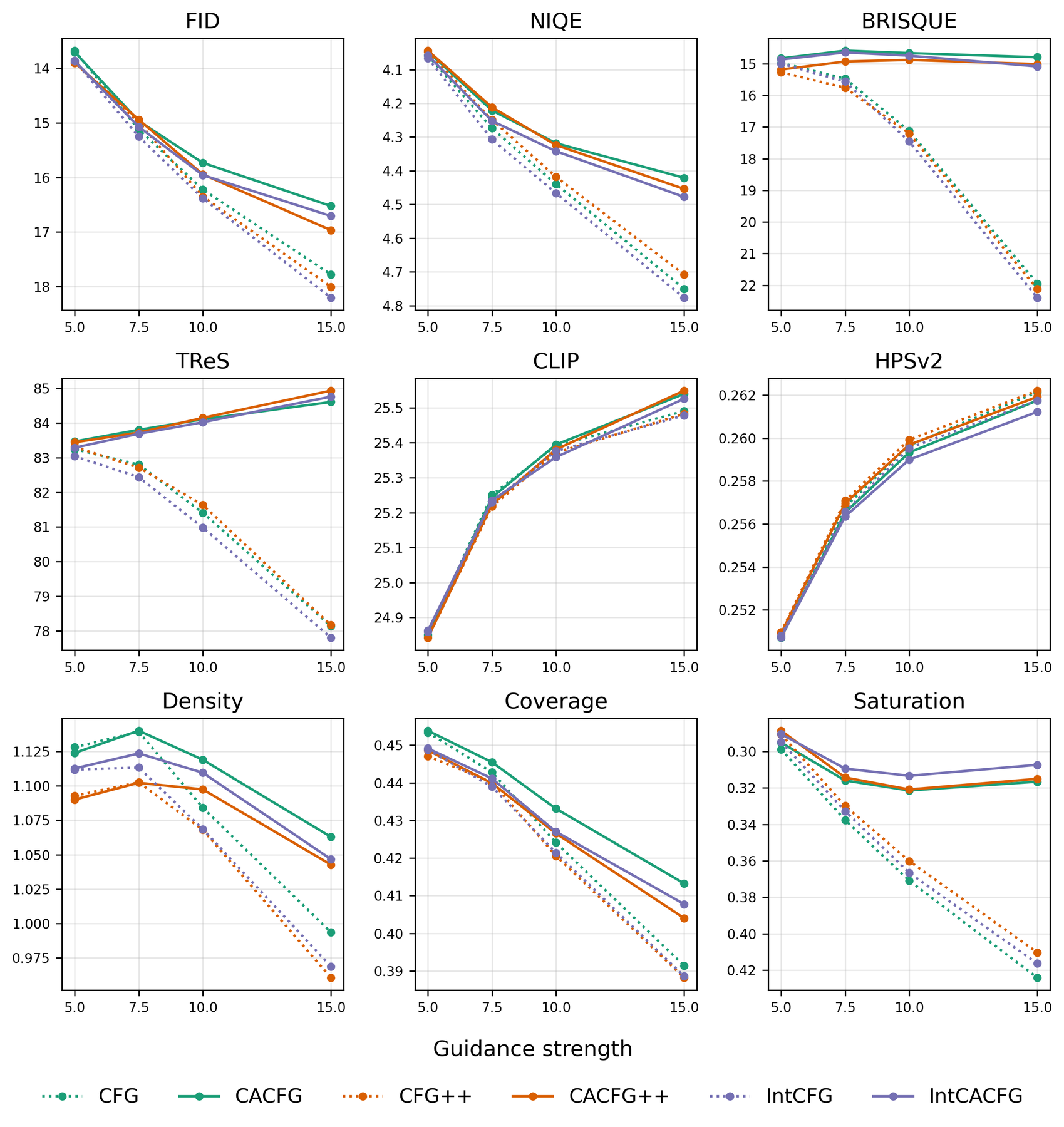}
        \caption{ImageNet}
    \end{subfigure}

    \vspace{1em}

    \begin{subfigure}{0.53\linewidth}
        \centering
        \includegraphics[width=\linewidth]{arxiv_images/MetricPlots/SD1_5_COCO.png}
        \caption{COCO}
    \end{subfigure}
    \caption{Performance of CFG, CFG++, and IntCFG with (solid line) and without (dotted line) our curvature-aware correction using SD1.5 on (a) ImageNet and (b) COCO. Some metric axes have been flipped so that the improvement direction is vertically up the page. Note at larger guidance strengths FID, NIQE, BRISQUE, TReS, Density, Coverage, and CLIP are all improved by our curvature-aware correction. We recommend zooming in for easier viewing.}
    \label{fig:Metric plot sd1.5}
\end{figure}

\begin{figure}[h]
    \centering
    \begin{subfigure}{0.53\linewidth}
        \centering
        \includegraphics[width=\linewidth]{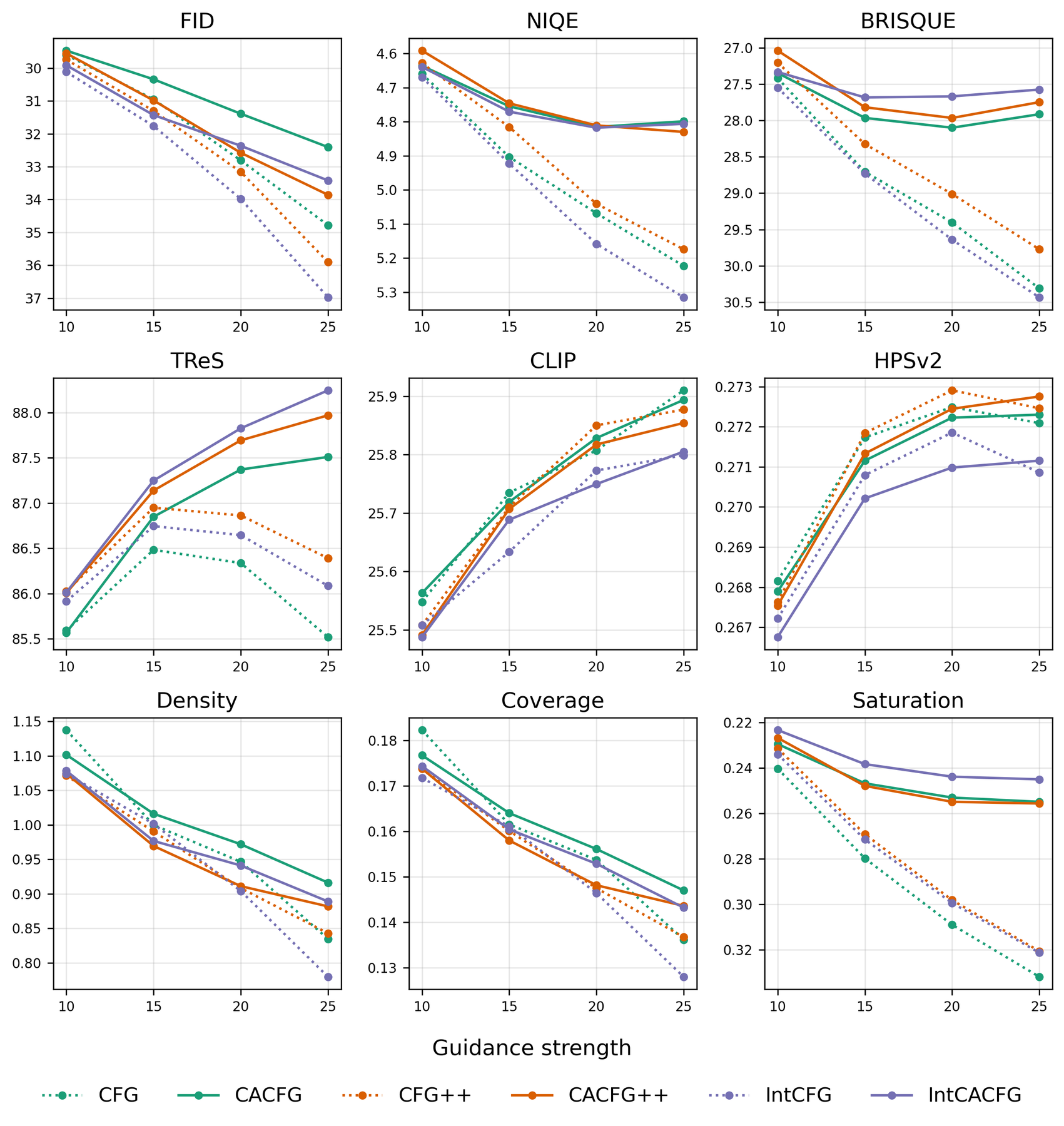}
        \caption{ImageNet}
    \end{subfigure}

    \vspace{1em}

    \begin{subfigure}{0.53\linewidth}
        \centering
        \includegraphics[width=\linewidth]{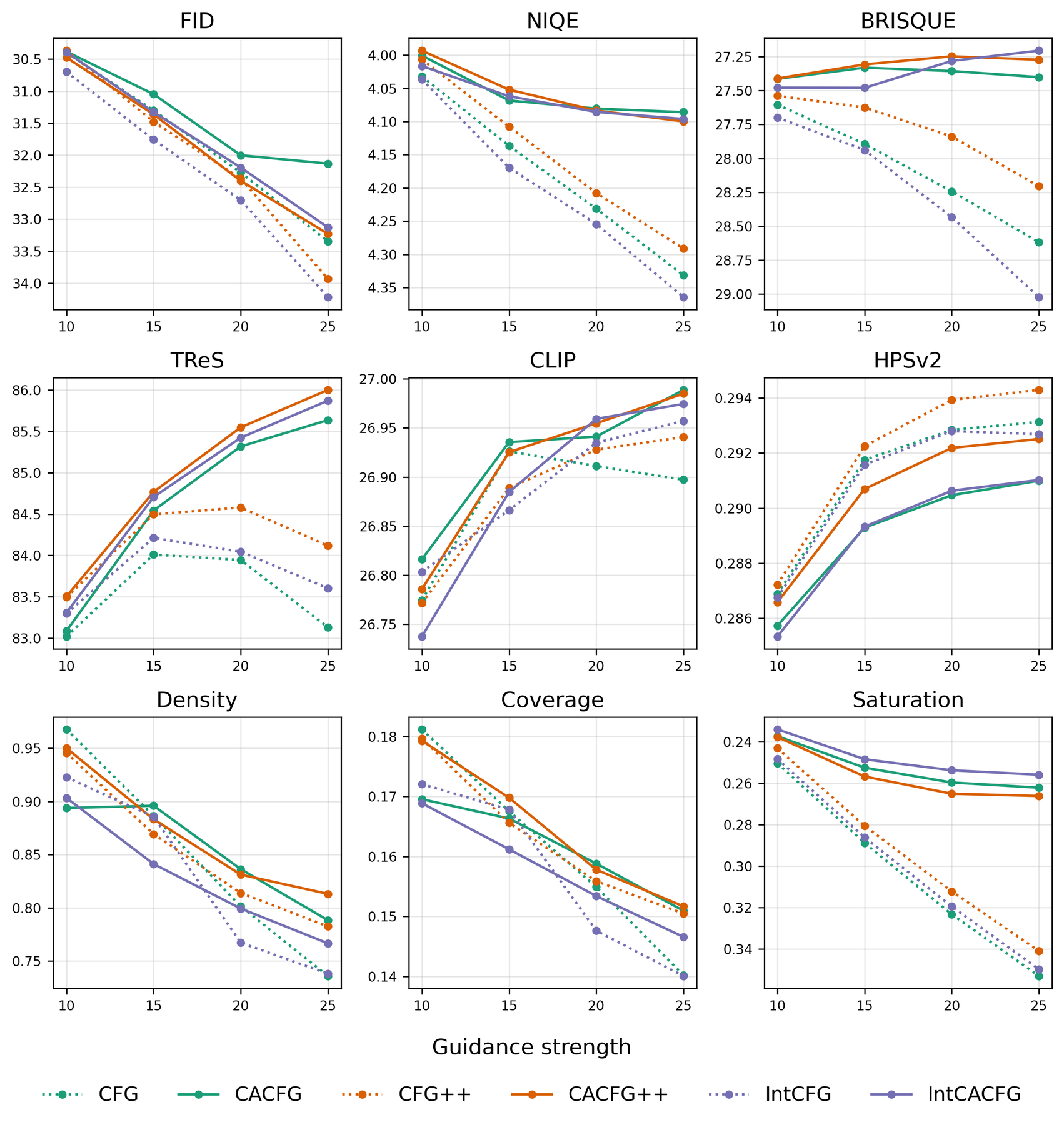}
        \caption{COCO}
    \end{subfigure}
    \caption{Performance of CFG, CFG++, and IntCFG with (solid line) and without (dotted line) our curvature-aware correction using SDXL on (a) ImageNet and (b) COCO. Some metric axes have been flipped so that the improvement direction is vertically up the page. Note at larger guidance strengths FID, NIQE, BRISQUE, TReS, Density, and Coverage are all improved by our curvature-aware correction. We recommend zooming in for easier viewing.}
    \label{fig:Metric plot sdxl}
\end{figure}

\clearpage
\subsection{Additional samples} \label{Apdx:Additional samples}
Here we present additional curvature-aware diffusion samples across varying guidance methods, guidance strengths, and models.
\subsubsection{SD1.5}

\begin{figure}[h]
    \centering
    \begin{subfigure}{\linewidth}
    \centering
    \resizebox{\linewidth}{!}{%
    \begin{tabular}{@{}c@{\hspace{2pt}}c@{}c@{}c@{}c@{}c@{}c@{}}
        \raisebox{-0.5\height}{\rotatebox{90}{\small CFG}} & \raisebox{-0.5\height}{\includegraphics[width=0.15\linewidth]{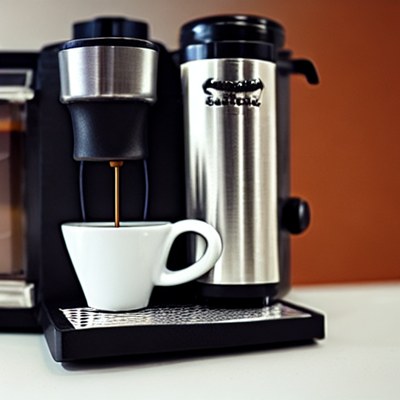}} & \raisebox{-0.5\height}{\includegraphics[width=0.15\linewidth]{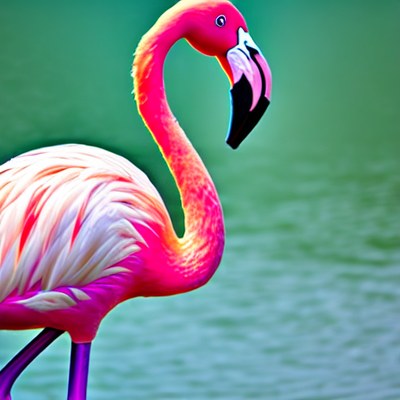}} & \raisebox{-0.5\height}{\includegraphics[width=0.15\linewidth]{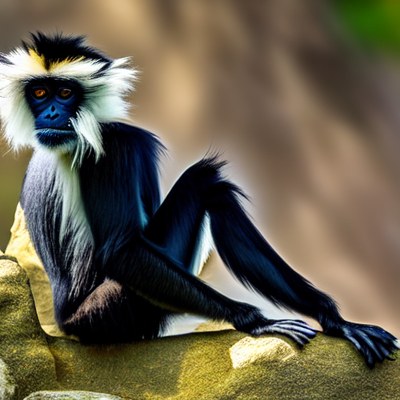}} & \raisebox{-0.5\height}{\includegraphics[width=0.15\linewidth]{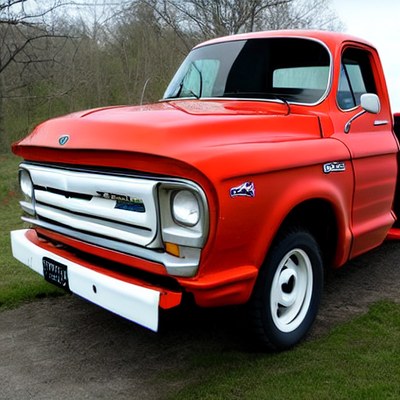}} & \raisebox{-0.5\height}{\includegraphics[width=0.15\linewidth]{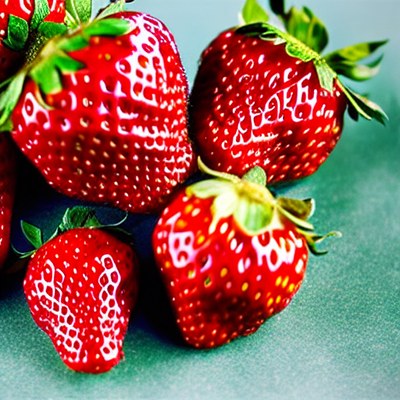}} & \raisebox{-0.5\height}{\includegraphics[width=0.15\linewidth]{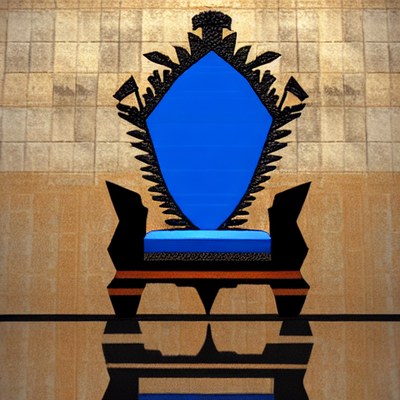}} \\
        \raisebox{-0.5\height}{\rotatebox{90}{\small CACFG}} & \raisebox{-0.5\height}{\includegraphics[width=0.15\linewidth]{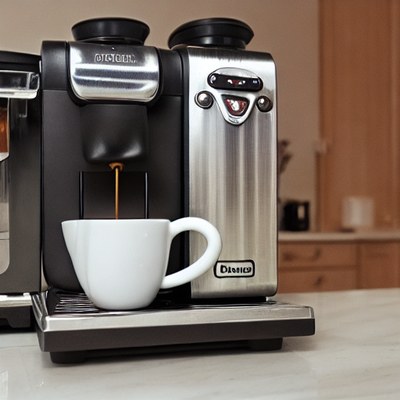}} & \raisebox{-0.5\height}{\includegraphics[width=0.15\linewidth]{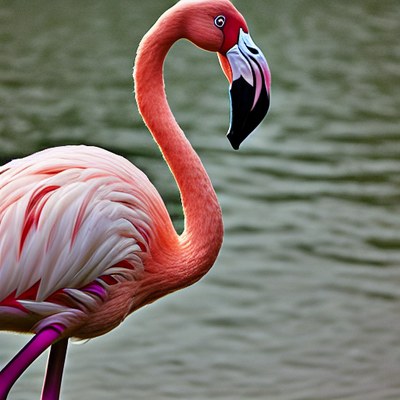}} & \raisebox{-0.5\height}{\includegraphics[width=0.15\linewidth]{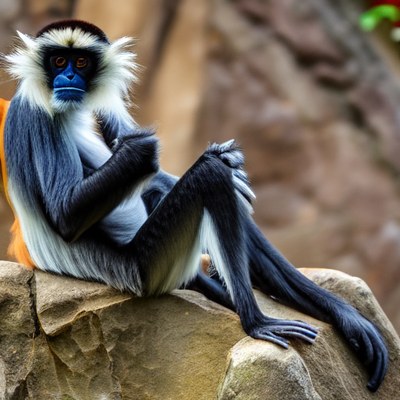}} & \raisebox{-0.5\height}{\includegraphics[width=0.15\linewidth]{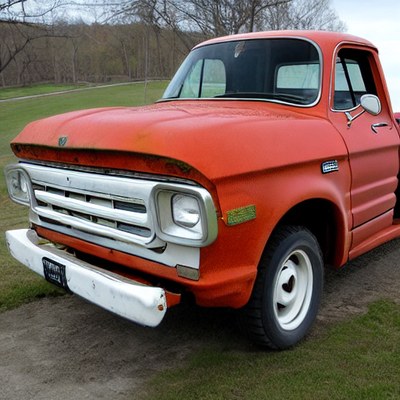}} & \raisebox{-0.5\height}{\includegraphics[width=0.15\linewidth]{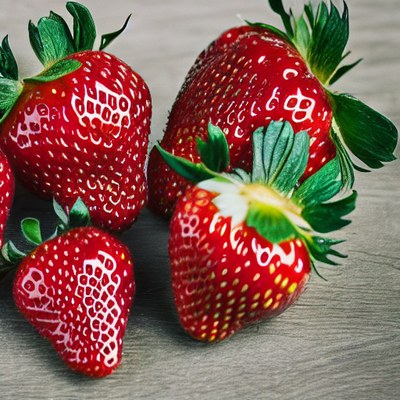}} & \raisebox{-0.5\height}{\includegraphics[width=0.15\linewidth]{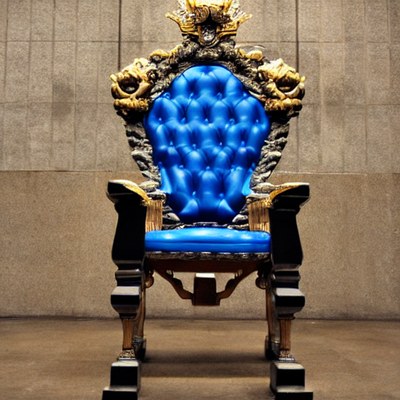}}
    \end{tabular}
    }
    \caption{CFG vs CACFG (ours)}
    \end{subfigure}
    \vspace{1.5em}

    \begin{subfigure}{\linewidth}
        \centering
        \resizebox{\linewidth}{!}{%
        \begin{tabular}{@{}c@{\hspace{2pt}}c@{}c@{}c@{}c@{}c@{}c@{}}
            \raisebox{-0.5\height}{\rotatebox{90}{\small CFG++}} & \raisebox{-0.5\height}{\includegraphics[width=0.15\linewidth]{arxiv_images/SD1_5_ImageNet_Images/SD1_5_ImNe_15_CFGpp_espresso_maker.jpg}} & \raisebox{-0.5\height}{\includegraphics[width=0.15\linewidth]{arxiv_images/SD1_5_ImageNet_Images/SD1_5_ImNe_15_CFGpp_flamingo.jpg}} & \raisebox{-0.5\height}{\includegraphics[width=0.15\linewidth]{arxiv_images/SD1_5_ImageNet_Images/SD1_5_ImNe_15_CFGpp_langur.jpg}} & \raisebox{-0.5\height}{\includegraphics[width=0.15\linewidth]{arxiv_images/SD1_5_ImageNet_Images/SD1_5_ImNe_15_CFGpp_pickup_pickup_truck.jpg}} & \raisebox{-0.5\height}{\includegraphics[width=0.15\linewidth]{arxiv_images/SD1_5_ImageNet_Images/SD1_5_ImNe_15_CFGpp_strawberry.jpg}} & \raisebox{-0.5\height}{\includegraphics[width=0.15\linewidth]{arxiv_images/SD1_5_ImageNet_Images/SD1_5_ImNe_15_CFGpp_throne.jpg}} \\
            \raisebox{-0.5\height}{\rotatebox{90}{\small CACFG++}} & \raisebox{-0.5\height}{\includegraphics[width=0.15\linewidth]{arxiv_images/SD1_5_ImageNet_Images/SD1_5_ImNe_15_CACFGpp_espresso_maker.jpg}} & \raisebox{-0.5\height}{\includegraphics[width=0.15\linewidth]{arxiv_images/SD1_5_ImageNet_Images/SD1_5_ImNe_15_CACFGpp_flamingo.jpg}} & \raisebox{-0.5\height}{\includegraphics[width=0.15\linewidth]{arxiv_images/SD1_5_ImageNet_Images/SD1_5_ImNe_15_CACFGpp_langur.jpg}} & \raisebox{-0.5\height}{\includegraphics[width=0.15\linewidth]{arxiv_images/SD1_5_ImageNet_Images/SD1_5_ImNe_15_CACFGpp_pickup_pickup_truck.jpg}} & \raisebox{-0.5\height}{\includegraphics[width=0.15\linewidth]{arxiv_images/SD1_5_ImageNet_Images/SD1_5_ImNe_15_CACFGpp_strawberry.jpg}} & \raisebox{-0.5\height}{\includegraphics[width=0.15\linewidth]{arxiv_images/SD1_5_ImageNet_Images/SD1_5_ImNe_15_CACFGpp_throne.jpg}}
        \end{tabular}
        }
        \caption{CFG++ vs CACFG++ (ours)}
    \end{subfigure}
    \vspace{1.5em}

    \begin{subfigure}{\linewidth}
        \centering
        \resizebox{\linewidth}{!}{%
        \begin{tabular}{@{}c@{\hspace{2pt}}c@{}c@{}c@{}c@{}c@{}c@{}}
            \raisebox{-0.5\height}{\rotatebox{90}{\small IntCFG}} & \raisebox{-0.5\height}{\includegraphics[width=0.15\linewidth]{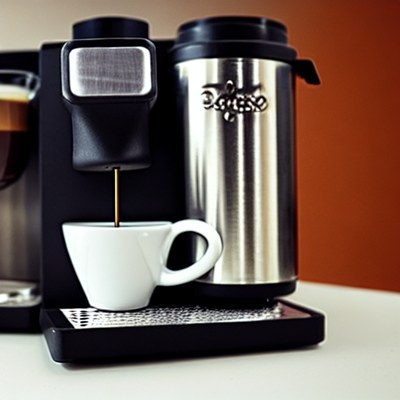}} & \raisebox{-0.5\height}{\includegraphics[width=0.15\linewidth]{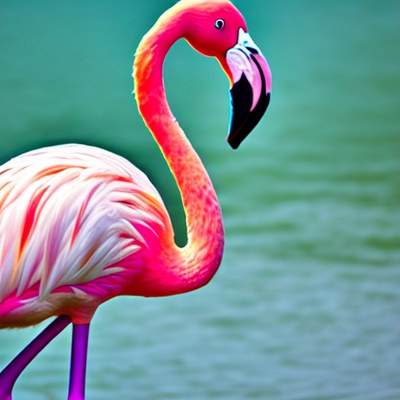}} & \raisebox{-0.5\height}{\includegraphics[width=0.15\linewidth]{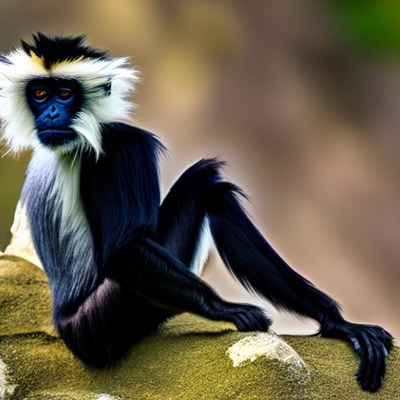}} & \raisebox{-0.5\height}{\includegraphics[width=0.15\linewidth]{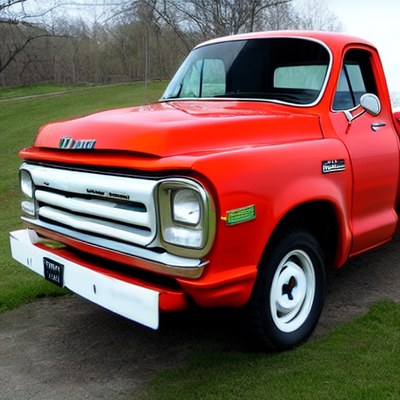}} & \raisebox{-0.5\height}{\includegraphics[width=0.15\linewidth]{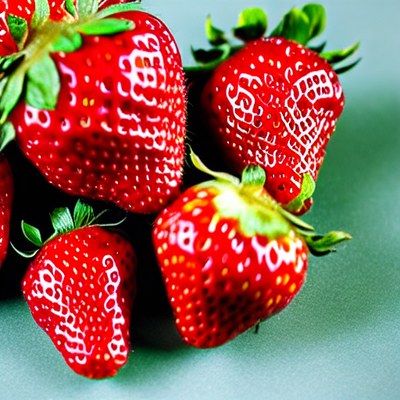}} & \raisebox{-0.5\height}{\includegraphics[width=0.15\linewidth]{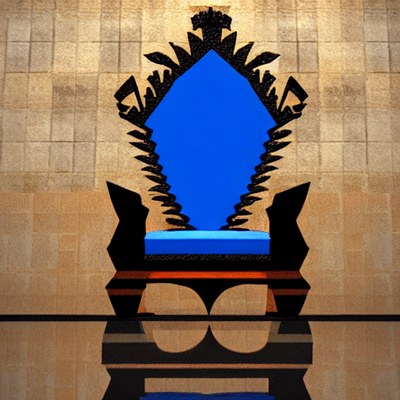}} \\
            \raisebox{-0.5\height}{\rotatebox{90}{\small IntCACFG}} & \raisebox{-0.5\height}{\includegraphics[width=0.15\linewidth]{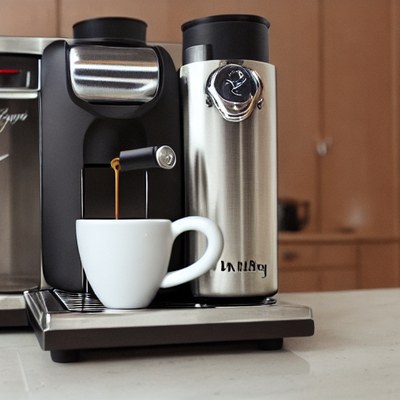}} & \raisebox{-0.5\height}{\includegraphics[width=0.15\linewidth]{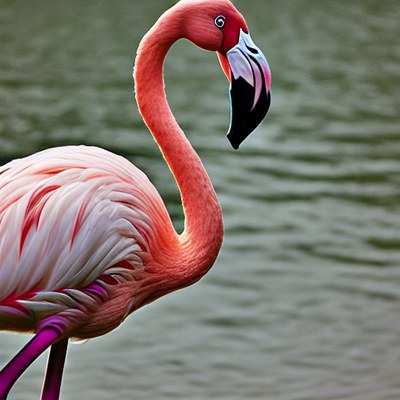}} & \raisebox{-0.5\height}{\includegraphics[width=0.15\linewidth]{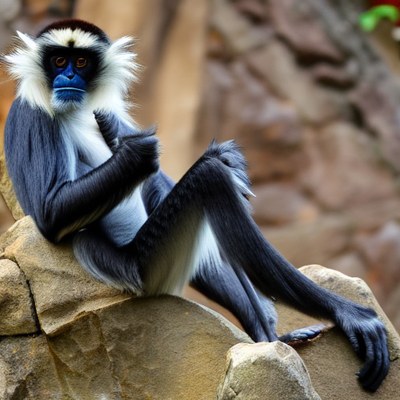}} & \raisebox{-0.5\height}{\includegraphics[width=0.15\linewidth]{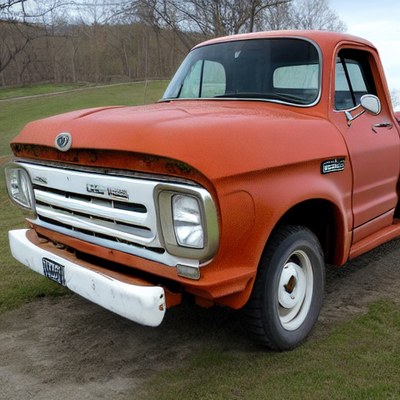}} & \raisebox{-0.5\height}{\includegraphics[width=0.15\linewidth]{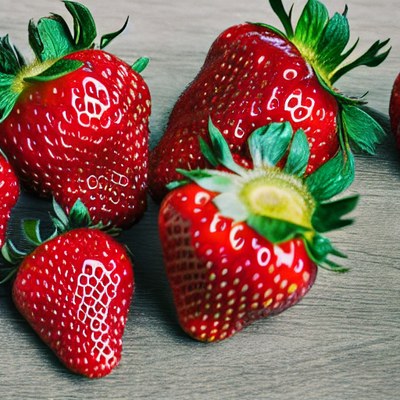}} & \raisebox{-0.5\height}{\includegraphics[width=0.15\linewidth]{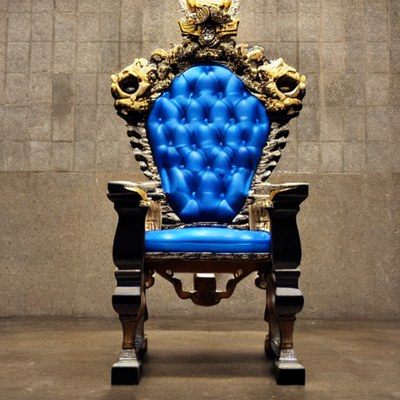}}
        \end{tabular}
        }
        \caption{IntCFG vs IntCACFG (ours)}
    \end{subfigure}
    \caption{CFG \citep{ClassifierFreeGuidance_Ho2022}, CFG++ \citep{CFG++_Chung2025}, and IntCFG \citep{IntCFG_Kynkaanniemi2024} visualised against their curvature-aware variants (ours): CACFG, CACFG++, IntCACFG. Samples are generated using SD1.5 \citep{StableDiffusion_Rombach2022} targeting the ImageNet \citep{ImageNetBenchmark_Russakovsky2015} dataset with a guidance strength of $\omega=15$.}
    \label{fig:SD1.5 ImageNet Guidance Grid}
\end{figure}

\begin{figure}[h]
    \centering
    \resizebox{0.7\linewidth}{!}{%
    \begin{tabular}{@{}c@{\hspace{2pt}}c@{}c@{}c@{}c@{}}
        & \makebox[0.15\linewidth][c]{\small $\omega = 5$} & \makebox[0.15\linewidth][c]{\small $\omega = 7.5$} & \makebox[0.15\linewidth][c]{\small $\omega = 10$} & \makebox[0.15\linewidth][c]{\small $\omega = 15$} \\
        \noalign{\vskip 4pt}
        \raisebox{-0.5\height}{\rotatebox{90}{\small CFG}} & \raisebox{-0.5\height}{\includegraphics[width=0.15\linewidth]{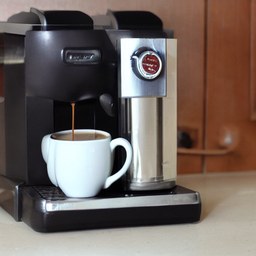}} & \raisebox{-0.5\height}{\includegraphics[width=0.15\linewidth]{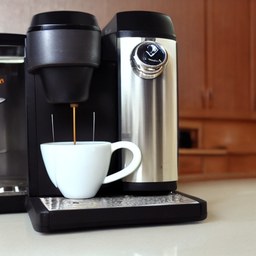}} & \raisebox{-0.5\height}{\includegraphics[width=0.15\linewidth]{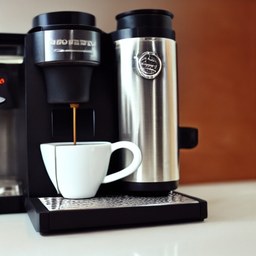}} & \raisebox{-0.5\height}{\includegraphics[width=0.15\linewidth]{arxiv_images/SD1_5_ImageNet_Images/SD1_5_ImNe_15_CFGpp_espresso_maker.jpg}} \\
        \raisebox{-0.5\height}{\rotatebox{90}{\small CACFG}} & \raisebox{-0.5\height}{\includegraphics[width=0.15\linewidth]{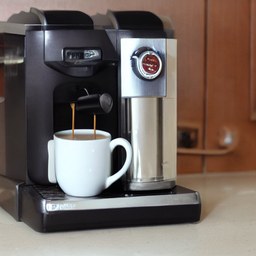}} & \raisebox{-0.5\height}{\includegraphics[width=0.15\linewidth]{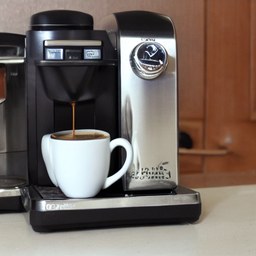}} & \raisebox{-0.5\height}{\includegraphics[width=0.15\linewidth]{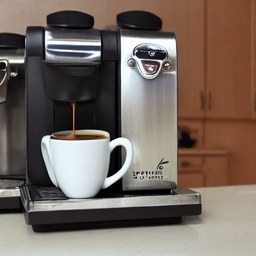}} & \raisebox{-0.5\height}{\includegraphics[width=0.15\linewidth]{arxiv_images/SD1_5_ImageNet_Images/SD1_5_ImNe_15_CACFGpp_espresso_maker.jpg}} \\
        \noalign{\vskip 4pt}
        \raisebox{-0.5\height}{\rotatebox{90}{\small CFG}} & \raisebox{-0.5\height}{\includegraphics[width=0.15\linewidth]{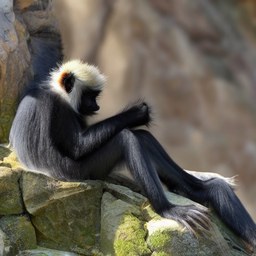}} & \raisebox{-0.5\height}{\includegraphics[width=0.15\linewidth]{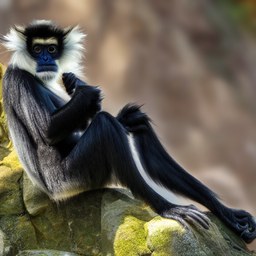}} & \raisebox{-0.5\height}{\includegraphics[width=0.15\linewidth]{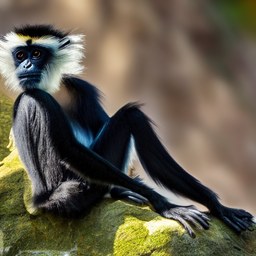}} & \raisebox{-0.5\height}{\includegraphics[width=0.15\linewidth]{arxiv_images/SD1_5_ImageNet_Images/SD1_5_ImNe_15_CFGpp_langur.jpg}} \\
        \raisebox{-0.5\height}{\rotatebox{90}{\small CACFG}} & \raisebox{-0.5\height}{\includegraphics[width=0.15\linewidth]{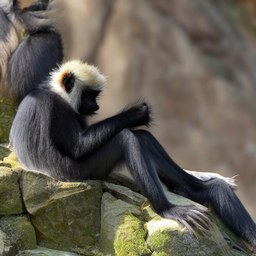}} & \raisebox{-0.5\height}{\includegraphics[width=0.15\linewidth]{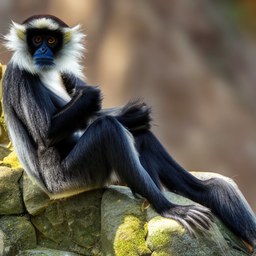}} & \raisebox{-0.5\height}{\includegraphics[width=0.15\linewidth]{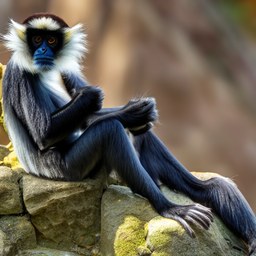}} & \raisebox{-0.5\height}{\includegraphics[width=0.15\linewidth]{arxiv_images/SD1_5_ImageNet_Images/SD1_5_ImNe_15_CACFGpp_langur.jpg}} \\
        \noalign{\vskip 4pt}
        \raisebox{-0.5\height}{\rotatebox{90}{\small CFG}} & \raisebox{-0.5\height}{\includegraphics[width=0.15\linewidth]{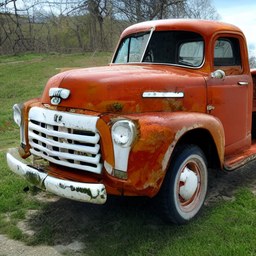}} & \raisebox{-0.5\height}{\includegraphics[width=0.15\linewidth]{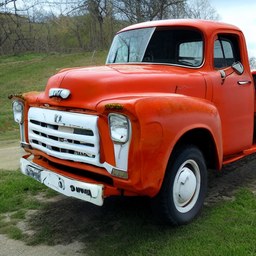}} & \raisebox{-0.5\height}{\includegraphics[width=0.15\linewidth]{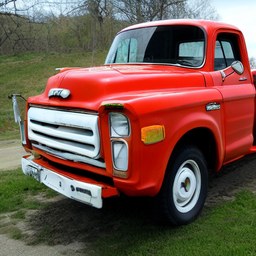}} & \raisebox{-0.5\height}{\includegraphics[width=0.15\linewidth]{arxiv_images/SD1_5_ImageNet_Images/SD1_5_ImNe_15_CFGpp_pickup_pickup_truck.jpg}} \\
        \raisebox{-0.5\height}{\rotatebox{90}{\small CACFG}} & \raisebox{-0.5\height}{\includegraphics[width=0.15\linewidth]{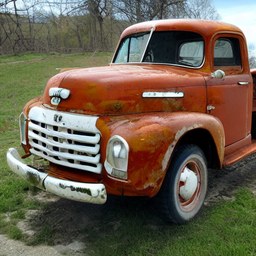}} & \raisebox{-0.5\height}{\includegraphics[width=0.15\linewidth]{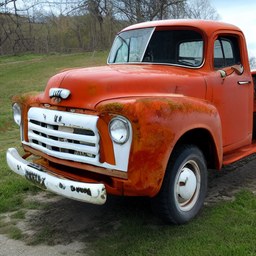}} & \raisebox{-0.5\height}{\includegraphics[width=0.15\linewidth]{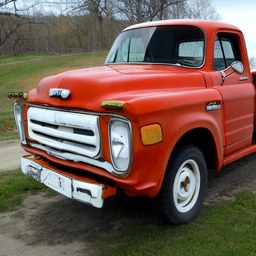}} & \raisebox{-0.5\height}{\includegraphics[width=0.15\linewidth]{arxiv_images/SD1_5_ImageNet_Images/SD1_5_ImNe_15_CACFGpp_pickup_pickup_truck.jpg}} \\
        \noalign{\vskip 4pt}
        \raisebox{-0.5\height}{\rotatebox{90}{\small CFG}} & \raisebox{-0.5\height}{\includegraphics[width=0.15\linewidth]{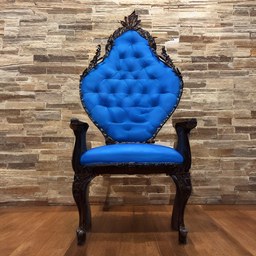}} & \raisebox{-0.5\height}{\includegraphics[width=0.15\linewidth]{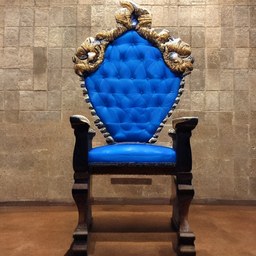}} & \raisebox{-0.5\height}{\includegraphics[width=0.15\linewidth]{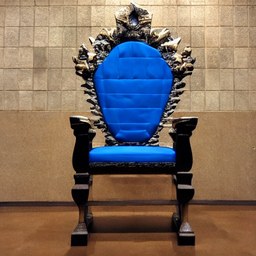}} & \raisebox{-0.5\height}{\includegraphics[width=0.15\linewidth]{arxiv_images/SD1_5_ImageNet_Images/SD1_5_ImNe_15_CFGpp_throne.jpg}} \\
        \raisebox{-0.5\height}{\rotatebox{90}{\small CACFG}} & \raisebox{-0.5\height}{\includegraphics[width=0.15\linewidth]{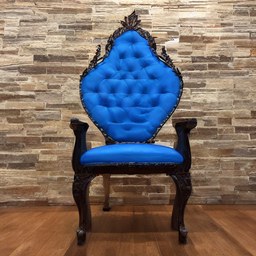}} & \raisebox{-0.5\height}{\includegraphics[width=0.15\linewidth]{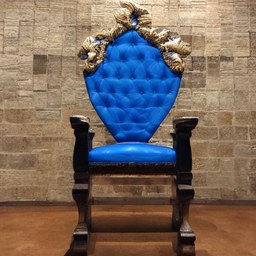}} & \raisebox{-0.5\height}{\includegraphics[width=0.15\linewidth]{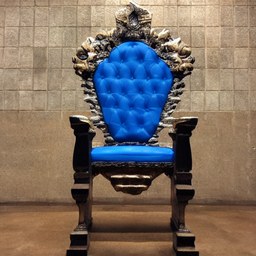}} & \raisebox{-0.5\height}{\includegraphics[width=0.15\linewidth]{arxiv_images/SD1_5_ImageNet_Images/SD1_5_ImNe_15_CACFGpp_throne.jpg}}
    \end{tabular}
    }
    \caption{CFG \citep{ClassifierFreeGuidance_Ho2022} and CACFG (ours) samples generated using SD1.5 \citep{StableDiffusion_Rombach2022} targeting the ImageNet \citep{ImageNetBenchmark_Russakovsky2015} dataset. Guidance scale varies across $\omega = 5, 7.5, 10, 15$.}
    \label{fig:SD1.5 ImageNet Guidance Strength}
\end{figure}

\begin{figure}[h]
    \centering
    \begin{subfigure}{\linewidth}
        \centering
        \resizebox{\linewidth}{!}{%
        \begin{tabular}{@{}c@{\hspace{2pt}}c@{}c@{}c@{}c@{}c@{}c@{}}
            \raisebox{-0.5\height}{\rotatebox{90}{\small CFG}} & \raisebox{-0.5\height}{\includegraphics[width=0.15\linewidth]{arxiv_images/SD1_5_COCO_Images/SD1_5_COCO_15_CFGpp_A_black_cat_wearing_a_little_party_hat.jpg}} & \raisebox{-0.5\height}{\includegraphics[width=0.15\linewidth]{arxiv_images/SD1_5_COCO_Images/SD1_5_COCO_15_CFGpp_A_clock_sitting_in_a_yard_behind_a_fence.jpg}} & \raisebox{-0.5\height}{\includegraphics[width=0.15\linewidth]{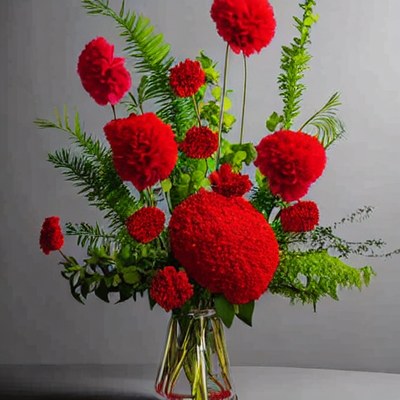}} & \raisebox{-0.5\height}{\includegraphics[width=0.15\linewidth]{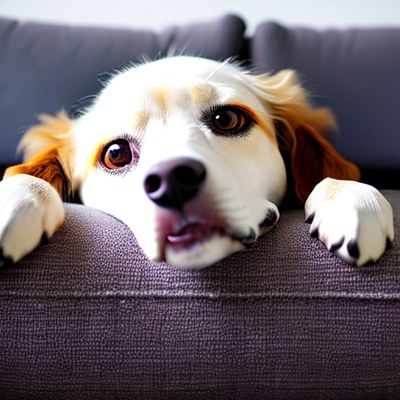}} & \raisebox{-0.5\height}{\includegraphics[width=0.15\linewidth]{arxiv_images/SD1_5_COCO_Images/SD1_5_COCO_15_CFGpp_A_picture_of_a_painting_and_a_couch.jpg}} & \raisebox{-0.5\height}{\includegraphics[width=0.15\linewidth]{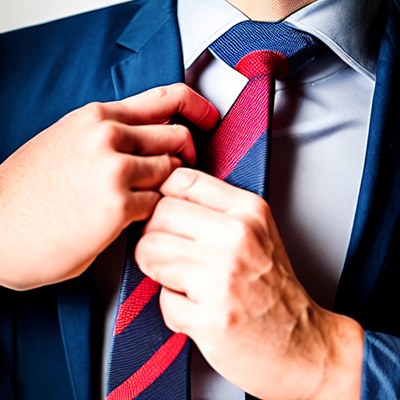}} \\
            \raisebox{-0.5\height}{\rotatebox{90}{\small CACFG}} & \raisebox{-0.5\height}{\includegraphics[width=0.15\linewidth]{arxiv_images/SD1_5_COCO_Images/SD1_5_COCO_15_CACFGpp_A_black_cat_wearing_a_little_party_hat.jpg}} & \raisebox{-0.5\height}{\includegraphics[width=0.15\linewidth]{arxiv_images/SD1_5_COCO_Images/SD1_5_COCO_15_CACFGpp_A_clock_sitting_in_a_yard_behind_a_fence.jpg}} & \raisebox{-0.5\height}{\includegraphics[width=0.15\linewidth]{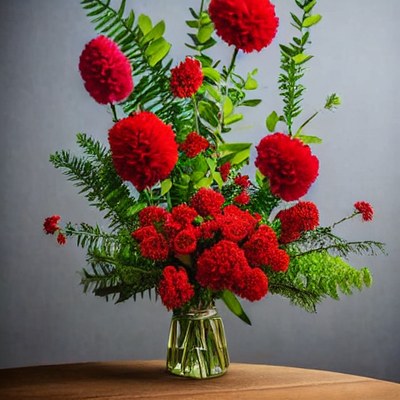}} & \raisebox{-0.5\height}{\includegraphics[width=0.15\linewidth]{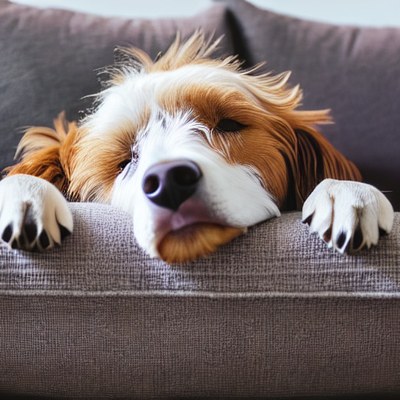}} & \raisebox{-0.5\height}{\includegraphics[width=0.15\linewidth]{arxiv_images/SD1_5_COCO_Images/SD1_5_COCO_15_CACFGpp_A_picture_of_a_painting_and_a_couch.jpg}} & \raisebox{-0.5\height}{\includegraphics[width=0.15\linewidth]{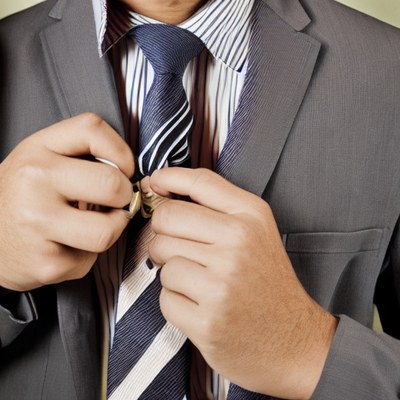}}
        \end{tabular}
        }
        \caption{CFG vs CACFG (ours)}
    \end{subfigure}
    \vspace{1.5em}
    
    \begin{subfigure}{\linewidth}
        \centering
        \resizebox{\linewidth}{!}{%
        \begin{tabular}{@{}c@{\hspace{2pt}}c@{}c@{}c@{}c@{}c@{}c@{}}
            \raisebox{-0.5\height}{\rotatebox{90}{\small CFG++}} & \raisebox{-0.5\height}{\includegraphics[width=0.15\linewidth]{arxiv_images/SD1_5_COCO_Images/SD1_5_COCO_15_CFGpp_A_black_cat_wearing_a_little_party_hat.jpg}} & \raisebox{-0.5\height}{\includegraphics[width=0.15\linewidth]{arxiv_images/SD1_5_COCO_Images/SD1_5_COCO_15_CFGpp_A_clock_sitting_in_a_yard_behind_a_fence.jpg}} & \raisebox{-0.5\height}{\includegraphics[width=0.15\linewidth]{arxiv_images/SD1_5_COCO_Images/SD1_5_COCO_15_CFGpp_A_decorative_arrangement_with_red_flowers_in_a_vase.jpg}} & \raisebox{-0.5\height}{\includegraphics[width=0.15\linewidth]{arxiv_images/SD1_5_COCO_Images/SD1_5_COCO_15_CFGpp_A_furry_dog_with_its_head_on_a_pillow_on_the_couch.jpg}} & \raisebox{-0.5\height}{\includegraphics[width=0.15\linewidth]{arxiv_images/SD1_5_COCO_Images/SD1_5_COCO_15_CFGpp_A_picture_of_a_painting_and_a_couch.jpg}} & \raisebox{-0.5\height}{\includegraphics[width=0.15\linewidth]{arxiv_images/SD1_5_COCO_Images/SD1_5_COCO_15_CFGpp_a_man_attemping_to_tie_his_tie.jpg}} \\
            \raisebox{-0.5\height}{\rotatebox{90}{\small CACFG++}} & \raisebox{-0.5\height}{\includegraphics[width=0.15\linewidth]{arxiv_images/SD1_5_COCO_Images/SD1_5_COCO_15_CACFGpp_A_black_cat_wearing_a_little_party_hat.jpg}} & \raisebox{-0.5\height}{\includegraphics[width=0.15\linewidth]{arxiv_images/SD1_5_COCO_Images/SD1_5_COCO_15_CACFGpp_A_clock_sitting_in_a_yard_behind_a_fence.jpg}} & \raisebox{-0.5\height}{\includegraphics[width=0.15\linewidth]{arxiv_images/SD1_5_COCO_Images/SD1_5_COCO_15_CACFGpp_A_decorative_arrangement_with_red_flowers_in_a_vase.jpg}} & \raisebox{-0.5\height}{\includegraphics[width=0.15\linewidth]{arxiv_images/SD1_5_COCO_Images/SD1_5_COCO_15_CACFGpp_A_furry_dog_with_its_head_on_a_pillow_on_the_couch.jpg}} & \raisebox{-0.5\height}{\includegraphics[width=0.15\linewidth]{arxiv_images/SD1_5_COCO_Images/SD1_5_COCO_15_CACFGpp_A_picture_of_a_painting_and_a_couch.jpg}} & \raisebox{-0.5\height}{\includegraphics[width=0.15\linewidth]{arxiv_images/SD1_5_COCO_Images/SD1_5_COCO_15_CACFGpp_a_man_attemping_to_tie_his_tie.jpg}}
        \end{tabular}
        }
        \caption{CFG++ vs CACFG++ (ours)}
    \end{subfigure}
    \vspace{1.5em}

    \begin{subfigure}{\linewidth}
        \centering
        \resizebox{\linewidth}{!}{%
        \begin{tabular}{@{}c@{\hspace{2pt}}c@{}c@{}c@{}c@{}c@{}c@{}}
            \raisebox{-0.5\height}{\rotatebox{90}{\small IntCFG}} & \raisebox{-0.5\height}{\includegraphics[width=0.15\linewidth]{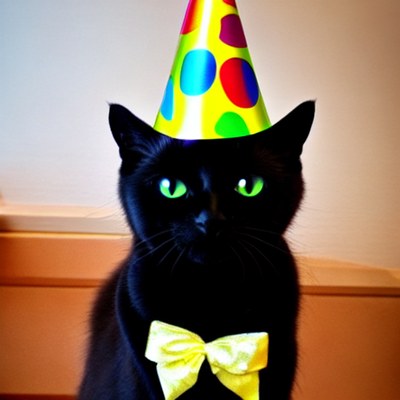}} & \raisebox{-0.5\height}{\includegraphics[width=0.15\linewidth]{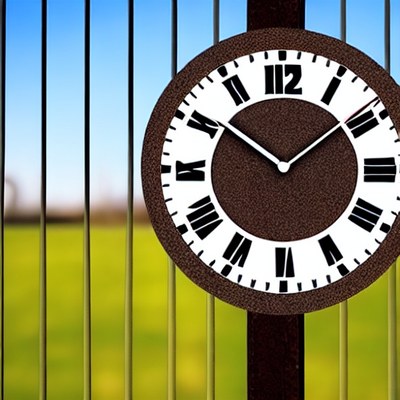}} & \raisebox{-0.5\height}{\includegraphics[width=0.15\linewidth]{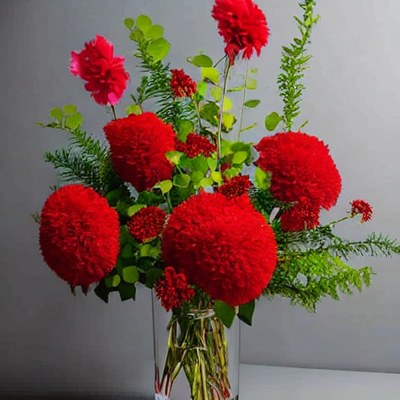}} & \raisebox{-0.5\height}{\includegraphics[width=0.15\linewidth]{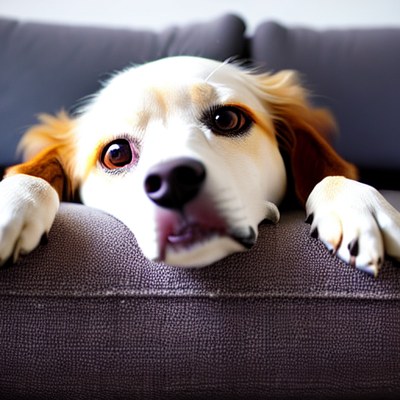}} & \raisebox{-0.5\height}{\includegraphics[width=0.15\linewidth]{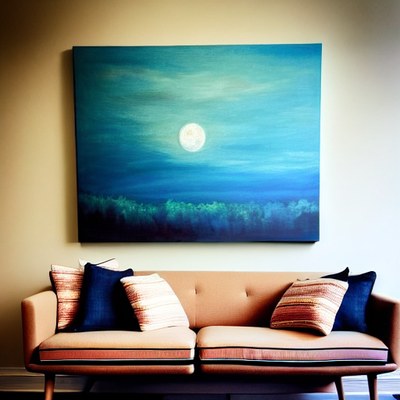}} & \raisebox{-0.5\height}{\includegraphics[width=0.15\linewidth]{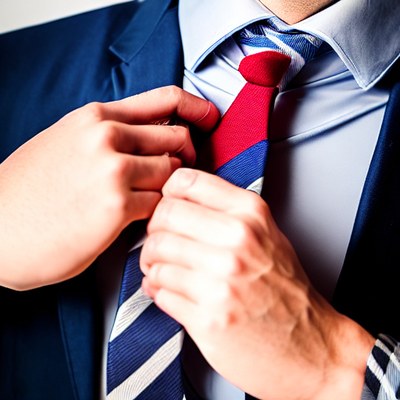}} \\
            \raisebox{-0.5\height}{\rotatebox{90}{\small IntCACFG}} & \raisebox{-0.5\height}{\includegraphics[width=0.15\linewidth]{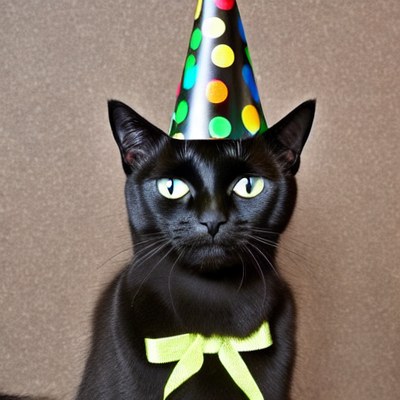}} & \raisebox{-0.5\height}{\includegraphics[width=0.15\linewidth]{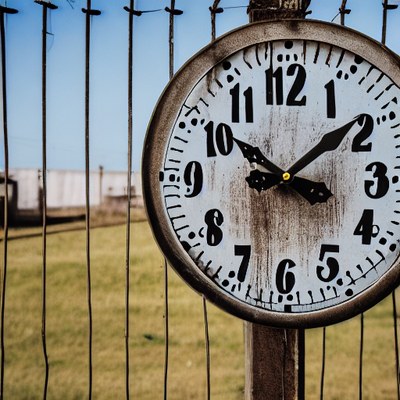}} & \raisebox{-0.5\height}{\includegraphics[width=0.15\linewidth]{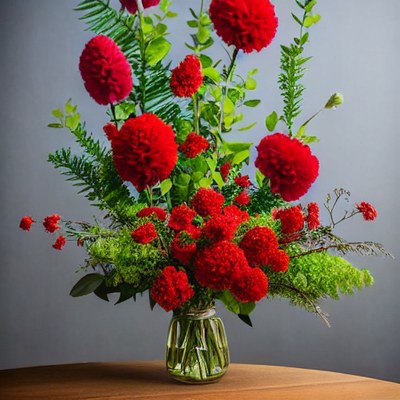}} & \raisebox{-0.5\height}{\includegraphics[width=0.15\linewidth]{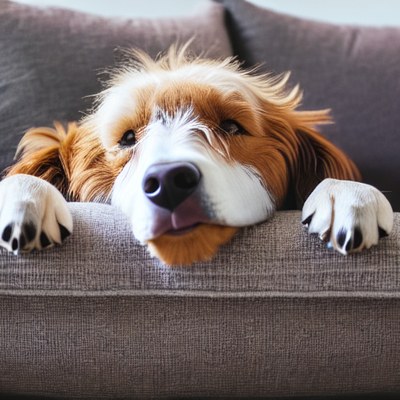}} & \raisebox{-0.5\height}{\includegraphics[width=0.15\linewidth]{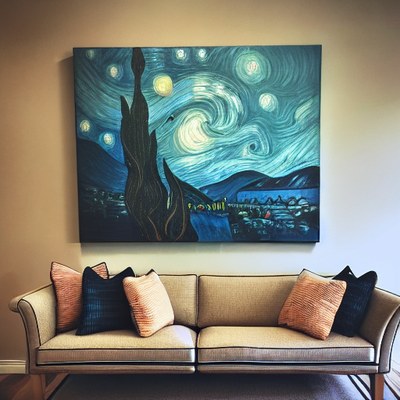}} & \raisebox{-0.5\height}{\includegraphics[width=0.15\linewidth]{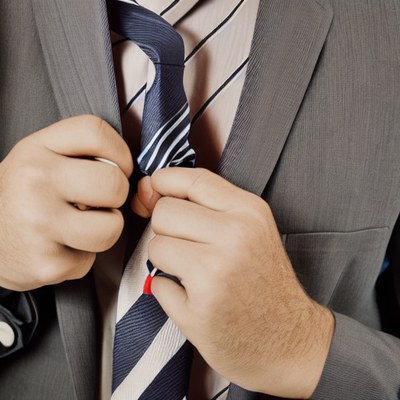}}
        \end{tabular}
        }
        \caption{IntCFG vs IntCACFG (ours)}
    \end{subfigure}
    \caption{CFG \citep{ClassifierFreeGuidance_Ho2022}, CFG++ \citep{CFG++_Chung2025}, and IntCFG \citep{IntCFG_Kynkaanniemi2024} visualised against their curvature-aware variants (ours): CACFG, CACFG++, IntCACFG. Samples are generated using SD1.5 \citep{StableDiffusion_Rombach2022} targeting the COCO \citep{COCO_Lin2014} dataset with a guidance strength of $\omega=15$.}
    \label{fig:SD1.5 COCO Guidance Grid}
\end{figure}

\begin{figure}[h]
    \centering
    \resizebox{0.7\linewidth}{!}{%
    \begin{tabular}{@{}c@{\hspace{2pt}}c@{}c@{}c@{}c@{}}
        & \makebox[0.15\linewidth][c]{\small $\omega = 5$} & \makebox[0.15\linewidth][c]{\small $\omega = 7.5$} & \makebox[0.15\linewidth][c]{\small $\omega = 10$} & \makebox[0.15\linewidth][c]{\small $\omega = 15$} \\
        \noalign{\vskip 4pt}
        \raisebox{-0.5\height}{\rotatebox{90}{\small CFG}} & \raisebox{-0.5\height}{\includegraphics[width=0.15\linewidth]{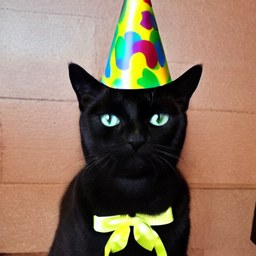}} & \raisebox{-0.5\height}{\includegraphics[width=0.15\linewidth]{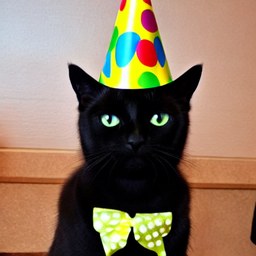}} & \raisebox{-0.5\height}{\includegraphics[width=0.15\linewidth]{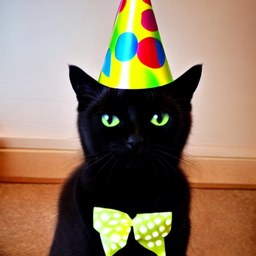}} & \raisebox{-0.5\height}{\includegraphics[width=0.15\linewidth]{arxiv_images/SD1_5_COCO_Images/SD1_5_COCO_15_CFGpp_A_black_cat_wearing_a_little_party_hat.jpg}} \\
        \raisebox{-0.5\height}{\rotatebox{90}{\small CACFG}} & \raisebox{-0.5\height}{\includegraphics[width=0.15\linewidth]{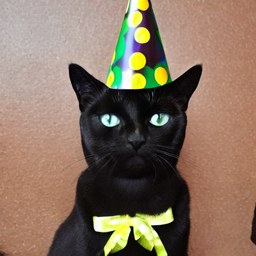}} & \raisebox{-0.5\height}{\includegraphics[width=0.15\linewidth]{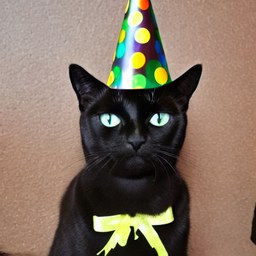}} & \raisebox{-0.5\height}{\includegraphics[width=0.15\linewidth]{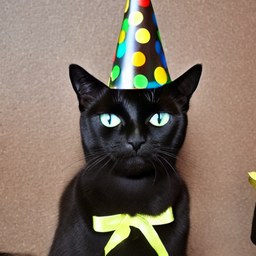}} & \raisebox{-0.5\height}{\includegraphics[width=0.15\linewidth]{arxiv_images/SD1_5_COCO_Images/SD1_5_COCO_15_CACFGpp_A_black_cat_wearing_a_little_party_hat.jpg}} \\
        \noalign{\vskip 4pt}
        \raisebox{-0.5\height}{\rotatebox{90}{\small CFG}} & \raisebox{-0.5\height}{\includegraphics[width=0.15\linewidth]{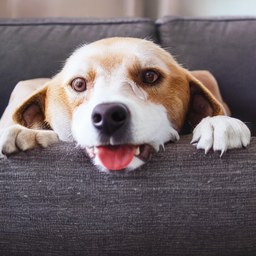}} & \raisebox{-0.5\height}{\includegraphics[width=0.15\linewidth]{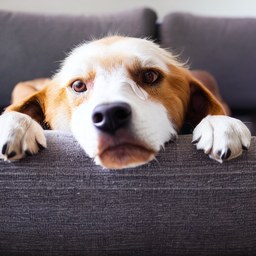}} & \raisebox{-0.5\height}{\includegraphics[width=0.15\linewidth]{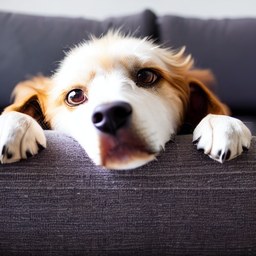}} & \raisebox{-0.5\height}{\includegraphics[width=0.15\linewidth]{arxiv_images/SD1_5_COCO_Images/SD1_5_COCO_15_CFGpp_A_furry_dog_with_its_head_on_a_pillow_on_the_couch.jpg}} \\
        \raisebox{-0.5\height}{\rotatebox{90}{\small CACFG}} & \raisebox{-0.5\height}{\includegraphics[width=0.15\linewidth]{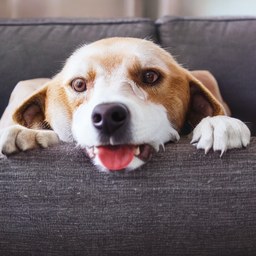}} & \raisebox{-0.5\height}{\includegraphics[width=0.15\linewidth]{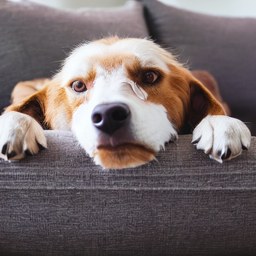}} & \raisebox{-0.5\height}{\includegraphics[width=0.15\linewidth]{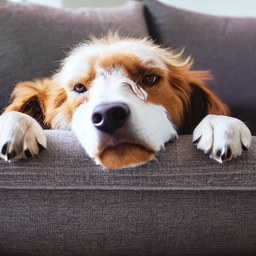}} & \raisebox{-0.5\height}{\includegraphics[width=0.15\linewidth]{arxiv_images/SD1_5_COCO_Images/SD1_5_COCO_15_CACFGpp_A_furry_dog_with_its_head_on_a_pillow_on_the_couch.jpg}} \\
        \noalign{\vskip 4pt}
        \raisebox{-0.5\height}{\rotatebox{90}{\small CFG}} & \raisebox{-0.5\height}{\includegraphics[width=0.15\linewidth]{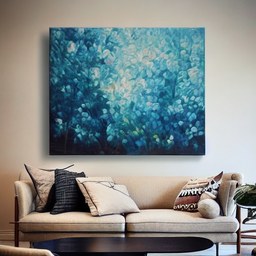}} & \raisebox{-0.5\height}{\includegraphics[width=0.15\linewidth]{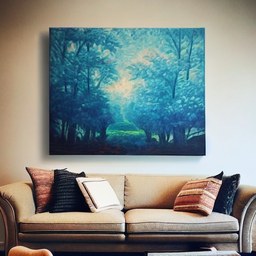}} & \raisebox{-0.5\height}{\includegraphics[width=0.15\linewidth]{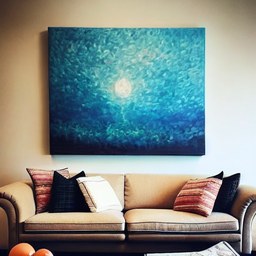}} & \raisebox{-0.5\height}{\includegraphics[width=0.15\linewidth]{arxiv_images/SD1_5_COCO_Images/SD1_5_COCO_15_CFGpp_A_picture_of_a_painting_and_a_couch.jpg}} \\
        \raisebox{-0.5\height}{\rotatebox{90}{\small CACFG}} & \raisebox{-0.5\height}{\includegraphics[width=0.15\linewidth]{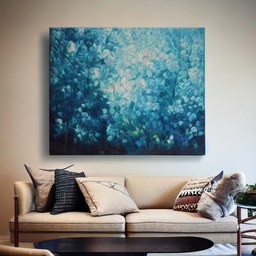}} & \raisebox{-0.5\height}{\includegraphics[width=0.15\linewidth]{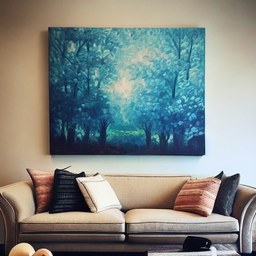}} & \raisebox{-0.5\height}{\includegraphics[width=0.15\linewidth]{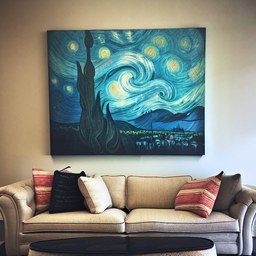}} & \raisebox{-0.5\height}{\includegraphics[width=0.15\linewidth]{arxiv_images/SD1_5_COCO_Images/SD1_5_COCO_15_CACFGpp_A_picture_of_a_painting_and_a_couch.jpg}} \\
        \noalign{\vskip 4pt}
        \raisebox{-0.5\height}{\rotatebox{90}{\small CFG}} & \raisebox{-0.5\height}{\includegraphics[width=0.15\linewidth]{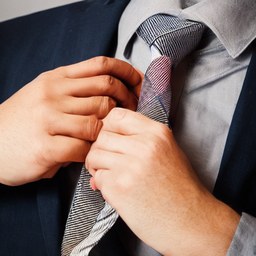}} & \raisebox{-0.5\height}{\includegraphics[width=0.15\linewidth]{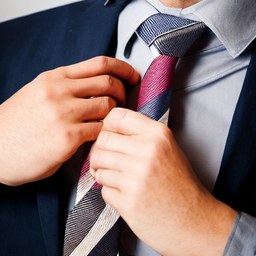}} & \raisebox{-0.5\height}{\includegraphics[width=0.15\linewidth]{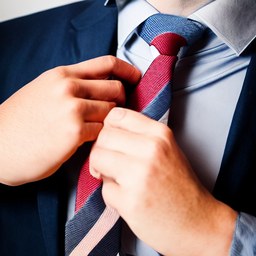}} & \raisebox{-0.5\height}{\includegraphics[width=0.15\linewidth]{arxiv_images/SD1_5_COCO_Images/SD1_5_COCO_15_CFGpp_a_man_attemping_to_tie_his_tie.jpg}} \\
        \raisebox{-0.5\height}{\rotatebox{90}{\small CACFG}} & \raisebox{-0.5\height}{\includegraphics[width=0.15\linewidth]{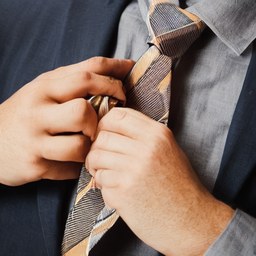}} & \raisebox{-0.5\height}{\includegraphics[width=0.15\linewidth]{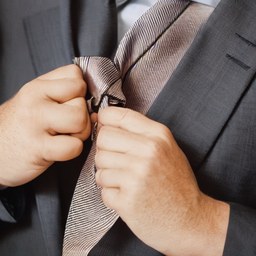}} & \raisebox{-0.5\height}{\includegraphics[width=0.15\linewidth]{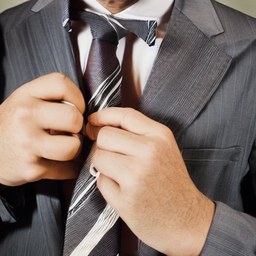}} & \raisebox{-0.5\height}{\includegraphics[width=0.15\linewidth]{arxiv_images/SD1_5_COCO_Images/SD1_5_COCO_15_CACFGpp_a_man_attemping_to_tie_his_tie.jpg}}
    \end{tabular}
    }
    \caption{CFG \citep{ClassifierFreeGuidance_Ho2022} and CACFG (ours) samples generated using SD1.5 \citep{StableDiffusion_Rombach2022} targeting the COCO \citep{COCO_Lin2014} dataset. Guidance scale varies across $\omega = 5, 7.5, 10, 15$.}
    \label{fig:SD1.5 COCO Guidance Strength}
\end{figure}

\clearpage
\subsubsection{SDXL}
\begin{figure}[h]
    \centering
    \begin{subfigure}{\linewidth}
        \centering
        \resizebox{\linewidth}{!}{%
        \begin{tabular}{@{}c@{\hspace{2pt}}c@{}c@{}c@{}c@{}c@{}c@{}}
            \raisebox{-0.5\height}{\rotatebox{90}{\small CFG}} & \raisebox{-0.5\height}{\includegraphics[width=0.15\linewidth]{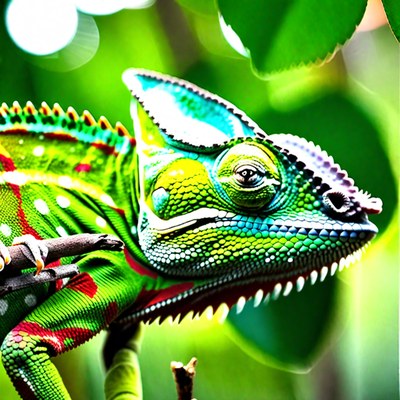}} & \raisebox{-0.5\height}{\includegraphics[width=0.15\linewidth]{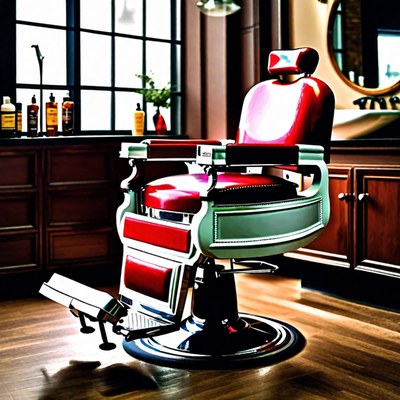}} & \raisebox{-0.5\height}{\includegraphics[width=0.15\linewidth]{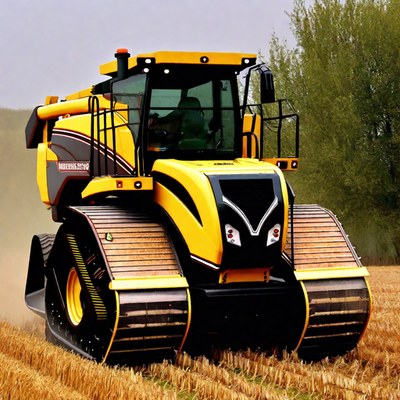}} & \raisebox{-0.5\height}{\includegraphics[width=0.15\linewidth]{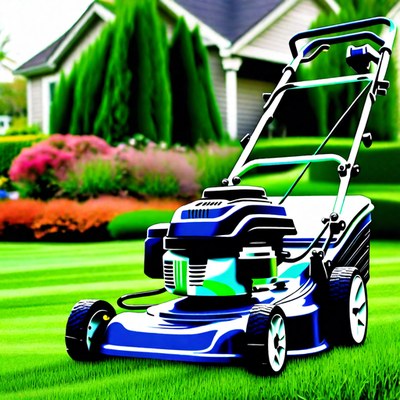}} & \raisebox{-0.5\height}{\includegraphics[width=0.15\linewidth]{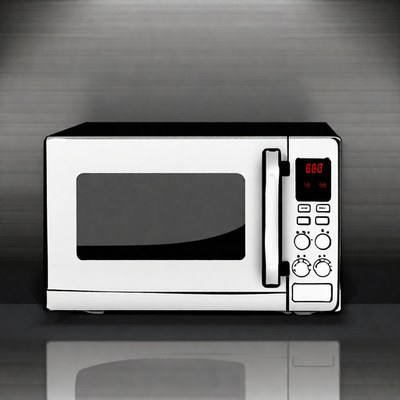}} & \raisebox{-0.5\height}{\includegraphics[width=0.15\linewidth]{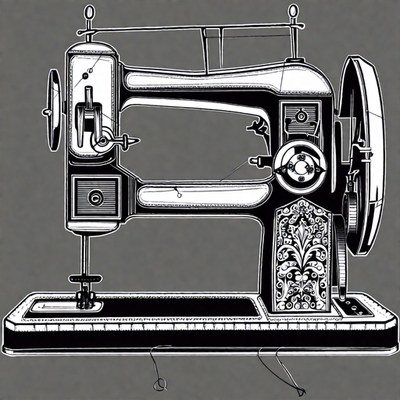}} \\
            \raisebox{-0.5\height}{\rotatebox{90}{\small CACFG}} & \raisebox{-0.5\height}{\includegraphics[width=0.15\linewidth]{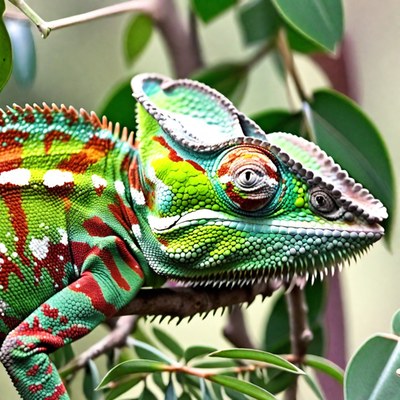}} & \raisebox{-0.5\height}{\includegraphics[width=0.15\linewidth]{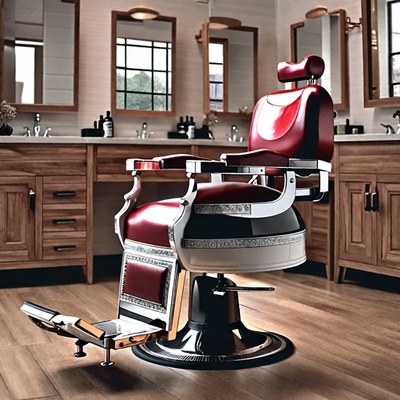}} & \raisebox{-0.5\height}{\includegraphics[width=0.15\linewidth]{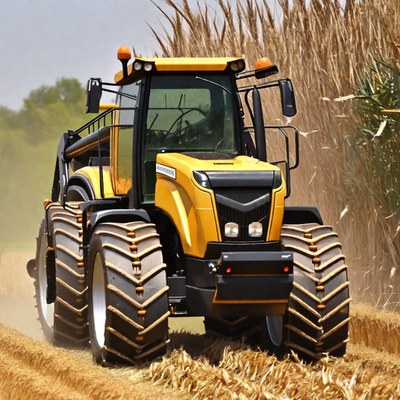}} & \raisebox{-0.5\height}{\includegraphics[width=0.15\linewidth]{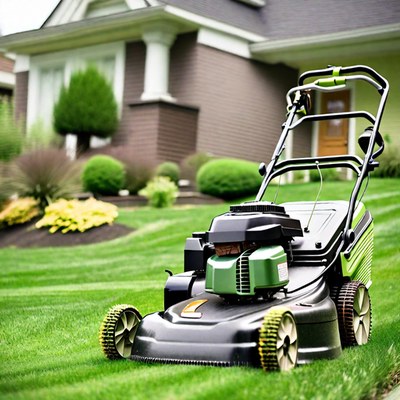}} & \raisebox{-0.5\height}{\includegraphics[width=0.15\linewidth]{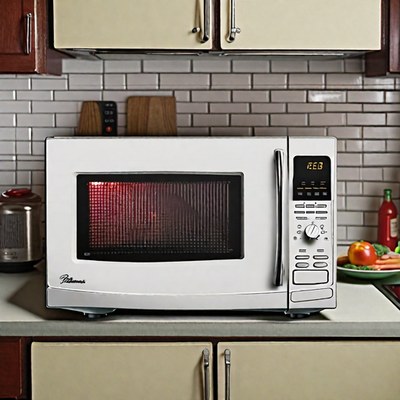}} & \raisebox{-0.5\height}{\includegraphics[width=0.15\linewidth]{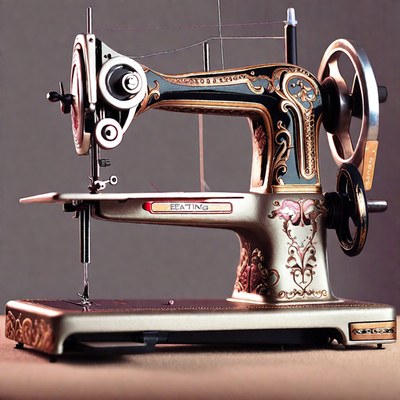}}
        \end{tabular}
        }
        \caption{CFG vs CACFG (ours)}
    \end{subfigure}
    \vspace{1.5em}

    \begin{subfigure}{\linewidth}
    \centering
    \resizebox{\linewidth}{!}{%
    \begin{tabular}{@{}c@{\hspace{2pt}}c@{}c@{}c@{}c@{}c@{}c@{}}
        \raisebox{-0.5\height}{\rotatebox{90}{\small CFG++}} & \raisebox{-0.5\height}{\includegraphics[width=0.15\linewidth]{arxiv_images/SDXL_ImageNet_Images/SDXL_ImNe_25_CFGpp_African_chameleon_Chamaeleo_chamaeleon.jpg}} & \raisebox{-0.5\height}{\includegraphics[width=0.15\linewidth]{arxiv_images/SDXL_ImageNet_Images/SDXL_ImNe_25_CFGpp_barber_chair.jpg}} & \raisebox{-0.5\height}{\includegraphics[width=0.15\linewidth]{arxiv_images/SDXL_ImageNet_Images/SDXL_ImNe_25_CFGpp_harvester_reaper.jpg}} & \raisebox{-0.5\height}{\includegraphics[width=0.15\linewidth]{arxiv_images/SDXL_ImageNet_Images/SDXL_ImNe_25_CFGpp_lawn_mower_mower.jpg}} & \raisebox{-0.5\height}{\includegraphics[width=0.15\linewidth]{arxiv_images/SDXL_ImageNet_Images/SDXL_ImNe_25_CFGpp_microwave_microwave_oven.jpg}} & \raisebox{-0.5\height}{\includegraphics[width=0.15\linewidth]{arxiv_images/SDXL_ImageNet_Images/SDXL_ImNe_25_CFGpp_sewing_machine.jpg}} \\
        \raisebox{-0.5\height}{\rotatebox{90}{\small CACFG++}} & \raisebox{-0.5\height}{\includegraphics[width=0.15\linewidth]{arxiv_images/SDXL_ImageNet_Images/SDXL_ImNe_25_CACFGpp_African_chameleon_Chamaeleo_chamaeleon.jpg}} & \raisebox{-0.5\height}{\includegraphics[width=0.15\linewidth]{arxiv_images/SDXL_ImageNet_Images/SDXL_ImNe_25_CACFGpp_barber_chair.jpg}} & \raisebox{-0.5\height}{\includegraphics[width=0.15\linewidth]{arxiv_images/SDXL_ImageNet_Images/SDXL_ImNe_25_CACFGpp_harvester_reaper.jpg}} & \raisebox{-0.5\height}{\includegraphics[width=0.15\linewidth]{arxiv_images/SDXL_ImageNet_Images/SDXL_ImNe_25_CACFGpp_lawn_mower_mower.jpg}} & \raisebox{-0.5\height}{\includegraphics[width=0.15\linewidth]{arxiv_images/SDXL_ImageNet_Images/SDXL_ImNe_25_CACFGpp_microwave_microwave_oven.jpg}} & \raisebox{-0.5\height}{\includegraphics[width=0.15\linewidth]{arxiv_images/SDXL_ImageNet_Images/SDXL_ImNe_25_CACFGpp_sewing_machine.jpg}}
    \end{tabular}
    }
    \caption{CFG++ vs CACFG++ (ours)}
    \end{subfigure}
    \vspace{1.5em}

    \begin{subfigure}{\linewidth}
    \centering
    \resizebox{\linewidth}{!}{%
    \begin{tabular}{@{}c@{\hspace{2pt}}c@{}c@{}c@{}c@{}c@{}c@{}}
        \raisebox{-0.5\height}{\rotatebox{90}{\small IntCFG}} & \raisebox{-0.5\height}{\includegraphics[width=0.15\linewidth]{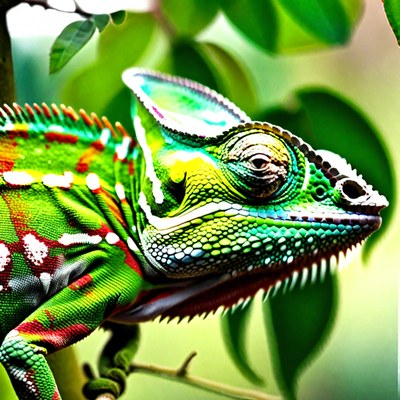}} & \raisebox{-0.5\height}{\includegraphics[width=0.15\linewidth]{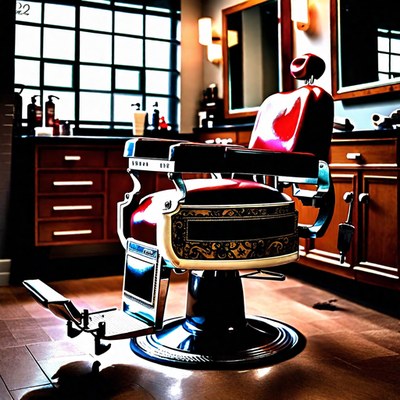}} & \raisebox{-0.5\height}{\includegraphics[width=0.15\linewidth]{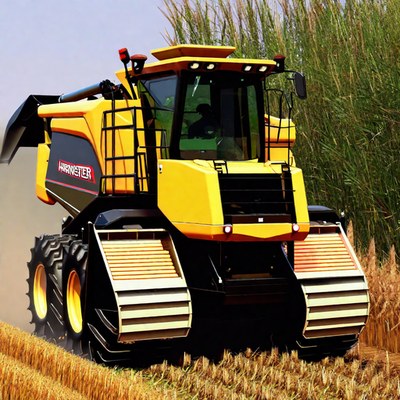}} & \raisebox{-0.5\height}{\includegraphics[width=0.15\linewidth]{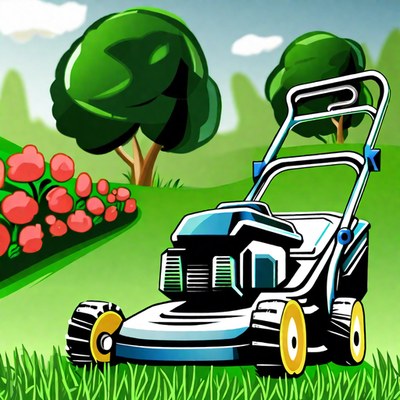}} & \raisebox{-0.5\height}{\includegraphics[width=0.15\linewidth]{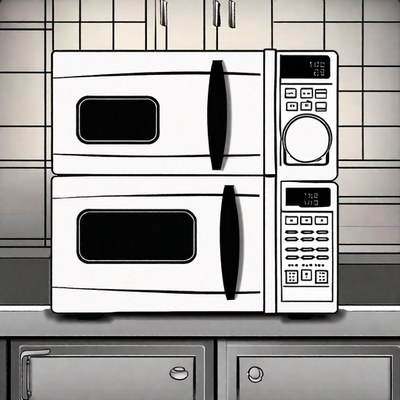}} & \raisebox{-0.5\height}{\includegraphics[width=0.15\linewidth]{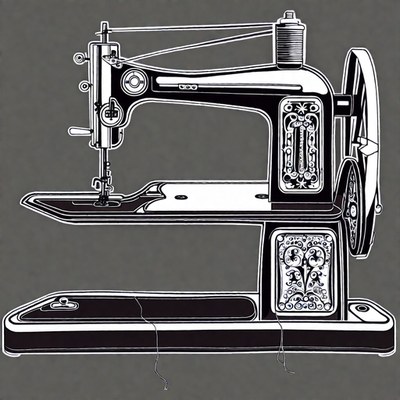}} \\
        \raisebox{-0.5\height}{\rotatebox{90}{\small IntCACFG}} & \raisebox{-0.5\height}{\includegraphics[width=0.15\linewidth]{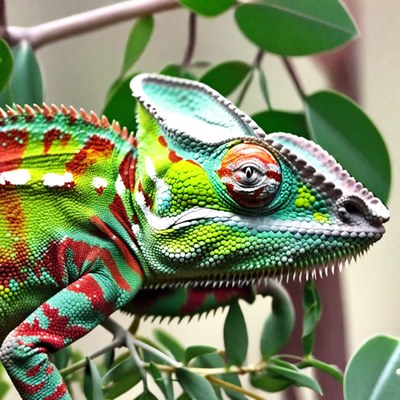}} & \raisebox{-0.5\height}{\includegraphics[width=0.15\linewidth]{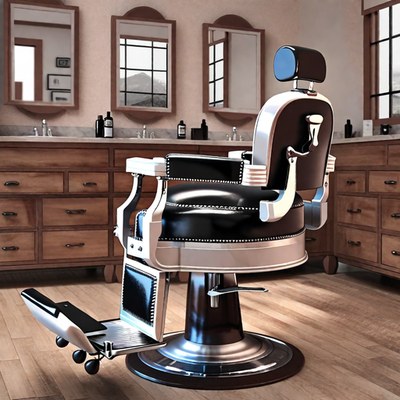}} & \raisebox{-0.5\height}{\includegraphics[width=0.15\linewidth]{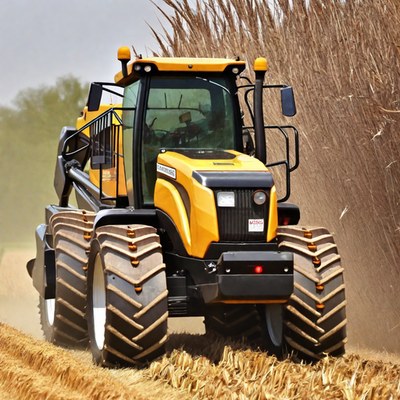}} & \raisebox{-0.5\height}{\includegraphics[width=0.15\linewidth]{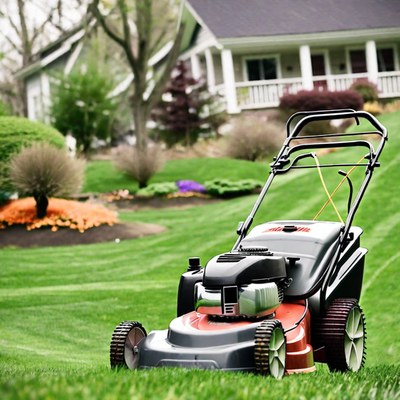}} & \raisebox{-0.5\height}{\includegraphics[width=0.15\linewidth]{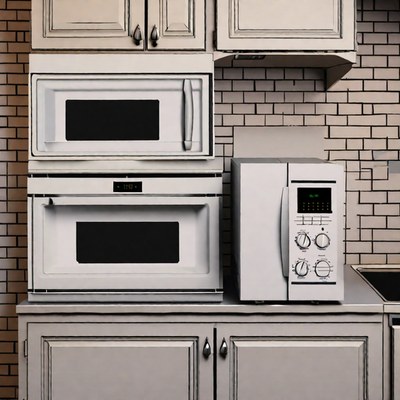}} & \raisebox{-0.5\height}{\includegraphics[width=0.15\linewidth]{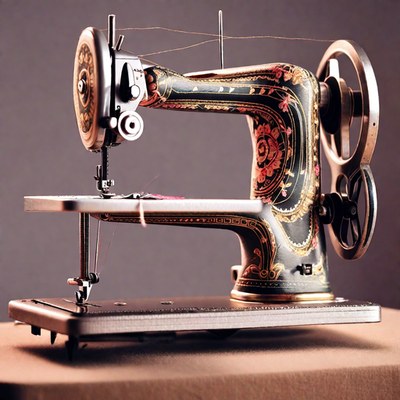}}
    \end{tabular}
    }
    \caption{IntCFG vs IntCACFG (ours)}
    \end{subfigure}
    \caption{CFG \citep{ClassifierFreeGuidance_Ho2022}, CFG++ \citep{CFG++_Chung2025}, and IntCFG \citep{IntCFG_Kynkaanniemi2024} visualised against their curvature-aware variants (ours): CACFG, CACFG++, IntCACFG. Samples are generated using SDXL \citep{SDXL_Podell2024} targeting the ImageNet \citep{ImageNetBenchmark_Russakovsky2015} dataset with a guidance strength of $\omega=25$.}
    \label{fig:SDXL ImageNet Guidance Grid}
\end{figure}

\begin{figure}[h]
    \centering
    \resizebox{0.7\linewidth}{!}{%
    \begin{tabular}{@{}c@{\hspace{2pt}}c@{}c@{}c@{}c@{}}
        & \makebox[0.15\linewidth][c]{\small $\omega = 10$} & \makebox[0.15\linewidth][c]{\small $\omega = 15$} & \makebox[0.15\linewidth][c]{\small $\omega = 20$} & \makebox[0.15\linewidth][c]{\small $\omega = 25$} \\
        \noalign{\vskip 4pt}
        \raisebox{-0.5\height}{\rotatebox{90}{\small CFG}} & \raisebox{-0.5\height}{\includegraphics[width=0.15\linewidth]{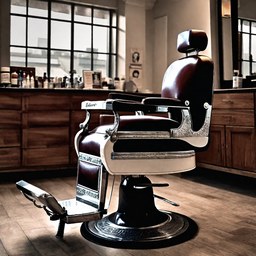}} & \raisebox{-0.5\height}{\includegraphics[width=0.15\linewidth]{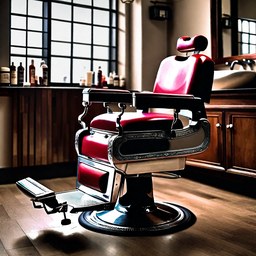}} & \raisebox{-0.5\height}{\includegraphics[width=0.15\linewidth]{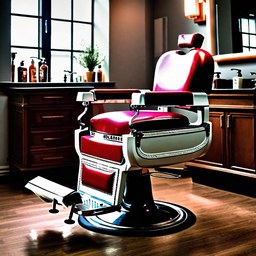}} & \raisebox{-0.5\height}{\includegraphics[width=0.15\linewidth]{arxiv_images/SDXL_ImageNet_Images/SDXL_ImNe_25_CFGpp_barber_chair.jpg}} \\
        \raisebox{-0.5\height}{\rotatebox{90}{\small CACFG}} & \raisebox{-0.5\height}{\includegraphics[width=0.15\linewidth]{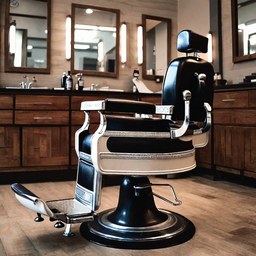}} & \raisebox{-0.5\height}{\includegraphics[width=0.15\linewidth]{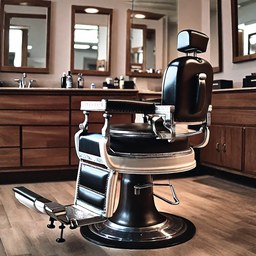}} & \raisebox{-0.5\height}{\includegraphics[width=0.15\linewidth]{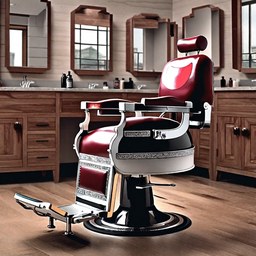}} & \raisebox{-0.5\height}{\includegraphics[width=0.15\linewidth]{arxiv_images/SDXL_ImageNet_Images/SDXL_ImNe_25_CACFGpp_barber_chair.jpg}} \\
        \noalign{\vskip 4pt}
        \raisebox{-0.5\height}{\rotatebox{90}{\small CFG}} & \raisebox{-0.5\height}{\includegraphics[width=0.15\linewidth]{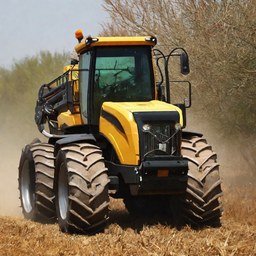}} & \raisebox{-0.5\height}{\includegraphics[width=0.15\linewidth]{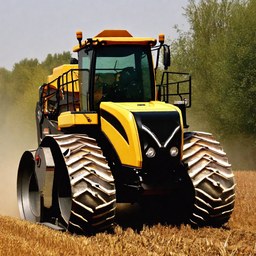}} & \raisebox{-0.5\height}{\includegraphics[width=0.15\linewidth]{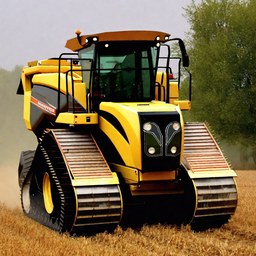}} & \raisebox{-0.5\height}{\includegraphics[width=0.15\linewidth]{arxiv_images/SDXL_ImageNet_Images/SDXL_ImNe_25_CFGpp_harvester_reaper.jpg}} \\
        \raisebox{-0.5\height}{\rotatebox{90}{\small CACFG}} & \raisebox{-0.5\height}{\includegraphics[width=0.15\linewidth]{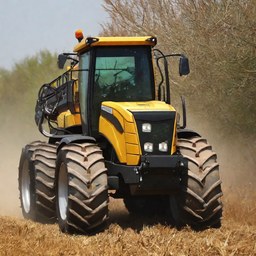}} & \raisebox{-0.5\height}{\includegraphics[width=0.15\linewidth]{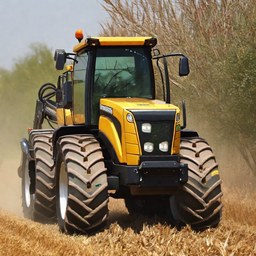}} & \raisebox{-0.5\height}{\includegraphics[width=0.15\linewidth]{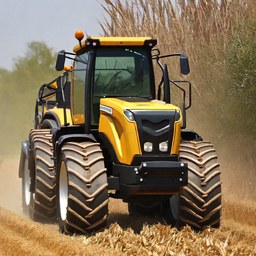}} & \raisebox{-0.5\height}{\includegraphics[width=0.15\linewidth]{arxiv_images/SDXL_ImageNet_Images/SDXL_ImNe_25_CACFGpp_harvester_reaper.jpg}} \\
        \noalign{\vskip 4pt}
        \raisebox{-0.5\height}{\rotatebox{90}{\small CFG}} & \raisebox{-0.5\height}{\includegraphics[width=0.15\linewidth]{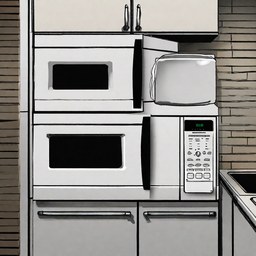}} & \raisebox{-0.5\height}{\includegraphics[width=0.15\linewidth]{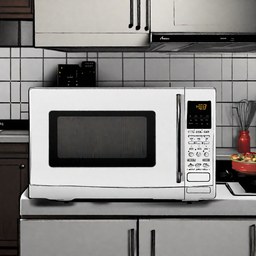}} & \raisebox{-0.5\height}{\includegraphics[width=0.15\linewidth]{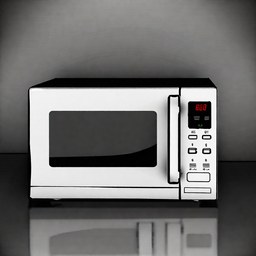}} & \raisebox{-0.5\height}{\includegraphics[width=0.15\linewidth]{arxiv_images/SDXL_ImageNet_Images/SDXL_ImNe_25_CFGpp_microwave_microwave_oven.jpg}} \\
        \raisebox{-0.5\height}{\rotatebox{90}{\small CACFG}} & \raisebox{-0.5\height}{\includegraphics[width=0.15\linewidth]{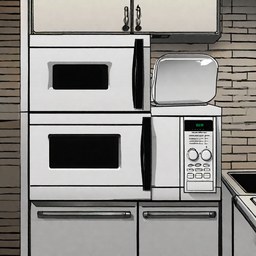}} & \raisebox{-0.5\height}{\includegraphics[width=0.15\linewidth]{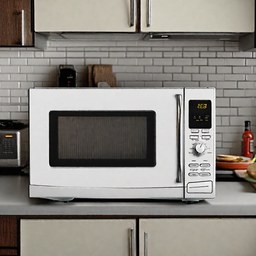}} & \raisebox{-0.5\height}{\includegraphics[width=0.15\linewidth]{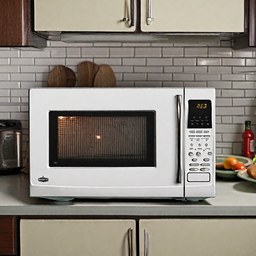}} & \raisebox{-0.5\height}{\includegraphics[width=0.15\linewidth]{arxiv_images/SDXL_ImageNet_Images/SDXL_ImNe_25_CACFGpp_microwave_microwave_oven.jpg}} \\
        \noalign{\vskip 4pt}
        \raisebox{-0.5\height}{\rotatebox{90}{\small CFG}} & \raisebox{-0.5\height}{\includegraphics[width=0.15\linewidth]{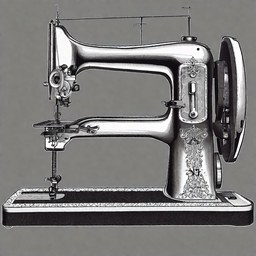}} & \raisebox{-0.5\height}{\includegraphics[width=0.15\linewidth]{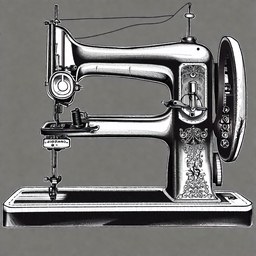}} & \raisebox{-0.5\height}{\includegraphics[width=0.15\linewidth]{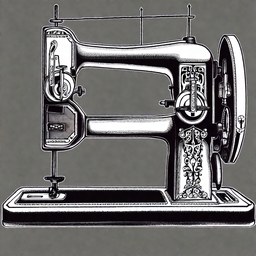}} & \raisebox{-0.5\height}{\includegraphics[width=0.15\linewidth]{arxiv_images/SDXL_ImageNet_Images/SDXL_ImNe_25_CFGpp_sewing_machine.jpg}} \\
        \raisebox{-0.5\height}{\rotatebox{90}{\small CACFG}} & \raisebox{-0.5\height}{\includegraphics[width=0.15\linewidth]{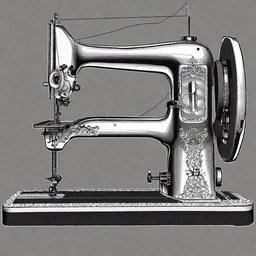}} & \raisebox{-0.5\height}{\includegraphics[width=0.15\linewidth]{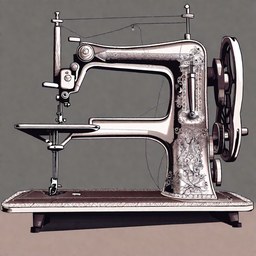}} & \raisebox{-0.5\height}{\includegraphics[width=0.15\linewidth]{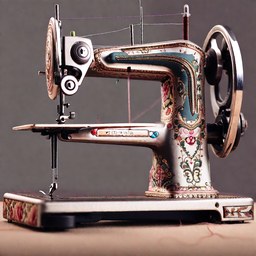}} & \raisebox{-0.5\height}{\includegraphics[width=0.15\linewidth]{arxiv_images/SDXL_ImageNet_Images/SDXL_ImNe_25_CACFGpp_sewing_machine.jpg}}
    \end{tabular}
    }
    \caption{CFG \citep{ClassifierFreeGuidance_Ho2022} and CACFG (ours) samples generated using SDXL \citep{SDXL_Podell2024} targeting the ImageNet \citep{ImageNetBenchmark_Russakovsky2015} dataset. Guidance scale varies across $\omega = 10, 15, 20, 25$.}
    \label{fig:SDXL ImageNet Guidance Strength}
\end{figure}

\begin{figure}[h]
    \centering
    \begin{subfigure}{\linewidth}
        \centering
        \resizebox{\linewidth}{!}{%
        \begin{tabular}{@{}c@{\hspace{2pt}}c@{}c@{}c@{}c@{}c@{}c@{}}
            \raisebox{-0.5\height}{\rotatebox{90}{\small CFG}} & \raisebox{-0.5\height}{\includegraphics[width=0.15\linewidth]{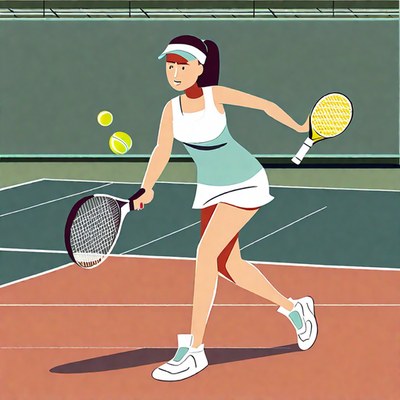}} & \raisebox{-0.5\height}{\includegraphics[width=0.15\linewidth]{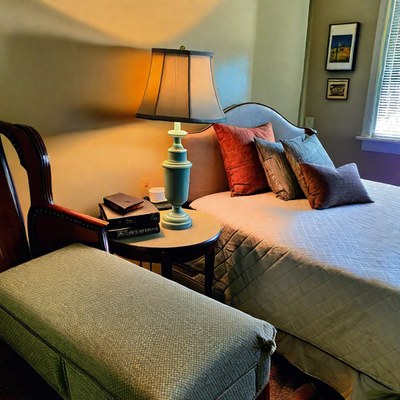}} & \raisebox{-0.5\height}{\includegraphics[width=0.15\linewidth]{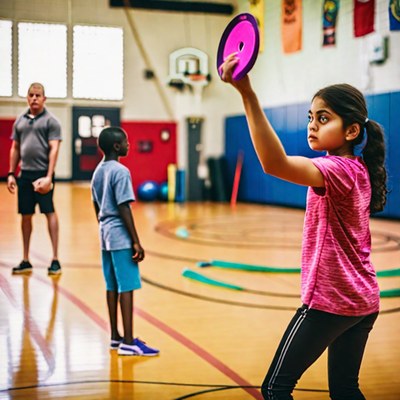}} & \raisebox{-0.5\height}{\includegraphics[width=0.15\linewidth]{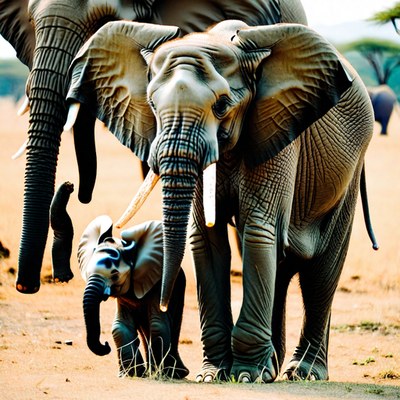}} & \raisebox{-0.5\height}{\includegraphics[width=0.15\linewidth]{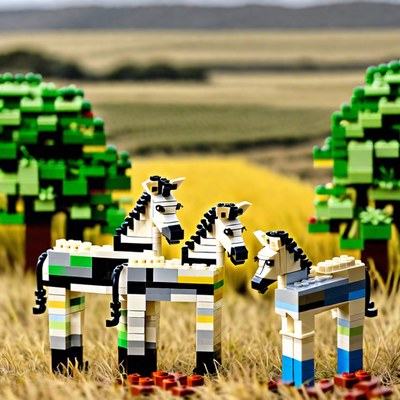}} & \raisebox{-0.5\height}{\includegraphics[width=0.15\linewidth]{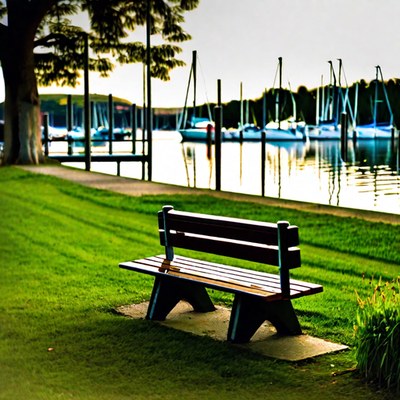}} \\
            \raisebox{-0.5\height}{\rotatebox{90}{\small CACFG}} & \raisebox{-0.5\height}{\includegraphics[width=0.15\linewidth]{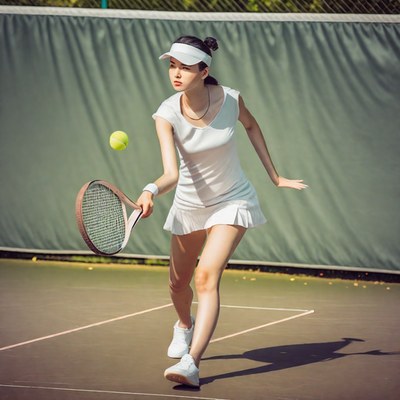}} & \raisebox{-0.5\height}{\includegraphics[width=0.15\linewidth]{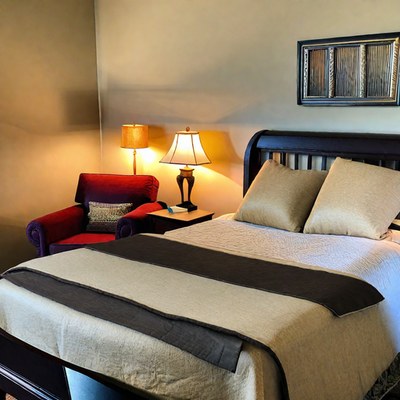}} & \raisebox{-0.5\height}{\includegraphics[width=0.15\linewidth]{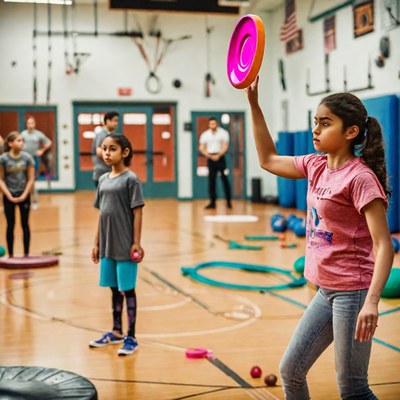}} & \raisebox{-0.5\height}{\includegraphics[width=0.15\linewidth]{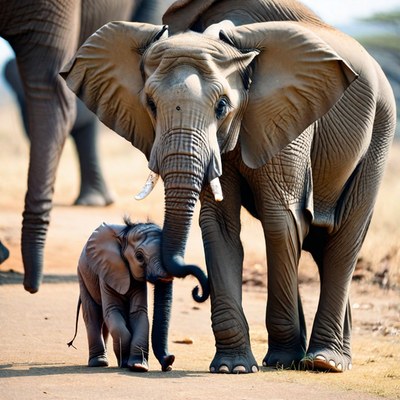}} & \raisebox{-0.5\height}{\includegraphics[width=0.15\linewidth]{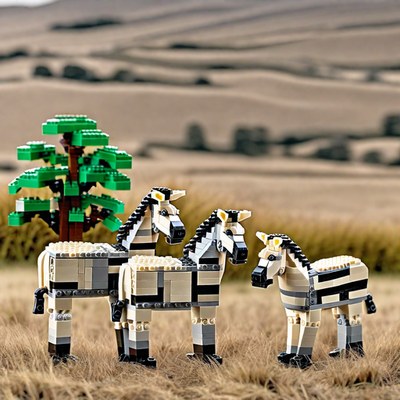}} & \raisebox{-0.5\height}{\includegraphics[width=0.15\linewidth]{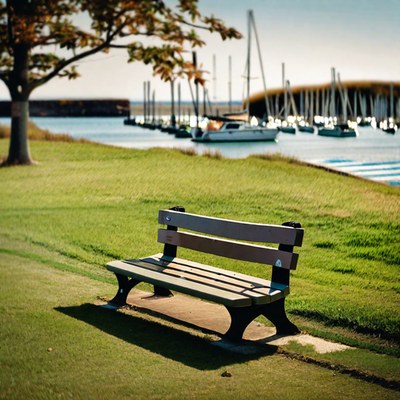}}
        \end{tabular}
        }
        \caption{CFG vs CACFG (ours)}
    \end{subfigure}

    \vspace{1.5em}
    \begin{subfigure}{\linewidth}
        \centering
        \resizebox{\linewidth}{!}{%
        \begin{tabular}{@{}c@{\hspace{2pt}}c@{}c@{}c@{}c@{}c@{}c@{}}
            \raisebox{-0.5\height}{\rotatebox{90}{\small CFG++}} & \raisebox{-0.5\height}{\includegraphics[width=0.15\linewidth]{arxiv_images/SDXL_COCO_Images/SDXL_COCO_25_CFGpp_A_YOUNG_LADY_IS_ON_THE_COURT_PLAYING_TENNIS.jpg}} & \raisebox{-0.5\height}{\includegraphics[width=0.15\linewidth]{arxiv_images/SDXL_COCO_Images/SDXL_COCO_25_CFGpp_A_bed_sitting_in_a_room_next_to_a_lamp_and_chair.jpg}} & \raisebox{-0.5\height}{\includegraphics[width=0.15\linewidth]{arxiv_images/SDXL_COCO_Images/SDXL_COCO_25_CFGpp_A_young_girl_prepares_to_throw_a_disc_through_a_hoop_while_other_watch_in_a_gym.jpg}} & \raisebox{-0.5\height}{\includegraphics[width=0.15\linewidth]{arxiv_images/SDXL_COCO_Images/SDXL_COCO_25_CFGpp_An_elephant_and_baby_stand_next_to_each_other.jpg}} & \raisebox{-0.5\height}{\includegraphics[width=0.15\linewidth]{arxiv_images/SDXL_COCO_Images/SDXL_COCO_25_CFGpp_Two_zebras_made_of_Lego_stand_in_a_field.jpg}} & \raisebox{-0.5\height}{\includegraphics[width=0.15\linewidth]{arxiv_images/SDXL_COCO_Images/SDXL_COCO_25_CFGpp_a_bench_sitting_in_the_grass_facing_the_water_and_boats.jpg}} \\
            \raisebox{-0.5\height}{\rotatebox{90}{\small CACFG++}} & \raisebox{-0.5\height}{\includegraphics[width=0.15\linewidth]{arxiv_images/SDXL_COCO_Images/SDXL_COCO_25_CACFGpp_A_YOUNG_LADY_IS_ON_THE_COURT_PLAYING_TENNIS.jpg}} & \raisebox{-0.5\height}{\includegraphics[width=0.15\linewidth]{arxiv_images/SDXL_COCO_Images/SDXL_COCO_25_CACFGpp_A_bed_sitting_in_a_room_next_to_a_lamp_and_chair.jpg}} & \raisebox{-0.5\height}{\includegraphics[width=0.15\linewidth]{arxiv_images/SDXL_COCO_Images/SDXL_COCO_25_CACFGpp_A_young_girl_prepares_to_throw_a_disc_through_a_hoop_while_other_watch_in_a_gym.jpg}} & \raisebox{-0.5\height}{\includegraphics[width=0.15\linewidth]{arxiv_images/SDXL_COCO_Images/SDXL_COCO_25_CACFGpp_An_elephant_and_baby_stand_next_to_each_other.jpg}} & \raisebox{-0.5\height}{\includegraphics[width=0.15\linewidth]{arxiv_images/SDXL_COCO_Images/SDXL_COCO_25_CACFGpp_Two_zebras_made_of_Lego_stand_in_a_field.jpg}} & \raisebox{-0.5\height}{\includegraphics[width=0.15\linewidth]{arxiv_images/SDXL_COCO_Images/SDXL_COCO_25_CACFGpp_a_bench_sitting_in_the_grass_facing_the_water_and_boats.jpg}}
        \end{tabular}
        }
        \caption{CFG++ vs CACFG++ (ours)}
    \end{subfigure}

    \vspace{1.5em}
    \begin{subfigure}{\linewidth}
        \centering
        \resizebox{\linewidth}{!}{%
        \begin{tabular}{@{}c@{\hspace{2pt}}c@{}c@{}c@{}c@{}c@{}c@{}}
            \raisebox{-0.5\height}{\rotatebox{90}{\small IntCFG}} & \raisebox{-0.5\height}{\includegraphics[width=0.15\linewidth]{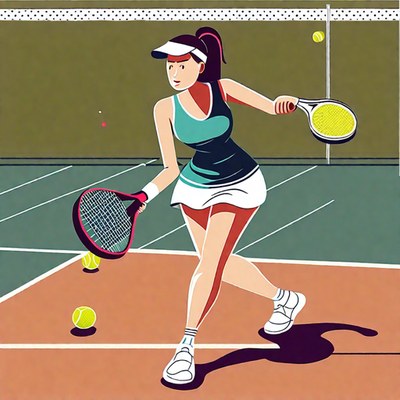}} & \raisebox{-0.5\height}{\includegraphics[width=0.15\linewidth]{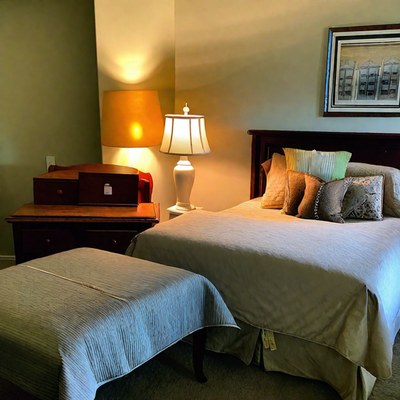}} & \raisebox{-0.5\height}{\includegraphics[width=0.15\linewidth]{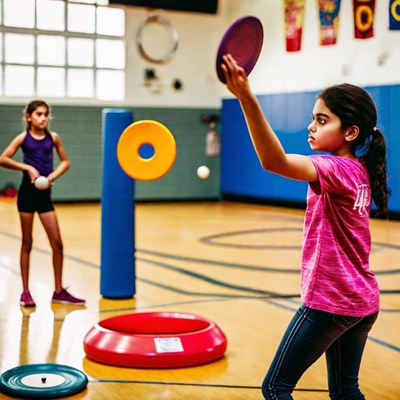}} & \raisebox{-0.5\height}{\includegraphics[width=0.15\linewidth]{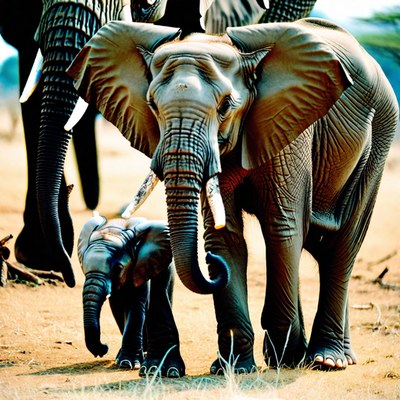}} & \raisebox{-0.5\height}{\includegraphics[width=0.15\linewidth]{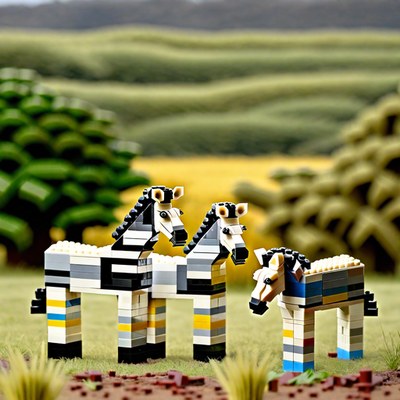}} & \raisebox{-0.5\height}{\includegraphics[width=0.15\linewidth]{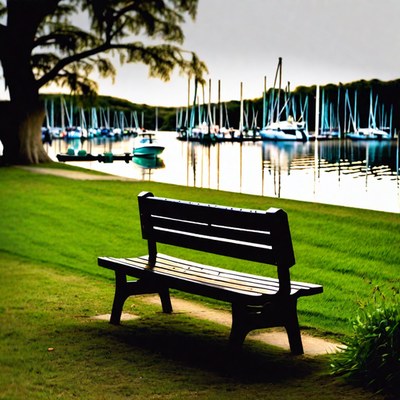}} \\
            \raisebox{-0.5\height}{\rotatebox{90}{\small IntCACFG}} & \raisebox{-0.5\height}{\includegraphics[width=0.15\linewidth]{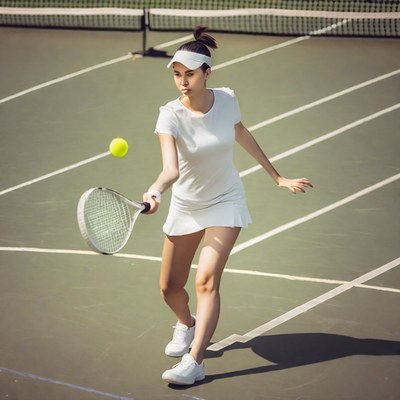}} & \raisebox{-0.5\height}{\includegraphics[width=0.15\linewidth]{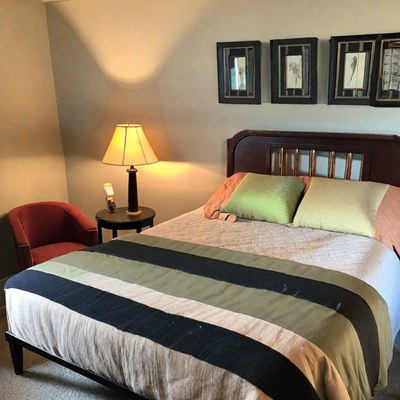}} & \raisebox{-0.5\height}{\includegraphics[width=0.15\linewidth]{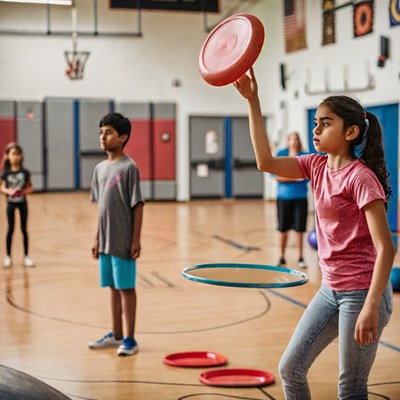}} & \raisebox{-0.5\height}{\includegraphics[width=0.15\linewidth]{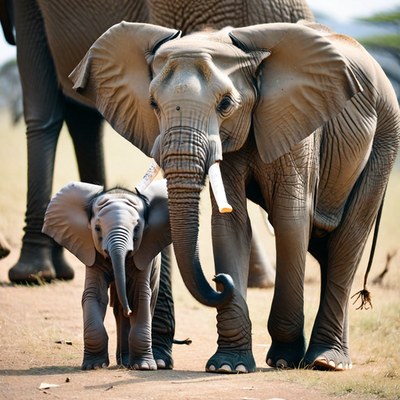}} & \raisebox{-0.5\height}{\includegraphics[width=0.15\linewidth]{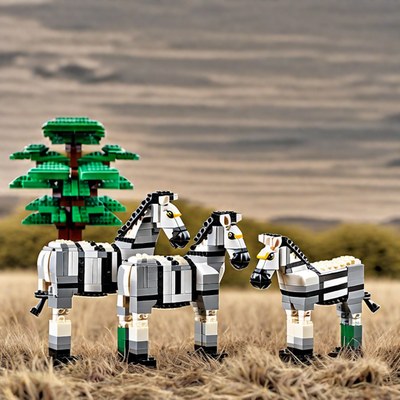}} & \raisebox{-0.5\height}{\includegraphics[width=0.15\linewidth]{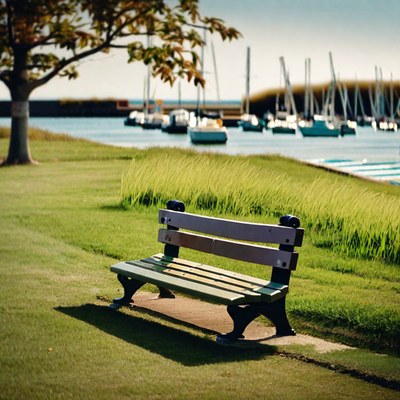}}
        \end{tabular}
        }
        \caption{IntCFG vs IntCACFG (ours)}
    \end{subfigure}
    \caption{CFG \citep{ClassifierFreeGuidance_Ho2022}, CFG++ \citep{CFG++_Chung2025}, and IntCFG \citep{IntCFG_Kynkaanniemi2024} visualised against their curvature-aware variants (ours): CACFG, CACFG++, IntCACFG. Samples are generated using SDXL \citep{SDXL_Podell2024} targeting the COCO \citep{COCO_Lin2014} dataset with a guidance strength of $\omega=25$.}
    \label{fig:guidance_strength_grid}
\end{figure}

\begin{figure}[h]
    \centering
    \resizebox{0.7\linewidth}{!}{%
    \begin{tabular}{@{}c@{\hspace{2pt}}c@{}c@{}c@{}c@{}}
        & \makebox[0.15\linewidth][c]{\small $\omega = 10$} & \makebox[0.15\linewidth][c]{\small $\omega = 15$} & \makebox[0.15\linewidth][c]{\small $\omega = 20$} & \makebox[0.15\linewidth][c]{\small $\omega = 25$} \\
        \noalign{\vskip 4pt}
        \raisebox{-0.5\height}{\rotatebox{90}{\small CFG}} & \raisebox{-0.5\height}{\includegraphics[width=0.15\linewidth]{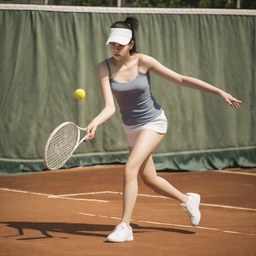}} & \raisebox{-0.5\height}{\includegraphics[width=0.15\linewidth]{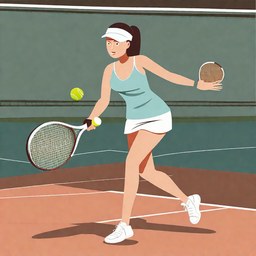}} & \raisebox{-0.5\height}{\includegraphics[width=0.15\linewidth]{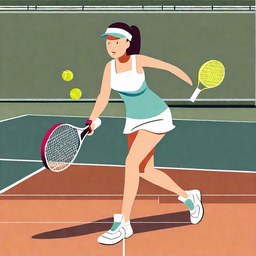}} & \raisebox{-0.5\height}{\includegraphics[width=0.15\linewidth]{arxiv_images/SDXL_COCO_Images/SDXL_COCO_25_CFGpp_A_YOUNG_LADY_IS_ON_THE_COURT_PLAYING_TENNIS.jpg}} \\
        \raisebox{-0.5\height}{\rotatebox{90}{\small CACFG}} & \raisebox{-0.5\height}{\includegraphics[width=0.15\linewidth]{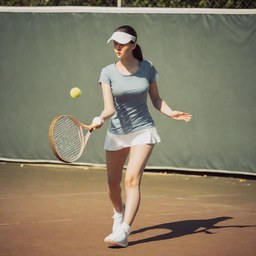}} & \raisebox{-0.5\height}{\includegraphics[width=0.15\linewidth]{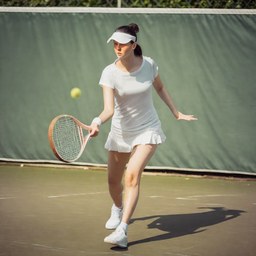}} & \raisebox{-0.5\height}{\includegraphics[width=0.15\linewidth]{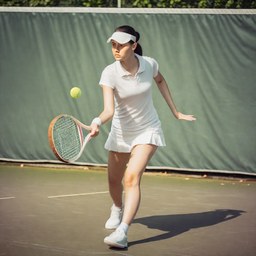}} & \raisebox{-0.5\height}{\includegraphics[width=0.15\linewidth]{arxiv_images/SDXL_COCO_Images/SDXL_COCO_25_CACFGpp_A_YOUNG_LADY_IS_ON_THE_COURT_PLAYING_TENNIS.jpg}} \\
        \noalign{\vskip 4pt}
        \raisebox{-0.5\height}{\rotatebox{90}{\small CFG}} & \raisebox{-0.5\height}{\includegraphics[width=0.15\linewidth]{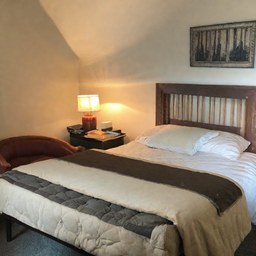}} & \raisebox{-0.5\height}{\includegraphics[width=0.15\linewidth]{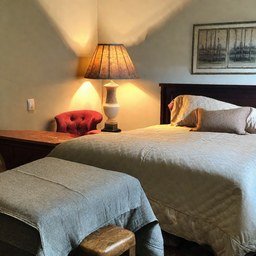}} & \raisebox{-0.5\height}{\includegraphics[width=0.15\linewidth]{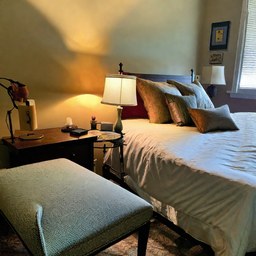}} & \raisebox{-0.5\height}{\includegraphics[width=0.15\linewidth]{arxiv_images/SDXL_COCO_Images/SDXL_COCO_25_CFGpp_A_bed_sitting_in_a_room_next_to_a_lamp_and_chair.jpg}} \\
        \raisebox{-0.5\height}{\rotatebox{90}{\small CACFG}} & \raisebox{-0.5\height}{\includegraphics[width=0.15\linewidth]{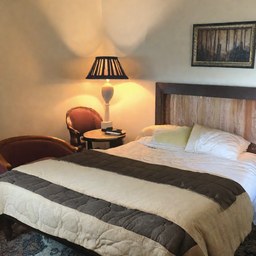}} & \raisebox{-0.5\height}{\includegraphics[width=0.15\linewidth]{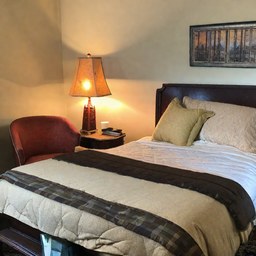}} & \raisebox{-0.5\height}{\includegraphics[width=0.15\linewidth]{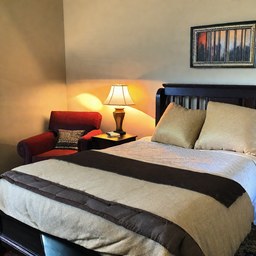}} & \raisebox{-0.5\height}{\includegraphics[width=0.15\linewidth]{arxiv_images/SDXL_COCO_Images/SDXL_COCO_25_CACFGpp_A_bed_sitting_in_a_room_next_to_a_lamp_and_chair.jpg}} \\
        \noalign{\vskip 4pt}
        \raisebox{-0.5\height}{\rotatebox{90}{\small CFG}} & \raisebox{-0.5\height}{\includegraphics[width=0.15\linewidth]{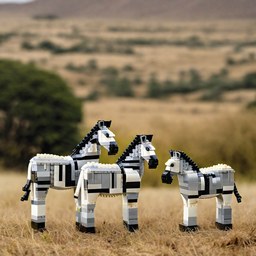}} & \raisebox{-0.5\height}{\includegraphics[width=0.15\linewidth]{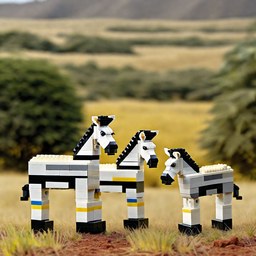}} & \raisebox{-0.5\height}{\includegraphics[width=0.15\linewidth]{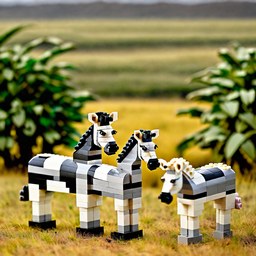}} & \raisebox{-0.5\height}{\includegraphics[width=0.15\linewidth]{arxiv_images/SDXL_COCO_Images/SDXL_COCO_25_CFGpp_Two_zebras_made_of_Lego_stand_in_a_field.jpg}} \\
        \raisebox{-0.5\height}{\rotatebox{90}{\small CACFG}} & \raisebox{-0.5\height}{\includegraphics[width=0.15\linewidth]{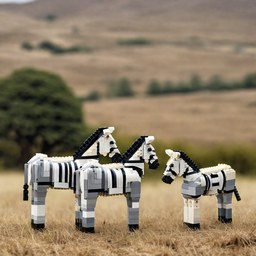}} & \raisebox{-0.5\height}{\includegraphics[width=0.15\linewidth]{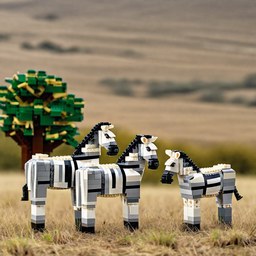}} & \raisebox{-0.5\height}{\includegraphics[width=0.15\linewidth]{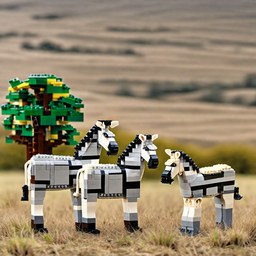}} & \raisebox{-0.5\height}{\includegraphics[width=0.15\linewidth]{arxiv_images/SDXL_COCO_Images/SDXL_COCO_25_CACFGpp_Two_zebras_made_of_Lego_stand_in_a_field.jpg}} \\
        \noalign{\vskip 4pt}
        \raisebox{-0.5\height}{\rotatebox{90}{\small CFG}} & \raisebox{-0.5\height}{\includegraphics[width=0.15\linewidth]{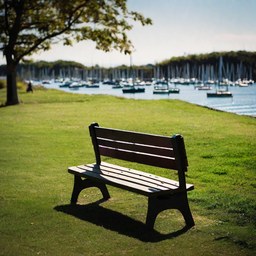}} & \raisebox{-0.5\height}{\includegraphics[width=0.15\linewidth]{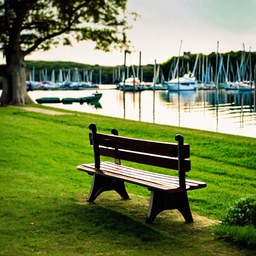}} & \raisebox{-0.5\height}{\includegraphics[width=0.15\linewidth]{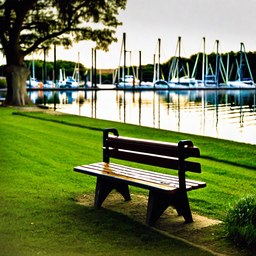}} & \raisebox{-0.5\height}{\includegraphics[width=0.15\linewidth]{arxiv_images/SDXL_COCO_Images/SDXL_COCO_25_CFGpp_a_bench_sitting_in_the_grass_facing_the_water_and_boats.jpg}} \\
        \raisebox{-0.5\height}{\rotatebox{90}{\small CACFG}} & \raisebox{-0.5\height}{\includegraphics[width=0.15\linewidth]{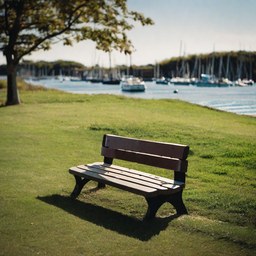}} & \raisebox{-0.5\height}{\includegraphics[width=0.15\linewidth]{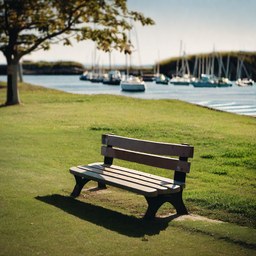}} & \raisebox{-0.5\height}{\includegraphics[width=0.15\linewidth]{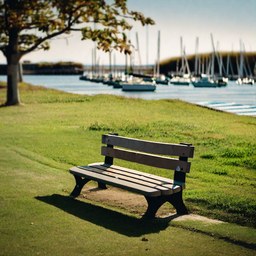}} & \raisebox{-0.5\height}{\includegraphics[width=0.15\linewidth]{arxiv_images/SDXL_COCO_Images/SDXL_COCO_25_CACFGpp_a_bench_sitting_in_the_grass_facing_the_water_and_boats.jpg}}
    \end{tabular}
    }
    \caption{CFG \citep{ClassifierFreeGuidance_Ho2022} and CACFG (ours) samples generated using SDXL \citep{SDXL_Podell2024} targeting the COCO \citep{COCO_Lin2014} dataset. Guidance scale varies across $\omega = 10, 15, 20, 25$.}
    \label{fig:SDXL COCO Guidance Strength}
\end{figure}

\clearpage
\subsection{Quantitative results}
Here we provide the numeric values that the Figures~\ref{fig:Metric plot sd1.5} and \ref{fig:Metric plot sdxl} were created from.

\begin{table}[h]
\centering
\caption{Image quality for SD1.5 samples generated for ImageNet and COCO datasets across a range of guidance strengths and schedules. Best result for each dataset and guidance strength is \textcolor{ForestGreen}{\textbf{bold}}, and the best value between curvature-aware and non-curvature-aware guidance is given in \textcolor{Green}{light green}, i.e., colour denotes whether CA or non-CA methods were favoured.}
\label{tab:SD1.5 image quality table}
\resizebox{\textwidth}{!}{%
\small
\begin{tabular}{l l l ccccccc cc}
\toprule
\textbf{Dataset} & \textbf{Strength} & \textbf{Guidance} & \textbf{FID} $\downarrow$ & \textbf{NIQE} $\downarrow$ & \textbf{BRISQUE} $\downarrow$ & \textbf{TReS} $\uparrow$ & \textbf{CLIP} $\uparrow$ & \textbf{HPSv2} $\uparrow$ & \textbf{Density} $\uparrow$ & \textbf{Coverage} $\uparrow$ & \textbf{Saturation} $\downarrow$ \\
\midrule

\multirow[c]{24}{*}{ImageNet}
 & \multirow[c]{6}{*}{5}
   & CFG      & $\textcolor{ForestGreen}{\mathbf{13.67}}$ & $4.06$ & $14.98$ & $83.24$ & $\textcolor{Green}{24.851}$ & $\textcolor{Green}{25.076}$ & $\textcolor{ForestGreen}{\mathbf{1.128}}$ & $0.453$ & $0.299$ \\
 & & CFG++    & $\textcolor{Green}{13.88}$ & $4.05$ & $15.27$ & $83.31$ & $24.842$ & $\textcolor{ForestGreen}{\mathbf{25.096}}$ & $\textcolor{Green}{1.093}$ & $0.447$ & $0.291$ \\
 & & IntCFG   & $\textcolor{Green}{13.86}$ & $4.07$ & $15.00$ & $83.04$ & $24.859$ & $\textcolor{Green}{25.079}$ & $1.112$ & $\textcolor{Green}{0.449}$ & $0.295$ \\
 \noalign{\vskip -1.5pt}
 \cmidrule[0.1pt](lr){4-12}
 \noalign{\vskip -1.5pt}
 & & CACFG    & $13.70$ & $\textcolor{Green}{4.05}$ & $\textcolor{ForestGreen}{\mathbf{14.83}}$ & $\textcolor{ForestGreen}{\mathbf{83.47}}$ & $24.849$ & $25.070$ & $1.124$ & $\textcolor{ForestGreen}{\mathbf{0.454}}$ & $\textcolor{Green}{0.295}$ \\
 & & CACFG++  & $13.90$ & $\textcolor{ForestGreen}{\mathbf{4.04}}$ & $\textcolor{Green}{15.19}$ & $\textcolor{Green}{83.44}$ & $\textcolor{Green}{24.843}$ & $25.090$ & $1.090$ & $\textcolor{Green}{0.449}$ & $\textcolor{ForestGreen}{\mathbf{0.289}}$ \\
 & & IntCACFG & $\textcolor{Green}{13.86}$ & $\textcolor{Green}{4.06}$ & $\textcolor{Green}{14.87}$ & $\textcolor{Green}{83.29}$ & $\textcolor{ForestGreen}{\mathbf{24.862}}$ & $25.072$ & $\textcolor{Green}{1.113}$ & $\textcolor{Green}{0.449}$ & $\textcolor{Green}{0.290}$ \\
 \cmidrule(lr){2-12}
 & \multirow[c]{6}{*}{7.5}
   & CFG      & $15.12$ & $4.27$ & $15.47$ & $82.80$ & $\textcolor{ForestGreen}{\mathbf{25.251}}$ & $\textcolor{Green}{25.680}$ & $1.139$ & $0.443$ & $0.338$ \\
 & & CFG++    & $14.99$ & $4.25$ & $15.76$ & $82.71$ & $25.218$ & $\textcolor{ForestGreen}{\mathbf{25.709}}$ & $\textcolor{Green}{1.103}$ & $\textcolor{Green}{0.440}$ & $0.330$ \\
 & & IntCFG   & $15.25$ & $4.31$ & $15.57$ & $82.44$ & $25.231$ & $\textcolor{Green}{25.658}$ & $1.113$ & $0.439$ & $0.333$ \\
 \noalign{\vskip -1.5pt}
 \cmidrule[0.1pt](lr){4-12}
 \noalign{\vskip -1.5pt}
 & & CACFG    & $\textcolor{Green}{14.97}$ & $\textcolor{Green}{4.22}$ & $\textcolor{ForestGreen}{\mathbf{14.59}}$ & $\textcolor{ForestGreen}{\mathbf{83.80}}$ & $25.243$ & $25.653$ & $\textcolor{ForestGreen}{\mathbf{1.140}}$ & $\textcolor{ForestGreen}{\mathbf{0.445}}$ & $\textcolor{Green}{0.316}$ \\
 & & CACFG++  & $\textcolor{ForestGreen}{\mathbf{14.94}}$ & $\textcolor{ForestGreen}{\mathbf{4.21}}$ & $\textcolor{Green}{14.93}$ & $\textcolor{Green}{83.72}$ & $\textcolor{Green}{25.223}$ & $25.696$ & $1.102$ & $\textcolor{Green}{0.440}$ & $\textcolor{Green}{0.314}$ \\
 & & IntCACFG & $\textcolor{Green}{15.07}$ & $\textcolor{Green}{4.25}$ & $\textcolor{Green}{14.64}$ & $\textcolor{Green}{83.69}$ & $\textcolor{Green}{25.235}$ & $25.635$ & $\textcolor{Green}{1.123}$ & $\textcolor{Green}{0.441}$ & $\textcolor{ForestGreen}{\mathbf{0.309}}$ \\
\cmidrule(lr){2-12}
 & \multirow[c]{6}{*}{10}
   & CFG      & $16.22$ & $4.44$ & $17.13$ & $81.41$ & $25.391$ & $\textcolor{Green}{25.967}$ & $1.084$ & $0.424$ & $0.371$ \\
 & & CFG++    & $16.35$ & $4.42$ & $17.21$ & $81.63$ & $25.372$ & $\textcolor{ForestGreen}{\mathbf{25.993}}$ & $1.068$ & $0.421$ & $0.360$ \\
 & & IntCFG   & $16.38$ & $4.47$ & $17.46$ & $80.98$ & $\textcolor{Green}{25.376}$ & $\textcolor{Green}{25.954}$ & $1.069$ & $0.421$ & $0.367$ \\
 \noalign{\vskip -1.5pt}
 \cmidrule[0.1pt](lr){4-12}
 \noalign{\vskip -1.5pt}
 & & CACFG    & $\textcolor{ForestGreen}{\mathbf{15.73}}$ & $\textcolor{ForestGreen}{\mathbf{4.32}}$ & $\textcolor{ForestGreen}{\mathbf{14.66}}$ & $\textcolor{Green}{84.10}$ & $\textcolor{ForestGreen}{\mathbf{25.395}}$ & $25.933$ & $\textcolor{ForestGreen}{\mathbf{1.119}}$ & $\textcolor{ForestGreen}{\mathbf{0.433}}$ & $\textcolor{Green}{0.321}$ \\
 & & CACFG++  & $\textcolor{Green}{15.94}$ & $\textcolor{ForestGreen}{\mathbf{4.32}}$ & $\textcolor{Green}{14.88}$ & $\textcolor{ForestGreen}{\mathbf{84.14}}$ & $\textcolor{Green}{25.381}$ & $25.970$ & $\textcolor{Green}{1.097}$ & $\textcolor{Green}{0.427}$ & $\textcolor{Green}{0.321}$ \\
 & & IntCACFG & $\textcolor{Green}{15.96}$ & $\textcolor{Green}{4.34}$ & $\textcolor{Green}{14.75}$ & $\textcolor{Green}{84.02}$ & $25.359$ & $25.900$ & $\textcolor{Green}{1.109}$ & $\textcolor{Green}{0.427}$ & $\textcolor{ForestGreen}{\mathbf{0.313}}$ \\
\cmidrule(lr){2-12}
 & \multirow[c]{6}{*}{15}
   & CFG      & $17.78$ & $4.75$ & $21.95$ & $78.14$ & $25.491$ & $\textcolor{Green}{26.214}$ & $0.994$ & $0.391$ & $0.424$ \\
 & & CFG++    & $18.01$ & $4.71$ & $22.12$ & $78.17$ & $25.483$ & $\textcolor{ForestGreen}{\mathbf{26.221}}$ & $0.961$ & $0.388$ & $0.410$ \\
 & & IntCFG   & $18.21$ & $4.78$ & $22.39$ & $77.81$ & $25.478$ & $\textcolor{Green}{26.175}$ & $0.969$ & $0.389$ & $0.416$ \\
 \noalign{\vskip -1.5pt}
 \cmidrule[0.1pt](lr){4-12}
 \noalign{\vskip -1.5pt}
 & & CACFG    & $\textcolor{ForestGreen}{\mathbf{16.52}}$ & $\textcolor{ForestGreen}{\mathbf{4.42}}$ & $\textcolor{ForestGreen}{\mathbf{14.80}}$ & $\textcolor{Green}{84.60}$ & $\textcolor{Green}{25.540}$ & $26.175$ & $\textcolor{ForestGreen}{\mathbf{1.063}}$ & $\textcolor{ForestGreen}{\mathbf{0.413}}$ & $\textcolor{Green}{0.317}$ \\
 & & CACFG++  & $\textcolor{Green}{16.97}$ & $\textcolor{Green}{4.45}$ & $\textcolor{Green}{15.01}$ & $\textcolor{ForestGreen}{\mathbf{84.93}}$ & $\textcolor{ForestGreen}{\mathbf{25.549}}$ & $26.192$ & $\textcolor{Green}{1.043}$ & $\textcolor{Green}{0.404}$ & $\textcolor{Green}{0.315}$ \\
 & & IntCACFG & $\textcolor{Green}{16.71}$ & $\textcolor{Green}{4.48}$ & $\textcolor{Green}{15.09}$ & $\textcolor{Green}{84.76}$ & $\textcolor{Green}{25.526}$ & $26.122$ & $\textcolor{Green}{1.047}$ & $\textcolor{Green}{0.408}$ & $\textcolor{ForestGreen}{\mathbf{0.307}}$ \\

\midrule

\multirow[c]{24}{*}{COCO}
 & \multirow[c]{6}{*}{5}
   & CFG      & $12.69$ & $3.62$ & $14.82$ & $83.84$ & $26.245$ & $\textcolor{Green}{25.310}$ & $0.893$ & $0.520$ & $0.296$ \\
 & & CFG++    & $12.68$ & $\textcolor{ForestGreen}{\mathbf{3.61}}$ & $14.83$ & $83.93$ & $\textcolor{ForestGreen}{\mathbf{26.264}}$ & $25.386$ & $0.883$ & $\textcolor{Green}{0.511}$ & $0.289$ \\
 & & IntCFG   & $12.96$ & $3.63$ & $14.81$ & $83.56$ & $\textcolor{Green}{26.243}$ & $\textcolor{Green}{25.312}$ & $\textcolor{Green}{0.890}$ & $\textcolor{Green}{0.514}$ & $0.294$ \\
 \noalign{\vskip -1.5pt}
 \cmidrule[0.1pt](lr){4-12}
 \noalign{\vskip -1.5pt}
 & & CACFG    & $\textcolor{Green}{12.67}$ & $\textcolor{ForestGreen}{\mathbf{3.61}}$ & $\textcolor{Green}{14.57}$ & $\textcolor{Green}{84.11}$ & $\textcolor{Green}{26.247}$ & $25.303$ & $\textcolor{ForestGreen}{\mathbf{0.904}}$ & $\textcolor{ForestGreen}{\mathbf{0.523}}$ & $\textcolor{Green}{0.292}$ \\
 & & CACFG++  & $\textcolor{ForestGreen}{\mathbf{12.65}}$ & $\textcolor{ForestGreen}{\mathbf{3.61}}$ & $\textcolor{Green}{14.71}$ & $\textcolor{ForestGreen}{\mathbf{84.12}}$ & $26.258$ & $\textcolor{ForestGreen}{\mathbf{25.389}}$ & $\textcolor{Green}{0.886}$ & $0.509$ & $\textcolor{ForestGreen}{\mathbf{0.287}}$ \\
 & & IntCACFG & $\textcolor{Green}{12.84}$ & $\textcolor{Green}{3.62}$ & $\textcolor{ForestGreen}{\mathbf{14.56}}$ & $\textcolor{Green}{83.91}$ & $26.237$ & $25.310$ & $\textcolor{Green}{0.890}$ & $\textcolor{Green}{0.514}$ & $\textcolor{Green}{0.288}$ \\
\cmidrule(lr){2-12}
 & \multirow[c]{6}{*}{7.5}
   & CFG      & $15.44$ & $3.80$ & $16.00$ & $82.96$ & $\textcolor{Green}{26.495}$ & $\textcolor{Green}{26.011}$ & $0.885$ & $0.496$ & $0.335$ \\
 & & CFG++    & $15.19$ & $3.76$ & $16.02$ & $82.90$ & $26.498$ & $\textcolor{ForestGreen}{\mathbf{26.109}}$ & $0.864$ & $0.482$ & $0.327$ \\
 & & IntCFG   & $15.51$ & $3.81$ & $16.01$ & $82.56$ & $26.488$ & $\textcolor{Green}{26.009}$ & $0.872$ & $0.482$ & $0.331$ \\
 \noalign{\vskip -1.5pt}
 \cmidrule[0.1pt](lr){4-12}
 \noalign{\vskip -1.5pt}
 & & CACFG    & $\textcolor{ForestGreen}{\mathbf{14.73}}$ & $\textcolor{Green}{3.73}$ & $\textcolor{Green}{14.63}$ & $\textcolor{ForestGreen}{\mathbf{84.50}}$ & $\textcolor{Green}{26.495}$ & $26.010$ & $\textcolor{ForestGreen}{\mathbf{0.902}}$ & $\textcolor{ForestGreen}{\mathbf{0.497}}$ & $\textcolor{Green}{0.310}$ \\
 & & CACFG++  & $\textcolor{Green}{14.82}$ & $\textcolor{ForestGreen}{\mathbf{3.72}}$ & $\textcolor{Green}{14.92}$ & $\textcolor{Green}{84.24}$ & $\textcolor{ForestGreen}{\mathbf{26.517}}$ & $26.107$ & $\textcolor{Green}{0.880}$ & $\textcolor{Green}{0.492}$ & $\textcolor{Green}{0.309}$ \\
 & & IntCACFG & $\textcolor{Green}{14.89}$ & $\textcolor{Green}{3.75}$ & $\textcolor{ForestGreen}{\mathbf{14.60}}$ & $\textcolor{Green}{84.32}$ & $\textcolor{Green}{26.499}$ & $25.993$ & $\textcolor{Green}{0.882}$ & $\textcolor{Green}{0.490}$ & $\textcolor{ForestGreen}{\mathbf{0.304}}$ \\
 \cmidrule(lr){2-12}
 & \multirow[c]{6}{*}{10}
   & CFG      & $17.09$ & $3.93$ & $17.88$ & $81.17$ & $26.567$ & $26.335$ & $0.850$ & $0.465$ & $0.368$ \\
 & & CFG++    & $16.98$ & $3.90$ & $17.70$ & $81.48$ & $26.574$ & $26.432$ & $0.834$ & $0.463$ & $0.358$ \\
 & & IntCFG   & $17.28$ & $3.94$ & $17.98$ & $80.70$ & $26.558$ & $\textcolor{Green}{26.338}$ & $0.820$ & $0.452$ & $0.366$ \\
 \noalign{\vskip -1.5pt}
 \cmidrule[0.1pt](lr){4-12}
 \noalign{\vskip -1.5pt}
 & & CACFG    & $\textcolor{ForestGreen}{\mathbf{15.74}}$ & $\textcolor{ForestGreen}{\mathbf{3.78}}$ & $\textcolor{Green}{14.61}$ & $\textcolor{Green}{84.90}$ & $\textcolor{Green}{26.614}$ & $\textcolor{Green}{26.349}$ & $\textcolor{ForestGreen}{\mathbf{0.872}}$ & $\textcolor{ForestGreen}{\mathbf{0.480}}$ & $\textcolor{Green}{0.313}$ \\
 & & CACFG++  & $\textcolor{Green}{16.03}$ & $\textcolor{ForestGreen}{\mathbf{3.78}}$ & $\textcolor{Green}{14.81}$ & $\textcolor{ForestGreen}{\mathbf{84.93}}$ & $\textcolor{Green}{26.612}$ & $\textcolor{ForestGreen}{\mathbf{26.466}}$ & $\textcolor{Green}{0.842}$ & $\textcolor{Green}{0.471}$ & $\textcolor{Green}{0.314}$ \\
 & & IntCACFG & $\textcolor{Green}{15.94}$ & $\textcolor{Green}{3.81}$ & $\textcolor{ForestGreen}{\mathbf{14.58}}$ & $\textcolor{Green}{84.92}$ & $\textcolor{ForestGreen}{\mathbf{26.621}}$ & $26.328$ & $\textcolor{Green}{0.847}$ & $\textcolor{Green}{0.476}$ & $\textcolor{ForestGreen}{\mathbf{0.306}}$ \\
\cmidrule(lr){2-12}
 & \multirow[c]{6}{*}{15}
   & CFG      & $19.36$ & $4.20$ & $22.01$ & $77.05$ & $26.610$ & $26.597$ & $0.809$ & $0.433$ & $0.419$ \\
 & & CFG++    & $19.44$ & $4.17$ & $22.05$ & $77.33$ & $26.634$ & $26.717$ & $0.786$ & $0.432$ & $0.407$ \\
 & & IntCFG   & $19.70$ & $4.24$ & $22.27$ & $76.64$ & $26.614$ & $26.580$ & $0.779$ & $0.423$ & $0.415$ \\
 \noalign{\vskip -1.5pt}
 \cmidrule[0.1pt](lr){4-12}
 \noalign{\vskip -1.5pt}
 & & CACFG    & $\textcolor{ForestGreen}{\mathbf{16.78}}$ & $\textcolor{ForestGreen}{\mathbf{3.83}}$ & $\textcolor{ForestGreen}{\mathbf{14.51}}$ & $\textcolor{Green}{85.81}$ & $\textcolor{Green}{26.744}$ & $\textcolor{Green}{26.683}$ & $\textcolor{ForestGreen}{\mathbf{0.820}}$ & $\textcolor{ForestGreen}{\mathbf{0.451}}$ & $\textcolor{Green}{0.305}$ \\
 & & CACFG++  & $\textcolor{Green}{17.11}$ & $\textcolor{Green}{3.87}$ & $\textcolor{Green}{14.59}$ & $\textcolor{ForestGreen}{\mathbf{86.34}}$ & $\textcolor{Green}{26.754}$ & $\textcolor{ForestGreen}{\mathbf{26.766}}$ & $\textcolor{Green}{0.814}$ & $\textcolor{Green}{0.448}$ & $\textcolor{Green}{0.304}$ \\
 & & IntCACFG & $\textcolor{Green}{16.87}$ & $\textcolor{Green}{3.88}$ & $\textcolor{Green}{14.52}$ & $\textcolor{Green}{86.00}$ & $\textcolor{ForestGreen}{\mathbf{26.763}}$ & $\textcolor{Green}{26.639}$ & $\textcolor{Green}{0.814}$ & $\textcolor{ForestGreen}{\mathbf{0.451}}$ & $\textcolor{ForestGreen}{\mathbf{0.297}}$ \\

\bottomrule
\end{tabular}%
}
\end{table}

\begin{table}[h]
\centering
\caption{Image quality for SDXL samples generated for ImageNet and COCO datasets across a range of guidance strengths and schedules. Best result for each dataset and guidance strength is \textcolor{ForestGreen}{\textbf{bold}}, and the best value between curvature-aware and non-curvature-aware guidance is given in \textcolor{Green}{light green}, i.e., colour denotes whether CA or non-CA methods were favoured.}
\label{tab:SDXL image quality table}
\resizebox{\textwidth}{!}{%
\small
\begin{tabular}{l l l ccccccc cc}
\toprule
\textbf{Dataset} & \textbf{Strength} & \textbf{Guidance} & \textbf{FID} $\downarrow$ & \textbf{NIQE} $\downarrow$ & \textbf{BRISQUE} $\downarrow$ & \textbf{TReS} $\uparrow$ & \textbf{CLIP} $\uparrow$ & \textbf{HPSv2} $\uparrow$ & \textbf{Density} $\uparrow$ & \textbf{Coverage} $\uparrow$ & \textbf{Saturation} $\downarrow$ \\
\midrule

\multirow[c]{24}{*}{ImageNet}
 & \multirow[c]{6}{*}{10}
   & CFG      & $29.60$ & $4.66$ & $27.42$ & $\textcolor{Green}{85.59}$ & $25.548$ & $\textcolor{ForestGreen}{\mathbf{26.815}}$ & $\textcolor{ForestGreen}{\mathbf{1.137}}$ & $\textcolor{ForestGreen}{\mathbf{0.182}}$ & $0.240$ \\
 & & CFG++    & $29.74$ & $4.63$ & $27.21$ & $\textcolor{ForestGreen}{\mathbf{86.03}}$ & $\textcolor{Green}{25.508}$ & $\textcolor{Green}{26.762}$ & $1.071$ & $\textcolor{Green}{0.174}$ & $0.231$ \\
 & & IntCFG   & $30.12$ & $4.67$ & $27.55$ & $85.91$ & $\textcolor{Green}{25.508}$ & $\textcolor{Green}{26.722}$ & $1.074$ & $0.172$ & $0.234$ \\
 \noalign{\vskip -1.5pt}
 \cmidrule[0.1pt](lr){4-12}
 \noalign{\vskip -1.5pt}
 & & CACFG    & $\textcolor{ForestGreen}{\mathbf{29.47}}$ & $\textcolor{Green}{4.64}$ & $\textcolor{Green}{27.35}$ & $85.57$ & $\textcolor{ForestGreen}{\mathbf{25.563}}$ & $26.789$ & $1.102$ & $0.177$ & $\textcolor{Green}{0.229}$ \\
 & & CACFG++  & $\textcolor{Green}{29.55}$ & $\textcolor{ForestGreen}{\mathbf{4.59}}$ & $\textcolor{ForestGreen}{\mathbf{27.04}}$ & $86.01$ & $25.491$ & $26.754$ & $\textcolor{Green}{1.074}$ & $\textcolor{Green}{0.174}$ & $\textcolor{Green}{0.227}$ \\
 & & IntCACFG & $\textcolor{Green}{29.91}$ & $\textcolor{Green}{4.64}$ & $\textcolor{Green}{27.33}$ & $\textcolor{Green}{86.01}$ & $25.488$ & $26.675$ & $\textcolor{Green}{1.078}$ & $\textcolor{Green}{0.174}$ & $\textcolor{ForestGreen}{\mathbf{0.223}}$ \\
 \cmidrule(lr){2-12}
 & \multirow[c]{6}{*}{15}
   & CFG      & $30.95$ & $4.90$ & $28.71$ & $86.48$ & $\textcolor{ForestGreen}{\mathbf{25.735}}$ & $\textcolor{Green}{27.173}$ & $1.000$ & $0.162$ & $0.280$ \\
 & & CFG++    & $31.30$ & $4.82$ & $28.32$ & $86.95$ & $\textcolor{Green}{25.710}$ & $\textcolor{ForestGreen}{\mathbf{27.184}}$ & $\textcolor{Green}{0.990}$ & $\textcolor{Green}{0.160}$ & $0.269$ \\
 & & IntCFG   & $31.76$ & $4.92$ & $28.73$ & $86.75$ & $25.633$ & $\textcolor{Green}{27.079}$ & $\textcolor{Green}{1.001}$ & $\textcolor{Green}{0.161}$ & $0.271$ \\
 \noalign{\vskip -1.5pt}
 \cmidrule[0.1pt](lr){4-12}
 \noalign{\vskip -1.5pt}
 & & CACFG    & $\textcolor{ForestGreen}{\mathbf{30.34}}$ & $\textcolor{ForestGreen}{\mathbf{4.75}}$ & $\textcolor{Green}{27.97}$ & $\textcolor{Green}{86.85}$ & $25.719$ & $27.116$ & $\textcolor{ForestGreen}{\mathbf{1.016}}$ & $\textcolor{ForestGreen}{\mathbf{0.164}}$ & $\textcolor{Green}{0.247}$ \\
 & & CACFG++  & $\textcolor{Green}{30.98}$ & $\textcolor{ForestGreen}{\mathbf{4.75}}$ & $\textcolor{Green}{27.82}$ & $\textcolor{Green}{87.14}$ & $25.707$ & $27.133$ & $0.969$ & $0.158$ & $\textcolor{Green}{0.248}$ \\
 & & IntCACFG & $\textcolor{Green}{31.43}$ & $\textcolor{Green}{4.77}$ & $\textcolor{ForestGreen}{\mathbf{27.68}}$ & $\textcolor{ForestGreen}{\mathbf{87.25}}$ & $\textcolor{Green}{25.689}$ & $27.021$ & $0.977$ & $0.160$ & $\textcolor{ForestGreen}{\mathbf{0.238}}$ \\
\cmidrule(lr){2-12}
 & \multirow[c]{6}{*}{20}
   & CFG      & $32.81$ & $5.07$ & $29.40$ & $86.34$ & $25.807$ & $\textcolor{Green}{27.249}$ & $0.947$ & $0.154$ & $0.309$ \\
 & & CFG++    & $33.16$ & $5.04$ & $29.01$ & $86.86$ & $\textcolor{ForestGreen}{\mathbf{25.850}}$ & $\textcolor{ForestGreen}{\mathbf{27.291}}$ & $0.908$ & $0.147$ & $0.298$ \\
 & & IntCFG   & $33.98$ & $5.16$ & $29.64$ & $86.65$ & $\textcolor{Green}{25.773}$ & $\textcolor{Green}{27.186}$ & $0.904$ & $0.146$ & $0.299$ \\
 \noalign{\vskip -1.5pt}
 \cmidrule[0.1pt](lr){4-12}
 \noalign{\vskip -1.5pt}
 & & CACFG    & $\textcolor{ForestGreen}{\mathbf{31.39}}$ & $\textcolor{Green}{4.82}$ & $\textcolor{Green}{28.10}$ & $\textcolor{Green}{87.37}$ & $\textcolor{Green}{25.828}$ & $27.223$ & $\textcolor{ForestGreen}{\mathbf{0.972}}$ & $\textcolor{ForestGreen}{\mathbf{0.156}}$ & $\textcolor{Green}{0.253}$ \\
 & & CACFG++  & $\textcolor{Green}{32.58}$ & $\textcolor{ForestGreen}{\mathbf{4.81}}$ & $\textcolor{Green}{27.97}$ & $\textcolor{Green}{87.69}$ & $25.817$ & $27.245$ & $\textcolor{Green}{0.911}$ & $\textcolor{Green}{0.148}$ & $\textcolor{Green}{0.255}$ \\
 & & IntCACFG & $\textcolor{Green}{32.36}$ & $\textcolor{Green}{4.82}$ & $\textcolor{ForestGreen}{\mathbf{27.67}}$ & $\textcolor{ForestGreen}{\mathbf{87.83}}$ & $25.750$ & $27.098$ & $\textcolor{Green}{0.941}$ & $\textcolor{Green}{0.153}$ & $\textcolor{ForestGreen}{\mathbf{0.244}}$ \\
\cmidrule(lr){2-12}
 & \multirow[c]{6}{*}{25}
   & CFG      & $34.78$ & $5.22$ & $30.31$ & $85.52$ & $\textcolor{ForestGreen}{\mathbf{25.910}}$ & $27.209$ & $0.835$ & $0.136$ & $0.332$ \\
 & & CFG++    & $35.90$ & $5.17$ & $29.77$ & $86.39$ & $\textcolor{Green}{25.877}$ & $27.246$ & $0.843$ & $0.137$ & $0.321$ \\
 & & IntCFG   & $36.98$ & $5.32$ & $30.43$ & $86.08$ & $25.799$ & $27.086$ & $0.780$ & $0.128$ & $0.321$ \\
 \noalign{\vskip -1.5pt}
 \cmidrule[0.1pt](lr){4-12}
 \noalign{\vskip -1.5pt}
 & & CACFG    & $\textcolor{ForestGreen}{\mathbf{32.40}}$ & $\textcolor{ForestGreen}{\mathbf{4.80}}$ & $\textcolor{Green}{27.91}$ & $\textcolor{Green}{87.51}$ & $25.894$ & $\textcolor{Green}{27.230}$ & $\textcolor{ForestGreen}{\mathbf{0.916}}$ & $\textcolor{ForestGreen}{\mathbf{0.147}}$ & $\textcolor{Green}{0.255}$ \\
 & & CACFG++  & $\textcolor{Green}{33.86}$ & $\textcolor{Green}{4.83}$ & $\textcolor{Green}{27.75}$ & $\textcolor{Green}{87.97}$ & $25.854$ & $\textcolor{ForestGreen}{\mathbf{27.276}}$ & $\textcolor{Green}{0.882}$ & $\textcolor{Green}{0.144}$ & $\textcolor{Green}{0.256}$ \\
 & & IntCACFG & $\textcolor{Green}{33.42}$ & $\textcolor{Green}{4.81}$ & $\textcolor{ForestGreen}{\mathbf{27.58}}$ & $\textcolor{ForestGreen}{\mathbf{88.25}}$ & $\textcolor{Green}{25.805}$ & $\textcolor{Green}{27.115}$ & $\textcolor{Green}{0.889}$ & $\textcolor{Green}{0.143}$ & $\textcolor{ForestGreen}{\mathbf{0.245}}$ \\

\midrule

\multirow[c]{24}{*}{COCO}
 & \multirow[c]{6}{*}{10}
   & CFG      & $30.40$ & $4.03$ & $27.60$ & $83.02$ & $26.774$ & $\textcolor{Green}{28.688}$ & $\textcolor{ForestGreen}{\mathbf{0.968}}$ & $\textcolor{ForestGreen}{\mathbf{0.181}}$ & $0.250$ \\
 & & CFG++    & $\textcolor{ForestGreen}{\mathbf{30.37}}$ & $4.01$ & $27.54$ & $83.49$ & $26.771$ & $\textcolor{ForestGreen}{\mathbf{28.722}}$ & $0.946$ & $\textcolor{Green}{0.180}$ & $0.243$ \\
 & & IntCFG   & $30.70$ & $4.04$ & $27.70$ & $83.30$ & $\textcolor{Green}{26.803}$ & $\textcolor{Green}{28.675}$ & $\textcolor{Green}{0.923}$ & $\textcolor{Green}{0.172}$ & $0.248$ \\
 \noalign{\vskip -1.5pt}
 \cmidrule[0.1pt](lr){4-12}
 \noalign{\vskip -1.5pt}
 & & CACFG    & $\textcolor{Green}{30.38}$ & $\textcolor{Green}{4.00}$ & $\textcolor{ForestGreen}{\mathbf{27.41}}$ & $\textcolor{Green}{83.09}$ & $\textcolor{ForestGreen}{\mathbf{26.816}}$ & $28.573$ & $0.894$ & $0.170$ & $\textcolor{Green}{0.237}$ \\
 & & CACFG++  & $30.47$ & $\textcolor{ForestGreen}{\mathbf{3.99}}$ & $\textcolor{ForestGreen}{\mathbf{27.41}}$ & $\textcolor{ForestGreen}{\mathbf{83.50}}$ & $\textcolor{Green}{26.786}$ & $28.658$ & $\textcolor{Green}{0.950}$ & $0.179$ & $\textcolor{Green}{0.238}$ \\
 & & IntCACFG & $\textcolor{Green}{30.39}$ & $\textcolor{Green}{4.02}$ & $\textcolor{Green}{27.48}$ & $\textcolor{Green}{83.31}$ & $26.737$ & $28.533$ & $0.903$ & $0.169$ & $\textcolor{ForestGreen}{\mathbf{0.234}}$ \\
\cmidrule(lr){2-12}
 & \multirow[c]{6}{*}{15}
   & CFG      & $31.30$ & $4.14$ & $27.89$ & $84.01$ & $26.926$ & $\textcolor{Green}{29.175}$ & $0.887$ & $\textcolor{Green}{0.168}$ & $0.289$ \\
 & & CFG++    & $31.48$ & $4.11$ & $27.62$ & $84.50$ & $26.889$ & $\textcolor{ForestGreen}{\mathbf{29.225}}$ & $0.869$ & $0.166$ & $0.281$ \\
 & & IntCFG   & $31.76$ & $4.17$ & $27.94$ & $84.21$ & $26.866$ & $\textcolor{Green}{29.158}$ & $\textcolor{Green}{0.886}$ & $\textcolor{Green}{0.168}$ & $0.286$ \\
 \noalign{\vskip -1.5pt}
 \cmidrule[0.1pt](lr){4-12}
 \noalign{\vskip -1.5pt}
 & & CACFG    & $\textcolor{ForestGreen}{\mathbf{31.05}}$ & $\textcolor{Green}{4.07}$ & $\textcolor{Green}{27.33}$ & $\textcolor{Green}{84.54}$ & $\textcolor{ForestGreen}{\mathbf{26.936}}$ & $28.929$ & $\textcolor{ForestGreen}{\mathbf{0.896}}$ & $0.166$ & $\textcolor{Green}{0.252}$ \\
 & & CACFG++  & $\textcolor{Green}{31.36}$ & $\textcolor{ForestGreen}{\mathbf{4.05}}$ & $\textcolor{ForestGreen}{\mathbf{27.31}}$ & $\textcolor{ForestGreen}{\mathbf{84.77}}$ & $\textcolor{Green}{26.925}$ & $29.069$ & $\textcolor{Green}{0.883}$ & $\textcolor{ForestGreen}{\mathbf{0.170}}$ & $\textcolor{Green}{0.257}$ \\
 & & IntCACFG & $\textcolor{Green}{31.33}$ & $\textcolor{Green}{4.06}$ & $\textcolor{Green}{27.48}$ & $\textcolor{Green}{84.71}$ & $\textcolor{Green}{26.885}$ & $28.933$ & $0.841$ & $0.161$ & $\textcolor{ForestGreen}{\mathbf{0.248}}$ \\
 \cmidrule(lr){2-12}
 & \multirow[c]{6}{*}{20}
   & CFG      & $32.27$ & $4.23$ & $28.24$ & $83.94$ & $26.911$ & $\textcolor{Green}{29.284}$ & $0.801$ & $0.155$ & $0.323$ \\
 & & CFG++    & $\textcolor{Green}{32.37}$ & $4.21$ & $27.84$ & $84.58$ & $26.928$ & $\textcolor{ForestGreen}{\mathbf{29.393}}$ & $0.814$ & $0.156$ & $0.312$ \\
 & & IntCFG   & $32.70$ & $4.25$ & $28.43$ & $84.04$ & $26.935$ & $\textcolor{Green}{29.280}$ & $0.767$ & $0.148$ & $0.319$ \\
 \noalign{\vskip -1.5pt}
 \cmidrule[0.1pt](lr){4-12}
 \noalign{\vskip -1.5pt}
 & & CACFG    & $\textcolor{ForestGreen}{\mathbf{32.00}}$ & $\textcolor{ForestGreen}{\mathbf{4.08}}$ & $\textcolor{Green}{27.36}$ & $\textcolor{Green}{85.32}$ & $\textcolor{Green}{26.941}$ & $29.047$ & $\textcolor{ForestGreen}{\mathbf{0.836}}$ & $\textcolor{ForestGreen}{\mathbf{0.159}}$ & $\textcolor{Green}{0.260}$ \\
 & & CACFG++  & $32.40$ & $\textcolor{ForestGreen}{\mathbf{4.08}}$ & $\textcolor{ForestGreen}{\mathbf{27.25}}$ & $\textcolor{ForestGreen}{\mathbf{85.55}}$ & $\textcolor{Green}{26.955}$ & $29.219$ & $\textcolor{Green}{0.831}$ & $\textcolor{Green}{0.158}$ & $\textcolor{Green}{0.265}$ \\
 & & IntCACFG & $\textcolor{Green}{32.19}$ & $\textcolor{Green}{4.09}$ & $\textcolor{Green}{27.28}$ & $\textcolor{Green}{85.42}$ & $\textcolor{ForestGreen}{\mathbf{26.959}}$ & $29.063$ & $\textcolor{Green}{0.799}$ & $\textcolor{Green}{0.153}$ & $\textcolor{ForestGreen}{\mathbf{0.254}}$ \\
\cmidrule(lr){2-12}
 & \multirow[c]{6}{*}{25}
   & CFG      & $33.35$ & $4.33$ & $28.62$ & $83.13$ & $26.897$ & $\textcolor{Green}{29.313}$ & $0.736$ & $0.140$ & $0.353$ \\
 & & CFG++    & $33.93$ & $4.29$ & $28.20$ & $84.12$ & $26.941$ & $\textcolor{ForestGreen}{\mathbf{29.429}}$ & $0.783$ & $0.150$ & $0.341$ \\
 & & IntCFG   & $34.22$ & $4.36$ & $29.02$ & $83.60$ & $26.957$ & $\textcolor{Green}{29.269}$ & $0.738$ & $0.140$ & $0.350$ \\
 \noalign{\vskip -1.5pt}
 \cmidrule[0.1pt](lr){4-12}
 \noalign{\vskip -1.5pt}
 & & CACFG    & $\textcolor{ForestGreen}{\mathbf{32.13}}$ & $\textcolor{ForestGreen}{\mathbf{4.09}}$ & $\textcolor{Green}{27.40}$ & $\textcolor{Green}{85.64}$ & $\textcolor{ForestGreen}{\mathbf{26.989}}$ & $29.099$ & $\textcolor{Green}{0.788}$ & $\textcolor{Green}{0.151}$ & $\textcolor{Green}{0.262}$ \\
 & & CACFG++  & $\textcolor{Green}{33.23}$ & $\textcolor{Green}{4.10}$ & $\textcolor{Green}{27.27}$ & $\textcolor{ForestGreen}{\mathbf{86.00}}$ & $\textcolor{Green}{26.985}$ & $29.251$ & $\textcolor{ForestGreen}{\mathbf{0.813}}$ & $\textcolor{ForestGreen}{\mathbf{0.152}}$ & $\textcolor{Green}{0.266}$ \\
 & & IntCACFG & $\textcolor{Green}{33.13}$ & $\textcolor{Green}{4.10}$ & $\textcolor{ForestGreen}{\mathbf{27.21}}$ & $\textcolor{Green}{85.87}$ & $\textcolor{Green}{26.975}$ & $29.102$ & $\textcolor{Green}{0.767}$ & $\textcolor{Green}{0.147}$ & $\textcolor{ForestGreen}{\mathbf{0.256}}$ \\

\bottomrule
\end{tabular}%
}
\end{table}

\clearpage
\subsection{Comparison with adaptive projected guidance (APG)}
\emph{Adaptive Projected Guidance (APG)} \citep{APG_Sadat2025} decomposes the CFG update direction, $\hat{\boldsymbol{x}}_0(t,\boldsymbol{x}_t,\xi)-\hat{\boldsymbol{x}}_0(t,\boldsymbol{x}_t,\Omega_Y)$, into components parallel and orthogonal to the conditional prediction. They show that manipulating the parallel and orthogonal gradient components allows the saturation and image quality to be controlled more effectively. In this section we compare APG against CFG and CACFG; however, unlike CFG++ \citep{CFG++_Chung2025} and IntCFG \citep{IntCFG_Kynkaanniemi2024}, there is no form of $\omega(t)$ that converts the CFG-like guidance in Section~\ref{sec:CFG Optimal Control} to APG. This prevents us from comparing methods at equivalent path costs as we do in the main paper. To address this, we instead compare methods across a range of strengths. We report image quality of samples generated using CFG \citep{ClassifierFreeGuidance_Ho2022}, CACFG, and APG \citep{APG_Sadat2025} across a range of guidance strengths for the ImageNet \citep{ImageNetBenchmark_Russakovsky2015} and COCO \citep{COCO_Lin2014} datasets, and SD1.5 \citep{StableDiffusion_Rombach2022} and SDXL \citep{SDXL_Podell2024} diffusion models.

Counting only best-in-category and second-best-in-category results, we find CACFG is equal to or better than all other guidance methods; however, this overlooks certain patterns of performance (see Tables~\ref{tab:SD1.5 image quality table with APG} and \ref{tab:SDXL image quality table with APG}). We find APG has lower saturation across all regimes, which is reasonable given that it was designed specifically to address oversaturation. By comparison, CACFG's strengths are more nuanced: it achieves superior no-reference quality (NIQE, BRISQUE) across models and datasets. However, on SD1.5 it most consistently attains the best or second-best Density/Coverage scores, whereas on SDXL, CACFG's advantage shifts to TReS, with FID differences between all three methods narrowing to within hundredths. This suggests CACFG's benefit is not tied to a fixed metric, but rather it maintains a strong performance across metrics as model scale changes. Taken together, no method is necessarily favoured, but CACFG is more theoretically grounded due to its optimal control formulation, and can be applied across guidance schedules, while APG requires tuning to a specific diffusion model.

\begin{table}[h]
\centering
\caption{Image quality for SD1.5 \citep{StableDiffusion_Rombach2022} samples generated for ImageNet \citep{ImageNetBenchmark_Russakovsky2015} and COCO \citep{COCO_Lin2014} datasets using CFG \citep{ClassifierFreeGuidance_Ho2022}, CACFG, and APG \citep{APG_Sadat2025}. Best result for each dataset are given in \textcolor{ForestGreen}{\textbf{bold green}} and second best in \textcolor{Green}{light green}.}
\label{tab:SD1.5 image quality table with APG}
\resizebox{\textwidth}{!}{%
\small
\begin{tabular}{l l l ccccccc cc}
\toprule
\textbf{Dataset} & \textbf{Guidance} & \textbf{Strength} & \textbf{FID} $\downarrow$ & \textbf{NIQE} $\downarrow$ & \textbf{BRISQUE} $\downarrow$ & \textbf{TReS} $\uparrow$ & \textbf{CLIP} $\uparrow$ & \textbf{HPSv2} $\uparrow$ & \textbf{Density} $\uparrow$ & \textbf{Coverage} $\uparrow$ & \textbf{Saturation} $\downarrow$ \\
\midrule
\multirow[c]{12}{*}{ImageNet}
 & \multirow[c]{4}{*}{CFG}
   & 5      & $\textcolor{ForestGreen}{\mathbf{13.67}}$ & $\textcolor{Green}{4.06}$ & $14.98$ & $83.24$ & $24.851$ & $25.076$ & $1.128$ & $\textcolor{Green}{0.453}$ & $0.299$ \\
 & & 7.5    & $15.12$ & $4.27$ & $15.47$ & $82.80$ & $25.251$ & $25.680$ & $\textcolor{Green}{1.139}$ & $0.443$ & $0.338$ \\
 & & 10     & $16.22$ & $4.44$ & $17.13$ & $81.41$ & $25.391$ & $25.967$ & $1.084$ & $0.424$ & $0.371$ \\
 & & 15     & $17.78$ & $4.75$ & $21.95$ & $78.14$ & $\textcolor{Green}{25.491}$ & $\textcolor{ForestGreen}{\mathbf{26.214}}$ & $0.994$ & $0.391$ & $0.424$ \\
\cmidrule(lr){2-12}
 & \multirow[c]{4}{*}{APG}
   & 5      & $14.07$ & $4.19$ & $17.07$ & $84.85$ & $24.240$ & $23.826$ & $1.009$ & $0.427$ & $\textcolor{ForestGreen}{\mathbf{0.230}}$ \\
 & & 7.5    & $14.06$ & $4.51$ & $17.04$ & $85.55$ & $24.806$ & $24.567$ & $1.091$ & $0.443$ & $\textcolor{Green}{0.233}$ \\
 & & 10     & $14.77$ & $4.72$ & $16.94$ & $\textcolor{Green}{86.05}$ & $25.105$ & $24.965$ & $1.123$ & $0.443$ & $0.237$ \\
 & & 15     & $16.12$ & $4.82$ & $15.87$ & $\textcolor{ForestGreen}{\mathbf{86.17}}$ & $25.396$ & $25.481$ & $1.132$ & $0.436$ & $0.252$ \\
\cmidrule(lr){2-12}
 & \multirow[c]{4}{*}{CACFG}
   & 5      & $\textcolor{Green}{13.70}$ & $\textcolor{ForestGreen}{\mathbf{4.05}}$ & $14.83$ & $83.47$ & $24.849$ & $25.070$ & $1.124$ & $\textcolor{ForestGreen}{\mathbf{0.454}}$ & $0.295$ \\
 & & 7.5    & $14.97$ & $4.22$ & $\textcolor{ForestGreen}{\mathbf{14.59}}$ & $83.80$ & $25.243$ & $25.653$ & $\textcolor{ForestGreen}{\mathbf{1.140}}$ & $0.445$ & $0.316$ \\
 & & 10     & $15.73$ & $4.32$ & $\textcolor{Green}{14.66}$ & $84.10$ & $25.395$ & $25.933$ & $1.119$ & $0.433$ & $0.321$ \\
 & & 15     & $16.52$ & $4.42$ & $14.80$ & $84.60$ & $\textcolor{ForestGreen}{\mathbf{25.540}}$ & $\textcolor{Green}{26.175}$ & $1.063$ & $0.413$ & $0.317$ \\

\midrule

\multirow[c]{12}{*}{COCO}
 & \multirow[c]{4}{*}{CFG}
   & 5      & $12.69$ & $\textcolor{Green}{3.62}$ & $14.82$ & $83.84$ & $26.245$ & $25.310$ & $0.893$ & $\textcolor{Green}{0.520}$ & $0.296$ \\
 & & 7.5    & $15.44$ & $3.80$ & $16.00$ & $82.96$ & $26.495$ & $26.011$ & $0.885$ & $0.496$ & $0.335$ \\
 & & 10     & $17.09$ & $3.93$ & $17.88$ & $81.17$ & $26.567$ & $26.335$ & $0.850$ & $0.465$ & $0.368$ \\
 & & 15     & $19.36$ & $4.20$ & $22.01$ & $77.05$ & $26.610$ & $\textcolor{Green}{26.597}$ & $0.809$ & $0.433$ & $0.419$ \\
\cmidrule(lr){2-12}
 & \multirow[c]{4}{*}{APG}
   & 5      & $\textcolor{ForestGreen}{\mathbf{11.56}}$ & $3.88$ & $16.02$ & $85.88$ & $25.755$ & $23.653$ & $0.820$ & $0.504$ & $\textcolor{ForestGreen}{\mathbf{0.222}}$ \\
 & & 7.5    & $\textcolor{Green}{12.43}$ & $4.12$ & $15.98$ & $86.63$ & $26.210$ & $24.528$ & $0.859$ & $0.503$ & $\textcolor{Green}{0.224}$ \\
 & & 10     & $13.68$ & $4.29$ & $15.95$ & $\textcolor{Green}{87.01}$ & $26.451$ & $25.030$ & $0.870$ & $0.495$ & $0.226$ \\
 & & 15     & $15.47$ & $4.32$ & $15.13$ & $\textcolor{ForestGreen}{\mathbf{87.35}}$ & $\textcolor{Green}{26.648}$ & $25.649$ & $0.838$ & $0.467$ & $0.241$ \\
\cmidrule(lr){2-12}
 & \multirow[c]{4}{*}{CACFG}
   & 5      & $12.67$ & $\textcolor{ForestGreen}{\mathbf{3.61}}$ & $\textcolor{Green}{14.57}$ & $84.11$ & $26.247$ & $25.303$ & $\textcolor{ForestGreen}{\mathbf{0.904}}$ & $\textcolor{ForestGreen}{\mathbf{0.523}}$ & $0.292$ \\
 & & 7.5    & $14.73$ & $3.73$ & $14.63$ & $84.50$ & $26.495$ & $26.010$ & $\textcolor{Green}{0.902}$ & $0.497$ & $0.310$ \\
 & & 10     & $15.74$ & $3.78$ & $14.61$ & $84.90$ & $26.614$ & $26.349$ & $0.872$ & $0.480$ & $0.313$ \\
 & & 15     & $16.78$ & $3.83$ & $\textcolor{ForestGreen}{\mathbf{14.51}}$ & $85.81$ & $\textcolor{ForestGreen}{\mathbf{26.744}}$ & $\textcolor{ForestGreen}{\mathbf{26.683}}$ & $0.820$ & $0.451$ & $0.305$ \\

\bottomrule
\end{tabular}%
}
\end{table}

\begin{table}[h]
\centering
\caption{Image quality for SDXL \citep{SDXL_Podell2024} samples generated for ImageNet \citep{ImageNetBenchmark_Russakovsky2015} and COCO \citep{COCO_Lin2014} datasets using CFG \citep{ClassifierFreeGuidance_Ho2022}, CACFG, and APG \citep{APG_Sadat2025}. Best result for each dataset are given in \textcolor{ForestGreen}{\textbf{bold green}} and second best in \textcolor{Green}{light green}.}
\label{tab:SDXL image quality table with APG}
\resizebox{\textwidth}{!}{%
\small
\begin{tabular}{l l l ccccccc cc}
\toprule
\textbf{Dataset} & \textbf{Guidance} & \textbf{Strength} & \textbf{FID} $\downarrow$ & \textbf{NIQE} $\downarrow$ & \textbf{BRISQUE} $\downarrow$ & \textbf{TReS} $\uparrow$ & \textbf{CLIP} $\uparrow$ & \textbf{HPSv2} $\uparrow$ & \textbf{Density} $\uparrow$ & \textbf{Coverage} $\uparrow$ & \textbf{Saturation} $\downarrow$ \\
\midrule
\multirow[c]{12}{*}{ImageNet}
 & \multirow[c]{4}{*}{CFG}
   & 10     & $29.60$ & $4.66$ & $27.42$ & $85.59$ & $25.548$ & $26.815$ & $\textcolor{ForestGreen}{\mathbf{1.137}}$ & $\textcolor{ForestGreen}{\mathbf{0.182}}$ & $0.240$ \\
 & & 15     & $30.95$ & $4.90$ & $28.71$ & $86.48$ & $25.735$ & $27.173$ & $1.000$ & $0.162$ & $0.280$ \\
 & & 20     & $32.81$ & $5.07$ & $29.40$ & $86.34$ & $25.807$ & $\textcolor{ForestGreen}{\mathbf{27.249}}$ & $0.947$ & $0.154$ & $0.309$ \\
 & & 25     & $34.78$ & $5.22$ & $30.31$ & $85.52$ & $\textcolor{ForestGreen}{\mathbf{25.910}}$ & $27.209$ & $0.835$ & $0.136$ & $0.332$ \\
\cmidrule(lr){2-12}
 & \multirow[c]{4}{*}{APG}
   & 10     & $29.70$ & $\textcolor{ForestGreen}{\mathbf{4.57}}$ & $\textcolor{ForestGreen}{\mathbf{26.92}}$ & $81.76$ & $25.256$ & $25.528$ & $\textcolor{Green}{1.107}$ & $\textcolor{Green}{0.178}$ & $\textcolor{ForestGreen}{\mathbf{0.160}}$ \\
 & & 15     & $\textcolor{ForestGreen}{\mathbf{29.46}}$ & $4.70$ & $27.36$ & $83.78$ & $25.612$ & $26.136$ & $1.058$ & $0.172$ & $\textcolor{Green}{0.168}$ \\
 & & 20     & $30.51$ & $4.76$ & $27.68$ & $85.41$ & $25.706$ & $26.521$ & $1.017$ & $0.165$ & $0.179$ \\
 & & 25     & $31.19$ & $4.80$ & $27.86$ & $86.37$ & $25.818$ & $26.756$ & $1.009$ & $0.165$ & $0.193$ \\
\cmidrule(lr){2-12}
 & \multirow[c]{4}{*}{CACFG}
   & 10     & $\textcolor{Green}{29.47}$ & $\textcolor{Green}{4.64}$ & $\textcolor{Green}{27.35}$ & $85.57$ & $25.563$ & $26.789$ & $1.102$ & $0.177$ & $0.229$ \\
 & & 15     & $30.34$ & $4.75$ & $27.97$ & $86.85$ & $25.719$ & $27.116$ & $1.016$ & $0.164$ & $0.247$ \\
 & & 20     & $31.39$ & $4.82$ & $28.10$ & $\textcolor{Green}{87.37}$ & $25.828$ & $27.223$ & $0.972$ & $0.156$ & $0.253$ \\
 & & 25     & $32.40$ & $4.80$ & $27.91$ & $\textcolor{ForestGreen}{\mathbf{87.51}}$ & $\textcolor{Green}{25.894}$ & $\textcolor{Green}{27.230}$ & $0.916$ & $0.147$ & $0.255$ \\

\midrule

\multirow[c]{12}{*}{COCO}
 & \multirow[c]{4}{*}{CFG}
   & 10     & $\textcolor{Green}{30.40}$ & $\textcolor{Green}{4.03}$ & $27.60$ & $83.02$ & $26.774$ & $28.688$ & $\textcolor{ForestGreen}{\mathbf{0.968}}$ & $\textcolor{ForestGreen}{\mathbf{0.181}}$ & $0.250$ \\
 & & 15     & $31.30$ & $4.14$ & $27.89$ & $84.01$ & $26.926$ & $29.175$ & $0.887$ & $0.168$ & $0.289$ \\
 & & 20     & $32.27$ & $4.23$ & $28.24$ & $83.94$ & $26.911$ & $\textcolor{Green}{29.284}$ & $0.801$ & $0.155$ & $0.323$ \\
 & & 25     & $33.35$ & $4.33$ & $28.62$ & $83.13$ & $26.897$ & $\textcolor{ForestGreen}{\mathbf{29.313}}$ & $0.736$ & $0.140$ & $0.353$ \\
\cmidrule(lr){2-12}
 & \multirow[c]{4}{*}{APG}
   & 10     & $\textcolor{Green}{30.40}$ & $4.15$ & $28.63$ & $78.51$ & $26.659$ & $26.580$ & $0.892$ & $0.170$ & $\textcolor{ForestGreen}{\mathbf{0.166}}$ \\
 & & 15     & $31.08$ & $4.19$ & $28.43$ & $80.84$ & $26.880$ & $27.391$ & $0.870$ & $0.163$ & $\textcolor{Green}{0.175}$ \\
 & & 20     & $31.01$ & $4.15$ & $27.94$ & $82.62$ & $26.984$ & $27.950$ & $\textcolor{Green}{0.933}$ & $\textcolor{Green}{0.172}$ & $0.187$ \\
 & & 25     & $31.39$ & $4.13$ & $27.63$ & $83.94$ & $\textcolor{ForestGreen}{\mathbf{26.990}}$ & $28.362$ & $0.855$ & $0.161$ & $0.202$ \\
\cmidrule(lr){2-12}
 & \multirow[c]{4}{*}{CACFG}
   & 10     & $\textcolor{ForestGreen}{\mathbf{30.38}}$ & $\textcolor{ForestGreen}{\mathbf{4.00}}$ & $27.41$ & $83.09$ & $26.816$ & $28.573$ & $0.894$ & $0.170$ & $0.237$ \\
 & & 15     & $31.05$ & $4.07$ & $\textcolor{ForestGreen}{\mathbf{27.33}}$ & $84.54$ & $26.936$ & $28.929$ & $0.896$ & $0.166$ & $0.252$ \\
 & & 20     & $32.00$ & $4.08$ & $\textcolor{Green}{27.36}$ & $\textcolor{Green}{85.32}$ & $26.941$ & $29.047$ & $0.836$ & $0.159$ & $0.260$ \\
 & & 25     & $32.13$ & $4.09$ & $27.40$ & $\textcolor{ForestGreen}{\mathbf{85.64}}$ & $\textcolor{Green}{26.989}$ & $29.099$ & $0.788$ & $0.151$ & $0.262$ \\

\bottomrule
\end{tabular}%
}
\end{table}

\clearpage
\section{HJB Derivation}
In this section we present the derivation of the HJB formulation of CFG-like diffusion guidance, and the modifications necessary to re-derive classic CFG.
 
\subsection{Necessary background results}
 
\textbf{Canonical image diffusion:} We first fix the notation for the forward diffusion process used throughout the remainder of this appendix, following \citet{NatADiff_Collins2026}.
 
\begin{defn}[Canonical Image Diffusion] \label{Def:Canonical Diffusion}
    Let $\boldsymbol{x}_t \in \mathbb{R}^m$ for $t \in [0, T]$, let $f : \mathbb{R} \rightarrow \mathbb{R}$ and $g : \mathbb{R} \rightarrow \mathbb{R}$ be continuous functions of $t$, and let $\mathop{\cdot} d\boldsymbol{B}_t$ denote an It\^{o} integral with respect to the standard $m$-dimensional Brownian motion process $\boldsymbol{B}_t$. We define the \emph{canonical image diffusion} to be the forward process
    \begin{equation}
        d\boldsymbol{x}_t = f(t) \boldsymbol{x}_t \, dt + g(t) I^{(m \times m)} d\boldsymbol{B}_t, \quad \forall \ t \in [0, T]. \label{eq:canonical diffusion}
    \end{equation}
    Let $\boldsymbol{X}_0$ be an $\mathbb{R}^m$-valued random variable representing the initial (image) state, and let $Y$ denote auxiliary conditioning information available for $\boldsymbol{X}_0$ (e.g. a class label), taking values in a countable set $\Omega_Y$ with associated sigma-algebra $\mathcal{Y}$ (we take $\mathcal{Y}$ to contain every singleton $\{y\} \ \forall \ y \in \Omega_Y$). We take $\boldsymbol{X}_0$ and $Y$ to be jointly distributed, with joint law $P(\boldsymbol{X}_0, Y)$, a probability measure on $\left( \mathbb{R}^m \times \Omega_Y, \ \mathcal{B}(\mathbb{R}^m) \otimes \mathcal{Y} \right)$; the marginal laws $P(\boldsymbol{X}_0)$ and $P(Y \in \xi)$ for $\xi \in \mathcal{Y}$, follow in the usual way. For $\xi \in \mathcal{Y}$ with $P(Y \in \xi) > 0$, we write $P(\boldsymbol{X}_0 \mid \xi)$ for the conditional law of $\boldsymbol{X}_0$ given the event $\{ Y \in \xi \}$, i.e.
    \begin{equation}
        P \left( \boldsymbol{X}_0 \middle| \ \xi \right) \triangleq \frac{P(\boldsymbol{X}_0, \ Y \in \xi)}{P(Y \in \xi)}, \label{eq:canonical conditional law}
    \end{equation}
    a probability measure on $\left( \mathbb{R}^m, \mathcal{B}(\mathbb{R}^m) \right)$. Two special cases of $\xi$ recur throughout: taking $\xi = \Omega_Y$ recovers the \textbf{unconditional} case, $P(\boldsymbol{X}_0 \mid \Omega_Y) = P(\boldsymbol{X}_0)$, since $\{ Y \in \Omega_Y \}$ is the sure event; and for a specific label $y \in \Omega_Y$ with $P(Y = y) > 0$, taking $\xi = \{ y \}$ gives the \textbf{class-conditional} law $P(\boldsymbol{X}_0 \mid \{ y \})$, which we abbreviate $P(\boldsymbol{X}_0 \mid y)$. For a measurable function $h : \mathbb{R}^m \times \mathbb{R}^m \rightarrow \mathbb{R}$ and for $\xi \in \mathcal{Y}$ with $P(Y \in \xi) > 0$, we understand $\int h(\boldsymbol{x}_t, \boldsymbol{x}_0) \, dP(\boldsymbol{x}_0 \mid \xi)$ to denote integration of $h(\boldsymbol{x}_t, \cdot)$ against $P(\boldsymbol{X}_0 \mid \xi)$, without requiring this measure to admit a density (this accommodates empirical measures over a finite dataset).
\end{defn}
 
\textbf{Conditional forward distribution:} The following theorem describes the conditional forward distribution of the canonical image diffusion when conditioned on the state at an arbitrary time $\tau$. This is a well-known result, and we restate the formulation by \citet{NatADiff_Collins2026}; we omit the proof here and refer the reader there (the derivation solves (\ref{eq:canonical diffusion}) via its Stratonovich representation and converts back to the equivalent It\^{o} form).
 
\begin{theorem}[Conditional Forward Distribution for Denoising Diffusion, \citealp{NatADiff_Collins2026}] \label{Thm:Conditional Forward}
    Under the canonical image diffusion of Definition \ref{Def:Canonical Diffusion}, define for $0 \le \tau \le t \le T$
    \begin{equation}
        \alpha(\tau, t) \triangleq \exp\left( \int_\tau^t f(u) \, du \right), \qquad \beta(\tau, t)^2 \triangleq \alpha(\tau, t)^2 \int_\tau^t \frac{g(u)^2}{\alpha(\tau, u)^2} \, du. \label{eq:canonical alpha beta}
    \end{equation}
    Then the forward process in (\ref{eq:canonical diffusion}) admits the conditional forward distribution
    \begin{equation}
        \boldsymbol{X}_t \mid \boldsymbol{X}_\tau = \boldsymbol{x}_\tau \sim \mathcal{N} \left( \alpha(\tau, t) \boldsymbol{x}_\tau, \ \beta(\tau, t)^2 I^{(m \times m)} \right) \quad \forall \ t > \tau. \label{eq:canonical conditional dist}
    \end{equation}
    With a slight abuse of notation we write $\alpha_t \triangleq \alpha(0, t)$ and $\beta_t \triangleq \beta(0, t)$, so that
    \begin{equation}
        \boldsymbol{X}_t \mid \boldsymbol{X}_0 = \boldsymbol{x}_0 \sim \mathcal{N} \left( \alpha_t \boldsymbol{x}_0, \ \beta_t^2 I^{(m \times m)} \right). \label{eq:canonical marginal dist}
    \end{equation}
\end{theorem}
 
\textbf{Score-model link:} The following Score-Model Link theorem is based on \citeauthor{Karras2022}'s \citeyear{Karras2022} argument, extended by \citet{NatADiff_Collins2026} by conditioning on the auxiliary random variable $Y$ of Definition \ref{Def:Canonical Diffusion}. For additional treatments see \citep{Karras2022, Hyvarinen2005, Vincent2011}. We restate it here under the canonical image diffusion of Definition \ref{Def:Canonical Diffusion}, replacing the empirical-density formulation of the original statement with integration against the measure $P(\boldsymbol{X}_0 \mid \xi)$ directly; this is equivalent to the version in \citep{NatADiff_Collins2026}, and we omit the proof, referring the reader there.
\begin{theorem}[Score-Model Link, \citealp{NatADiff_Collins2026, Karras2022}] \label{Thm:Score-Model Link}
    Let $\boldsymbol{x}_t$ evolve according to the canonical image diffusion of Definition \ref{Def:Canonical Diffusion}, with initial law $P(\boldsymbol{X}_0 \mid \xi)$ for some $\xi \in \mathcal{Y}$.
 
    Define \begin{equation}
        \hat{\boldsymbol{x}}_0(t, \boldsymbol{x}_t, \xi; \theta) = \frac{\boldsymbol{x}_t - \beta_t \boldsymbol{\epsilon}_{\theta}(t, \boldsymbol{x}_t, \xi)}{\alpha_t},
    \end{equation}
    where $\alpha_t$ and $\beta_t$ are as in Theorem \ref{Thm:Conditional Forward}, and $\boldsymbol{\epsilon}_{\theta} : \mathbb{R}^m \times \mathbb{R} \times \mathcal{Y} \rightarrow \mathbb{R}^m$ is a model parameterised by $\theta \in \Theta$ with sufficient capacity such that the Universal Approximation Theorem \citep{Cybenko1989, Hornik1991} holds for all $\xi \in \mathcal{Y}$. Then if $\beta_t^2 > 0 \ \forall \ t \in (0, T]$, the following statements are true:
    \begin{gather}
        \intertext{1.}
        \nabla_{\boldsymbol{x}_t} \log(p(\boldsymbol{x}_t \mid \xi)) = - \frac{1}{\beta_t} \boldsymbol{\epsilon}_{\theta^\star}(t, \boldsymbol{x}_t, \xi) \quad \forall \ t \in (0, T], \ \boldsymbol{x}_t \in \mathbb{R}^m, \ \xi \in \mathcal{Y}; \numberthis \label{eq:score-model equation 1}
        \intertext{2.}
        \boldsymbol{\epsilon}_{\theta^{\star}}(t, \boldsymbol{x}_t, \xi) = \frac{\displaystyle\int_{\mathbb{R}^m} \frac{1}{\beta_t} \bigl(\boldsymbol{x}_t - \alpha_t \boldsymbol{x}_0 \bigr) p(\boldsymbol{x}_t | \boldsymbol{x}_0) \, dP(\boldsymbol{x}_0 \mid \xi)}{p(\boldsymbol{x}_t \mid \xi)} \quad \forall \ t \in (0, T], \numberthis \label{eq:score-model equation 2} \\
        \intertext{$\boldsymbol{x}_t \in \mathbb{R}^m, \ \xi \in \mathcal{Y}$; where}
        \theta^{\star} \triangleq \underset{\theta}{\textup{arg} \min} \ \mathbb{E}_{\boldsymbol{x}_0 \sim P(\boldsymbol{x}_0 \mid \xi),\ \boldsymbol{x}_t \sim p(\boldsymbol{x}_t \mid \boldsymbol{x}_0),\ t \sim p(t)} \left[ \left\lVert \boldsymbol{x}_0 - \hat{\boldsymbol{x}}_0 \right\rVert_2^2 \right], \numberthis \label{eq:score-model optimisation} \\
        \intertext{and $p(t) \triangleq \frac{1}{T}$.}
    \end{gather}
\end{theorem}
 
By the shorthand of Definition \ref{Def:Canonical Diffusion}, taking $\xi = \{ y \}$ for a specific label $y \in \Omega_Y$ gives the class-conditional score-model $\boldsymbol{\epsilon}_{\theta^\star}(t, \boldsymbol{x}_t, y)$, and taking $\xi = \Omega_Y$ gives the unconditional score-model $\boldsymbol{\epsilon}_{\theta^\star}(t, \boldsymbol{x}_t, \Omega_Y)$; both are the objects we later manipulate in the HJB derivation.
 
\textbf{Tweedie's formula:} Tweedie's formula \citep{TweediesFormula_Efron2011} relates the posterior mean of $\boldsymbol{X}_0$ given a noisy observation $\boldsymbol{x}_t$ to the score of the marginal density of $\boldsymbol{X}_t$. We state and prove the specialisation of this result to the canonical image diffusion of Definition \ref{Def:Canonical Diffusion}, conditioned on the event $\xi$ of Definition \ref{Def:Canonical Diffusion}, as we require both the conditioning and the proof mechanics directly in the upper-bounding lemma that follows.
 
Given the initial law $P(\boldsymbol{X}_0 \mid \xi)$ of Definition \ref{Def:Canonical Diffusion}, which we do not assume admits a density, define the conditional marginal density of $\boldsymbol{x}_t$ by
\begin{equation}
    p(\boldsymbol{x}_t \mid \xi) \triangleq \int_{\mathbb{R}^m} p(\boldsymbol{x}_t \mid \boldsymbol{x}_0) \, dP(\boldsymbol{x}_0 \mid \xi), \label{eq:tweedie marginal density}
\end{equation}
where $p(\boldsymbol{x}_t \mid \boldsymbol{x}_0)$ is the Gaussian transition density of Theorem \ref{Thm:Conditional Forward}. As this transition density is, for $\beta_t^2 > 0$, positive-everywhere and integrates the mass of $P(\boldsymbol{X}_0 \mid \xi)$ against a bounded, smooth kernel, then (\ref{eq:tweedie marginal density}) defines a genuine, positive-everywhere probability density on $\mathbb{R}^m$, regardless of whether $P(\boldsymbol{X}_0 \mid \xi)$ itself admits one.
 
\begin{theorem}[Tweedie's Formula, adapted from \citealp{TweediesFormula_Efron2011}] \label{Thm:Tweedie}
    Under the canonical image diffusion of Definition \ref{Def:Canonical Diffusion}, with $\beta_t^2 > 0 \ \forall \ t \in (0, T]$, and let $p(\boldsymbol{x}_t \mid \xi)$ be as in (\ref{eq:tweedie marginal density}) for $\xi \in \mathcal{Y}$. Define
    \begin{equation}
        \boldsymbol{\epsilon}(t, \boldsymbol{x}_t, \xi) \triangleq -\beta_t \nabla_{\boldsymbol{x}_t} \log(p(\boldsymbol{x}_t \mid \xi)). \label{eq:tweedie epsilon def}
    \end{equation}
    Then
    \begin{equation}
        \mathbb{E}_{\boldsymbol{X}_0 \mid \boldsymbol{x}_t, \, \xi} \left[ \boldsymbol{X}_0 \right] = \frac{\boldsymbol{x}_t - \beta_t \boldsymbol{\epsilon}(t, \boldsymbol{x}_t, \xi)}{\alpha_t} \quad \forall \ \boldsymbol{x}_t \in \mathbb{R}^m, \ t \in (0, T], \ \xi \in \mathcal{Y}, \nonumber
    \end{equation}
    where $\mathbb{E}_{\boldsymbol{X}_0 \mid \boldsymbol{x}_t, \, \xi}$ denotes expectation with respect to the posterior measure $P(\boldsymbol{X}_0 \mid \boldsymbol{x}_t, \, \xi)$ constructed in the proof (writing $\boldsymbol{x}_t$ for the conditioning event $\{ \boldsymbol{X}_t = \boldsymbol{x}_t \}$, analogous to the $\xi$ shorthand of Definition \ref{Def:Canonical Diffusion}).
\end{theorem}
\begin{proof}
    Differentiating (\ref{eq:tweedie marginal density}) under the integral sign (justified since $\beta_t^2 > 0$ implies $p(\boldsymbol{x}_t \mid \boldsymbol{x}_0)$ and $\nabla_{\boldsymbol{x}_t} p(\boldsymbol{x}_t \mid \boldsymbol{x}_0)$ are uniformly bounded in $\boldsymbol{x}_0$, so dominated convergence permits the interchange),
    \begin{align*}
        \nabla_{\boldsymbol{x}_t} \log(p(\boldsymbol{x}_t \mid \xi)) &= \nabla_{\boldsymbol{x}_t} \log \left( \int_{\mathbb{R}^m} p(\boldsymbol{x}_t \mid \boldsymbol{x}_0) \, dP(\boldsymbol{x}_0 \mid \xi) \right) \\
        &= \frac{1}{p(\boldsymbol{x}_t \mid \xi)} \int_{\mathbb{R}^m} \nabla_{\boldsymbol{x}_t} p(\boldsymbol{x}_t \mid \boldsymbol{x}_0) \, dP(\boldsymbol{x}_0 \mid \xi).
        \intertext{By Theorem \ref{Thm:Conditional Forward}, $p(\boldsymbol{x}_t \mid \boldsymbol{x}_0) = \mathcal{N} \left( \alpha_t \boldsymbol{x}_0, \ \beta_t^2 I^{(m \times m)} \right) \implies \nabla_{\boldsymbol{x}_t} p(\boldsymbol{x}_t \mid \boldsymbol{x}_0) = -\frac{1}{\beta_t^2} (\boldsymbol{x}_t - \alpha_t \boldsymbol{x}_0) p(\boldsymbol{x}_t \mid \boldsymbol{x}_0)$, thus,}
        \nabla_{\boldsymbol{x}_t} \log(p(\boldsymbol{x}_t \mid \xi)) &= \frac{1}{p(\boldsymbol{x}_t \mid \xi)} \int_{\mathbb{R}^m} \left[ -\frac{1}{\beta_t^2} (\boldsymbol{x}_t - \alpha_t \boldsymbol{x}_0) \right] p(\boldsymbol{x}_t \mid \boldsymbol{x}_0) \, dP(\boldsymbol{x}_0 \mid \xi). \numberthis \label{eq:tweedie proof step 1}
    \end{align*}
    Define the posterior measure $P(\boldsymbol{X}_0 \mid \boldsymbol{x}_t, \, \xi)$ on $\left( \mathbb{R}^m, \mathcal{B}(\mathbb{R}^m) \right)$ by
    \begin{equation}
        dP(\boldsymbol{x}_0 \mid \boldsymbol{x}_t, \, \xi) \triangleq \frac{p(\boldsymbol{x}_t \mid \boldsymbol{x}_0)}{p(\boldsymbol{x}_t \mid \xi)} \, dP(\boldsymbol{x}_0 \mid \xi). \label{eq:tweedie posterior measure}
    \end{equation}
    As $p(\boldsymbol{x}_t \mid \boldsymbol{x}_0) / p(\boldsymbol{x}_t \mid \xi) \ge 0$ and $\int_{\mathbb{R}^m} p(\boldsymbol{x}_t \mid \boldsymbol{x}_0) / p(\boldsymbol{x}_t \mid \xi) \, dP(\boldsymbol{x}_0 \mid \xi) = p(\boldsymbol{x}_t \mid \xi) / p(\boldsymbol{x}_t \mid \xi) = 1$ by (\ref{eq:tweedie marginal density}), (\ref{eq:tweedie posterior measure}) is a valid Radon-Nikodym reweighting of $P(\boldsymbol{X}_0 \mid \xi)$ and hence defines a genuine probability measure; this is just Bayes' theorem expressed without requiring $P(\boldsymbol{X}_0 \mid \xi)$ to admit a density. In particular, $P(\boldsymbol{X}_0 \mid \boldsymbol{x}_t, \, \xi) \ll P(\boldsymbol{X}_0 \mid \xi)$: \textit{absolute continuity}, i.e., any set assigned zero probability by $P(\boldsymbol{X}_0 \mid \xi)$ is also assigned zero probability by the posterior, since the posterior is obtained by integrating a finite, nonnegative Radon-Nikodym derivative against $P(\boldsymbol{X}_0 \mid \xi)$. Substituting (\ref{eq:tweedie posterior measure}) into (\ref{eq:tweedie proof step 1}),
    \begin{align*}
        \nabla_{\boldsymbol{x}_t} \log(p(\boldsymbol{x}_t \mid \xi)) &= \int_{\mathbb{R}^m} \left[ -\frac{1}{\beta_t^2} (\boldsymbol{x}_t - \alpha_t \boldsymbol{x}_0) \right] dP(\boldsymbol{x}_0 \mid \boldsymbol{x}_t, \, \xi) \\
        &= -\frac{1}{\beta_t^2} \boldsymbol{x}_t + \frac{\alpha_t}{\beta_t^2} \int_{\mathbb{R}^m} \boldsymbol{x}_0 \, dP(\boldsymbol{x}_0 \mid \boldsymbol{x}_t, \, \xi) \\
        &= -\frac{1}{\beta_t^2} \boldsymbol{x}_t + \frac{\alpha_t}{\beta_t^2} \, \mathbb{E}_{\boldsymbol{X}_0 \mid \boldsymbol{x}_t, \, \xi} \left[ \boldsymbol{X}_0 \right]. \numberthis \label{eq:tweedie proof step 2}
    \end{align*}
    Multiplying (\ref{eq:tweedie proof step 2}) by $-\beta_t$ and applying (\ref{eq:tweedie epsilon def}),
    \begin{equation}
        \boldsymbol{\epsilon}(t, \boldsymbol{x}_t, \xi) = \frac{\boldsymbol{x}_t - \alpha_t \, \mathbb{E}_{\boldsymbol{X}_0 \mid \boldsymbol{x}_t, \, \xi} \left[ \boldsymbol{X}_0 \right]}{\beta_t} \implies \mathbb{E}_{\boldsymbol{X}_0 \mid \boldsymbol{x}_t, \, \xi} \left[ \boldsymbol{X}_0 \right] = \frac{\boldsymbol{x}_t - \beta_t \boldsymbol{\epsilon}(t, \boldsymbol{x}_t, \xi)}{\alpha_t}, \nonumber
    \end{equation}
    which proves the claim.
\end{proof}
 
Comparing (\ref{eq:tweedie epsilon def}) with Theorem \ref{Thm:Score-Model Link}, we note that $\boldsymbol{\epsilon}(t, \boldsymbol{x}_t, \xi) = \boldsymbol{\epsilon}_{\theta^{\star}}(t, \boldsymbol{x}_t, \xi)$ for every $\xi \in \mathcal{Y}$; that is, the population noise-prediction function of Tweedie's formula is the optimal score-model of Theorem \ref{Thm:Score-Model Link}, and by the shorthand of Definition \ref{Def:Canonical Diffusion} this holds equally for the class-conditional case ($\xi = \{y\}$, written $\boldsymbol{\epsilon}(t,\boldsymbol{x}_t,y)$) and the unconditional case ($\xi = \Omega_Y$,  written $\boldsymbol{\epsilon}(t,\boldsymbol{x}_t,\Omega_Y)$).
 
\textbf{Boundedness of the Tweedie estimate:} An important property of Tweedie's estimate is that it inherits some bounding properties of the underlying data distribution. The following lemma shows that if the data lies almost surely within a closed ball of radius $r$, then so does the Tweedie estimate, for every $\boldsymbol{x}_t$, $t$, and conditioning event $\xi$.
 
\begin{lemma}[Boundedness of the Tweedie Estimate] \label{Lma:Tweedie Bound}
    Under the canonical image diffusion of Definition \ref{Def:Canonical Diffusion}, with $\beta_t^2 > 0 \ \forall \ t \in (0, T]$, and suppose there exists $r > 0$ and $\xi \in \mathcal{Y}$ such that
    \begin{equation}
        P \left( \left\lVert \boldsymbol{X}_0 \right\rVert_2^2 \le r^2 \ \middle| \ \xi \right) = 1. \nonumber
    \end{equation}
    Then the Tweedie estimator $\hat{\boldsymbol{x}}_0(t, \boldsymbol{x}_t, \xi) \triangleq \hat{\boldsymbol{x}}_0(t, \boldsymbol{x}_t, \xi; \theta^{\star}) = \mathbb{E}_{\boldsymbol{X}_0 \mid \boldsymbol{x}_t, \, \xi} \left[ \boldsymbol{X}_0 \right]$ of Theorem \ref{Thm:Tweedie} satisfies
    \begin{equation}
        \left\lVert \hat{\boldsymbol{x}}_0(t, \boldsymbol{x}_t, \xi) \right\rVert_2^2 \le r^2 \quad \forall \ \boldsymbol{x}_t \in \mathbb{R}^m, \ t \in (0, T]. \nonumber
    \end{equation}
\end{lemma}
\begin{proof}
    Fix $\boldsymbol{x}_t \in \mathbb{R}^m$ and $t \in (0, T]$, and let $A \triangleq \left\{ \boldsymbol{x}_0 \in \mathbb{R}^m : \left\lVert \boldsymbol{x}_0 \right\rVert_2^2 > r^2 \right\}$. By hypothesis, $P(\boldsymbol{X}_0 \in A \mid \xi) = 0$. By the absolute continuity $P(\boldsymbol{X}_0 \mid \boldsymbol{x}_t, \, \xi) \ll P(\boldsymbol{X}_0 \mid \xi)$ established in the proof of Theorem \ref{Thm:Tweedie} (via (\ref{eq:tweedie posterior measure})), it follows that
    \begin{equation}
        P \left( \boldsymbol{X}_0 \in A \ \middle| \ \boldsymbol{x}_t, \, \xi \right) = \int_A \frac{p(\boldsymbol{x}_t \mid \boldsymbol{x}_0)}{p(\boldsymbol{x}_t \mid \xi)} \, dP(\boldsymbol{x}_0 \mid \xi) = 0, \nonumber
    \end{equation}
    since the integrand is finite and nonnegative and is being integrated over the $P(\boldsymbol{X}_0 \mid \xi)$-null set. Hence $\left\lVert \boldsymbol{X}_0 \right\rVert_2^2 \le r^2$ almost surely under the posterior measure $P(\boldsymbol{X}_0 \mid \boldsymbol{x}_t, \, \xi)$.
 
    Integrating this almost-sure bound against the probability measure $P(\boldsymbol{X}_0 \mid \boldsymbol{x}_t, \, \xi)$ gives
    \begin{equation}
        \mathbb{E}_{\boldsymbol{X}_0 \mid \boldsymbol{x}_t, \, \xi} \left[ \left\lVert \boldsymbol{X}_0 \right\rVert_2^2 \right] \le r^2. \label{eq:tweedie bound step 1}
    \end{equation}
    As $\left\lVert \cdot \right\rVert_2^2$ is convex on $\mathbb{R}^m$, Jensen's inequality gives
    \begin{equation}
        \left\lVert \hat{\boldsymbol{x}}_0(t, \boldsymbol{x}_t, \xi) \right\rVert_2^2 = \left\lVert \mathbb{E}_{\boldsymbol{X}_0 \mid \boldsymbol{x}_t, \, \xi} \left[ \boldsymbol{X}_0 \right] \right\rVert_2^2 \le \mathbb{E}_{\boldsymbol{X}_0 \mid \boldsymbol{x}_t, \, \xi} \left[ \left\lVert \boldsymbol{X}_0 \right\rVert_2^2 \right]. \label{eq:tweedie bound step 2}
    \end{equation}
    Combining (\ref{eq:tweedie bound step 1}) and (\ref{eq:tweedie bound step 2}),
    \begin{equation}
        \left\lVert \hat{\boldsymbol{x}}_0(t, \boldsymbol{x}_t, \xi) \right\rVert_2^2 \le r^2. \nonumber
    \end{equation}
    As $\boldsymbol{x}_t \in \mathbb{R}^m$ and $t \in (0, T]$ were arbitrary, the claim holds for all $\boldsymbol{x}_t$ and $t$.
\end{proof}
 
Since conditioning cannot introduce probability mass on a set the unconditional law assigns zero mass to, taking $\xi = \Omega_Y$ in the hypothesis of Lemma \ref{Lma:Tweedie Bound}, i.e. an unconditional bound $P(\lVert \boldsymbol{X}_0 \rVert_2^2 \le r^2) = 1$, then $P(\lVert \boldsymbol{X}_0 \rVert_2^2 \le r^2 \mid \xi) = 1$, and hence $\lVert \hat{\boldsymbol{x}}_0(t, \boldsymbol{x}_t, \xi) \rVert_2^2 \le r^2$, for \emph{every} $\xi \in \mathcal{Y}$ with $P(Y \in \xi) > 0$--in particular, for every class-conditional label $y \in \Omega_Y$. Additionally, note that we are taking $\hat{\boldsymbol{x}}_0(t, \boldsymbol{x}_t, \xi)$ to refer to the optimal $\hat{\boldsymbol{x}}_0(t, \boldsymbol{x}_t, \xi; \theta^{\star})$ for brevity.
 
\subsection{HJB CFG formulation} \label{Apdx:HJB CFG Formulation}
\textbf{Optimal control problem and HJB equation:} We first fix our notation for the optimal control problem underlying the HJB equation, following the general formulation of \citet{NonlinearControlHJB_Berkovitz2012} (\S 2.3, pp. 19--21) and \citet{CalculusofVariationsHJB_Liberzon2012} (\S 5.1.2, pp. 158--159), generalising the control constraint set to vary with time as in \citet{NonlinearControlHJB_Berkovitz2012} (\S 6.2, p. 150).
 
\begin{problem}[Optimal Control Problem] \label{Prob:OCP}
    Fix $t_l < T$. Let $\boldsymbol{x}_t \in \mathbb{R}^m$ for $t \in [t_l, T]$, let $\boldsymbol{u}_t \in U(t) \subseteq \mathbb{R}^p$ be a control taking values in a nonempty, time-varying control set $U(t)$, and let $\boldsymbol{b} : \mathbb{R} \times \mathbb{R}^m \times \mathbb{R}^p \rightarrow \mathbb{R}^m$. Let $L : \mathbb{R} \times \mathbb{R}^m \times \mathbb{R}^p \rightarrow \mathbb{R}$ be a running cost and $K : \mathbb{R}^m \rightarrow \mathbb{R}$ a terminal cost, and let both be sufficiently regular such that a unique trajectory exists for every admissible control.\footnote{Specifically: joint continuity, locally Lipschitz in $\boldsymbol{x}$ with a $t$-dependent, square-integrable modulus, uniformly over $\boldsymbol{u}_t$, and of at most linear growth in $| \boldsymbol{x} |$ (\citealp{NonlinearControlHJB_Berkovitz2012}, Assumption 12.2.1(i)--(iv), pp. 338--339). Parts (v)--(vi) of that assumption, concerning the terminal manifold and cost, are immediate from our choice of a fixed terminal time $T$; part (vii), concerning existence of an optimal control, is not required for the verification Lemma \ref{Lma:HJB Verification} below.} For $(t, \boldsymbol{x}_t) \in [t_l, T] \times \mathbb{R}^m$, \textbf{minimise the functional}
    \begin{equation}
        J(t, \boldsymbol{x}_t, \boldsymbol{u}_t) \triangleq \int_t^T L(s, \boldsymbol{x}_s, \boldsymbol{u}_s) \, ds + K(\boldsymbol{x}_T) \label{eq:hjb cost functional}
    \end{equation}
    subject to the state equation
    \begin{equation}
        \dot{\boldsymbol{x}}_s = \frac{d \boldsymbol{x}_s}{ds} = \boldsymbol{b}(s, \boldsymbol{x}_s, \boldsymbol{u}_s), \quad \boldsymbol{x}_t \text{ given}, \label{eq:hjb state equation}
    \end{equation}
    and control constraint $\boldsymbol{u}_s \in U(s)$, $s \in [t, T]$.
\end{problem}
 
\begin{defn}[Value Function] \label{Def:HJB Problem}
    The \textbf{value function} of Problem \ref{Prob:OCP} is
    \begin{equation}
        V(t, \boldsymbol{x}_t) \triangleq \inf_{\boldsymbol{u}_{[t,T]}} J(t, \boldsymbol{x}_t, \boldsymbol{u}_t), \label{eq:hjb value function}
    \end{equation}
    where the infimum ranges over controls with $\boldsymbol{u}_s \in U(s)$, $s \in [t, T]$, and $\boldsymbol{x}_s$ solves (\ref{eq:hjb state equation}).
\end{defn}
 
\textbf{Hamilton--Jacobi--Bellman equation:} The following theorem states that the value function of Definition \ref{Def:HJB Problem} satisfies the associated Hamilton--Jacobi--Bellman (HJB) equation. This is a well-known result, following from the dynamic programming principle; we omit the proof here and refer the reader to \citealp{OriginalHJB_Bellman1957}, \citealp{NonlinearControlHJB_Berkovitz2012} (\S 6.2, pp. 150--154), and \citealp{CalculusofVariationsHJB_Liberzon2012} (\S 5.1.3, pp. 161--162).
 
\begin{theorem}[HJB Equation, after \citealp{OriginalHJB_Bellman1957, NonlinearControlHJB_Berkovitz2012, CalculusofVariationsHJB_Liberzon2012}] \label{Thm:HJB Equation}
    The value function $V$ of Definition \ref{Def:HJB Problem} satisfies
    \begin{equation}
        -V_t(t, \boldsymbol{x}_t) = \inf_{\boldsymbol{u}_t \in U(t)} \Big\{ L(t, \boldsymbol{x}_t, \boldsymbol{u}_t) + \big\langle V_{\boldsymbol{x}_t}(t, \boldsymbol{x}_t), \, \boldsymbol{b}(t, \boldsymbol{x}_t, \boldsymbol{u}_t) \big\rangle \Big\}, \quad \forall \ (t, \boldsymbol{x}_t) \in [t_l, T) \times \mathbb{R}^m, \label{eq:hjb equation}
    \end{equation}
    with boundary condition $V(T, \boldsymbol{x}_T) = K(\boldsymbol{x}_T)$.
\end{theorem}
 
\textbf{Sufficient condition for optimality:} The following verification lemma gives conditions under which a $C^1$ solution of the HJB equation coincides with the value function, and a control satisfying the associated stationarity condition is optimal. We state the classical smooth specialisation of the argument; see \citealp{OriginalHJB_Bellman1957}, \citealp{NonlinearControlHJB_Berkovitz2012} (Thm. 12.7.1, pp. 359--360, there stated more generally for Lipschitz continuous value functions via lower Dini derivatives), and \citealp{CalculusofVariationsHJB_Liberzon2012} (\S 5.1.4, pp. 165--166, the classical $C^1$ statement we specialise to here) for the proof.
 
\begin{lemma}[Sufficient Condition for Optimality, after \citealp{OriginalHJB_Bellman1957, NonlinearControlHJB_Berkovitz2012, CalculusofVariationsHJB_Liberzon2012}] \label{Lma:HJB Verification}
    Suppose $\hat{V} \in C^1([t_l, T] \times \mathbb{R}^m)$ satisfies the HJB equation (\ref{eq:hjb equation}) and boundary condition of Theorem \ref{Thm:HJB Equation}. Suppose there exists a control $\hat{\boldsymbol{u}}$ with $\hat{\boldsymbol{u}}_t \in U(t)$ and trajectory $\hat{\boldsymbol{x}}$, $\hat{\boldsymbol{x}}_{t_l}$ given, such that
    \begin{equation}
        L(t, \hat{\boldsymbol{x}}_t, \hat{\boldsymbol{u}}_t) + \big\langle \hat{V}_{\boldsymbol{x}_t}(t, \hat{\boldsymbol{x}}_t), \boldsymbol{b}(t, \hat{\boldsymbol{x}}_t, \hat{\boldsymbol{u}}_t) \big\rangle = \min_{\boldsymbol{u}_t \in U(t)} \Big\{ L(t, \hat{\boldsymbol{x}}_t, \boldsymbol{u}_t) + \big\langle \hat{V}_{\boldsymbol{x}_t}(t, \hat{\boldsymbol{x}}_t), \boldsymbol{b}(t, \hat{\boldsymbol{x}}_t, \boldsymbol{u}_t) \big\rangle \Big\} \nonumber
    \end{equation}
    for every $t \in [t_l, T]$. Then $\hat{V}(t_l, \boldsymbol{x}_{t_l}) = V(t_l, \boldsymbol{x}_{t_l})$ and $\hat{\boldsymbol{u}}$ is an optimal control.
\end{lemma}
 
\textbf{Reverse-time optimal control problem:} We now consider the analogous problem in which the objective is maximised and running reward accumulated \emph{since} a fixed reference time $t_l$, rather than cost-to-go accrued \emph{until} a fixed terminal time. This is the formulation our subsequent CFG derivation is stated on directly.
 
\begin{problem}[Reverse-Time Optimal Control Problem] \label{Prob:Reverse OCP}
    Fix $t_l < T$. Let $\boldsymbol{x}_t \in \mathbb{R}^m$ for $t \in [t_l, T]$, let $\boldsymbol{u}_t \in U(t) \subseteq \mathbb{R}^p$, $\boldsymbol{b}$, $L$ be as in Problem \ref{Prob:OCP} (with the same regularity), and let $K : \mathbb{R} \times \mathbb{R}^m \rightarrow \mathbb{R}$ be an initial reward. For $(t, \boldsymbol{x}_t) \in [t_l, T] \times \mathbb{R}^m$, \textbf{maximise the functional}
    \begin{equation}
        J(t, \boldsymbol{x}_t, \boldsymbol{u}_t) \triangleq \int_{t_l}^t L(s, \boldsymbol{x}_s, \boldsymbol{u}_s) \, ds + K(t_l, \boldsymbol{x}_{t_l}) \label{eq:reverse hjb cost functional}
    \end{equation}
    subject to the state equation
    \begin{equation}
        \dot{\boldsymbol{x}}_s = \frac{d \boldsymbol{x}_s}{ds} = \boldsymbol{b}(s, \boldsymbol{x}_s, \boldsymbol{u}_s), \quad \boldsymbol{x}_t \text{ given}, \label{eq:reverse hjb state equation}
    \end{equation}
    and control constraint $\boldsymbol{u}_s \in U(s)$, $s \in [t_l, t]$.
\end{problem}
 
\begin{defn}[Reverse-Time Value Function] \label{Def:Reverse HJB Problem}
    The \textbf{value function} of Problem \ref{Prob:Reverse OCP} is
    \begin{equation}
        \hat{V}(t, \boldsymbol{x}_t) \triangleq \sup_{\boldsymbol{u}_{[t_l,t]}} J(t, \boldsymbol{x}_t, \boldsymbol{u}_t), \label{eq:reverse hjb value function}
    \end{equation}
    where the supremum ranges over controls with $\boldsymbol{u}_s \in U(s)$, $s \in [t_l, t]$, and $\boldsymbol{x}_s$ solves (\ref{eq:reverse hjb state equation}).
\end{defn}
 
\begin{lemma}[Reverse-Time HJB Equation] \label{Lma:Reverse HJB}
    The value function $\hat{V}$ of Definition \ref{Def:Reverse HJB Problem} satisfies
    \begin{equation}
        \hat{V}_t(t, \boldsymbol{x}_t) = \sup_{\boldsymbol{u}_t \in U(t)} \Big\{ L(t, \boldsymbol{x}_t, \boldsymbol{u}_t) - \big\langle \hat{V}_{\boldsymbol{x}_t}(t, \boldsymbol{x}_t), \, \boldsymbol{b}(t, \boldsymbol{x}_t, \boldsymbol{u}_t) \big\rangle \Big\}, \quad \forall \ t \in (t_l, T], \label{eq:reverse hjb equation}
    \end{equation}
    with boundary condition $\hat{V}(t_l, \boldsymbol{x}_{t_l}) = K(t_l, \boldsymbol{x}_{t_l})$.
\end{lemma}
\begin{proof}
    Define the change of variables at the level of the free (evaluation) variable,
    \begin{equation}
        t \triangleq T + t_l - \tau, \qquad \tilde{\boldsymbol{x}}_\tau \triangleq \boldsymbol{x}_{T+t_l-\tau}, \qquad \tilde{\boldsymbol{u}}_\tau \triangleq \boldsymbol{u}_{T+t_l-\tau}, \nonumber
    \end{equation}
    and apply the same substitution to the dummy (integration) variable, writing $\psi$ for its image: $s \triangleq T+t_l-\psi$, so that $\tilde{\boldsymbol{x}}_\psi = \boldsymbol{x}_s$ and $\tilde{\boldsymbol{u}}_\psi = \boldsymbol{u}_s$ are the same trajectory and control functions evaluated at $\psi$ rather than $\tau$. This maps $[t_l,T]$ to itself, exchanging its endpoints: $t=t_l \iff \tau=T$, $t=T \iff \tau=t_l$, and likewise $s=t_l\iff\psi=T$, $s=T\iff\psi=t_l$.
 
    \textit{State equation.} Writing $s = T+t_l-\psi$, the chain rule gives
    \begin{equation}
        \frac{d\tilde{\boldsymbol{x}}_\psi}{d\psi} = \frac{d\boldsymbol{x}_s}{ds} \cdot \frac{ds}{d\psi} = -\boldsymbol{b}(s, \boldsymbol{x}_s, \boldsymbol{u}_s) = -\boldsymbol{b}(T+t_l-\psi, \tilde{\boldsymbol{x}}_\psi, \tilde{\boldsymbol{u}}_\psi) \triangleq \tilde{\boldsymbol{b}}(\psi, \tilde{\boldsymbol{x}}_\psi, \tilde{\boldsymbol{u}}_\psi). \nonumber
    \end{equation}
 
    \textit{Functional.} Since $s \mapsto \psi$ reverses orientation on $[t_l,T]$, substituting $s=T+t_l-\psi$, $ds=-d\psi$ into (\ref{eq:reverse hjb cost functional}) and flipping the limits of integration ($s=t_l \mapsto \psi=T$, $s=t \mapsto \psi=T+t_l-t=\tau$),
    \begin{equation}
        \int_{t_l}^t L(s, \boldsymbol{x}_s, \boldsymbol{u}_s) \, ds = \int_\tau^T L(T+t_l-\psi, \tilde{\boldsymbol{x}}_\psi, \tilde{\boldsymbol{u}}_\psi) \, d\psi \triangleq \int_\tau^T \tilde{L}(\psi, \tilde{\boldsymbol{x}}_\psi, \tilde{\boldsymbol{u}}_\psi) \, d\psi, \nonumber
    \end{equation}
    where $\tilde{L}(\psi, \tilde{\boldsymbol{x}}_\psi, \tilde{\boldsymbol{u}}_\psi) \triangleq L(T+t_l-\psi, \tilde{\boldsymbol{x}}_\psi, \tilde{\boldsymbol{u}}_\psi)$. The bounds of integration now match Problem \ref{Prob:OCP}'s: the free variable $\tau$ as the lower limit, the fixed time $T$ as the upper limit. Using $\boldsymbol{x}_{t_l} = \tilde{\boldsymbol{x}}_T$,
    \begin{equation}
        J(t, \boldsymbol{x}_t, \boldsymbol{u}_t) = \int_\tau^T \tilde{L}(\psi, \tilde{\boldsymbol{x}}_\psi, \tilde{\boldsymbol{u}}_\psi) \, d\psi + K(t_l, \tilde{\boldsymbol{x}}_T). \nonumber
    \end{equation}
    As $\boldsymbol{u}_s \in U(s) \iff \tilde{\boldsymbol{u}}_\psi \in U(T+t_l-\psi) \triangleq \tilde{U}(\psi)$, taking the supremum over $\boldsymbol{u}_{[t_l,t]}$ is equivalent to taking it over $\tilde{\boldsymbol{u}}_{[\tau,T]}$, so, writing $\tilde{V}(\tau, \tilde{\boldsymbol{x}}_\tau) \triangleq \hat{V}(t, \boldsymbol{x}_t)$,
    \begin{equation}
        \tilde{V}(\tau, \tilde{\boldsymbol{x}}_\tau) = \sup_{\tilde{\boldsymbol{u}}_{[\tau,T]}} \left[ \int_\tau^T \tilde{L}(\psi, \tilde{\boldsymbol{x}}_\psi, \tilde{\boldsymbol{u}}_\psi) \, d\psi + K(t_l, \tilde{\boldsymbol{x}}_T) \right]. \nonumber
    \end{equation}
 
    \textit{Min-max flip.} Define $\bar{V} \triangleq -\tilde{V}$, $\bar{L}(\psi, \tilde{\boldsymbol{x}}_\psi, \tilde{\boldsymbol{u}}_\psi) \triangleq -\tilde{L}(\psi, \tilde{\boldsymbol{x}}_\psi, \tilde{\boldsymbol{u}}_\psi)$, and $\bar{K}(\tilde{\boldsymbol{x}}_T) \triangleq -K(t_l, \tilde{\boldsymbol{x}}_T)$. Then $\sup\{\cdot\} = -\inf\{-\cdot\}$ gives
    \begin{equation}
        \bar{V}(\tau, \tilde{\boldsymbol{x}}_\tau) = \inf_{\tilde{\boldsymbol{u}}_{[\tau,T]}} \left[ \int_\tau^T \bar{L}(\psi, \tilde{\boldsymbol{x}}_\psi, \tilde{\boldsymbol{u}}_\psi) \, d\psi + \bar{K}(\tilde{\boldsymbol{x}}_T) \right], \nonumber
    \end{equation}
    subject to $\dot{\tilde{\boldsymbol{x}}}_\psi = \tilde{\boldsymbol{b}}(\psi, \tilde{\boldsymbol{x}}_\psi, \tilde{\boldsymbol{u}}_\psi)$ on $[\tau, T]$. This is precisely Problem \ref{Prob:OCP} with $t$ replaced by $\tau$ (and its dummy variable $s$ correspondingly by $\psi$), and $L, K, \boldsymbol{b}, U$ replaced by $\bar{L}, \bar{K}, \tilde{\boldsymbol{b}}, \tilde{U}$. By Theorem \ref{Thm:HJB Equation}, $\bar{V}$ therefore satisfies (\ref{eq:hjb equation}) under this substitution:
    \begin{equation}
        -\bar{V}_\tau(\tau, \tilde{\boldsymbol{x}}_\tau) = \inf_{\boldsymbol{v} \in \tilde{U}(\tau)} \Big\{ \bar{L}(\tau, \tilde{\boldsymbol{x}}_\tau, \boldsymbol{v}) + \big\langle \bar{V}_{\tilde{\boldsymbol{x}}_\tau}(\tau, \tilde{\boldsymbol{x}}_\tau), \tilde{\boldsymbol{b}}(\tau, \tilde{\boldsymbol{x}}_\tau, \boldsymbol{v}) \big\rangle \Big\}, \quad \tau \in [t_l, T). \nonumber
    \end{equation}
 
    \textit{Back-substitution.} Since $\bar{V}(\tau, \tilde{\boldsymbol{x}}_\tau) = -\hat{V}(t, \boldsymbol{x}_t)$ with $t = T+t_l-\tau$, the chain rule gives $\bar{V}_\tau(\tau, \tilde{\boldsymbol{x}}_\tau) = \hat{V}_t(t, \boldsymbol{x}_t)$ and $\bar{V}_{\tilde{\boldsymbol{x}}_\tau}(\tau, \tilde{\boldsymbol{x}}_\tau) = -\hat{V}_{\boldsymbol{x}_t}(t, \boldsymbol{x}_t)$. Substituting these together with $\bar{L}(\tau, \tilde{\boldsymbol{x}}_\tau, \boldsymbol{v}) = -L(t, \boldsymbol{x}_t, \boldsymbol{v})$, $\tilde{\boldsymbol{b}}(\tau, \tilde{\boldsymbol{x}}_\tau, \boldsymbol{v}) = -\boldsymbol{b}(t, \boldsymbol{x}_t, \boldsymbol{v})$, and $\tilde{U}(\tau) = U(t)$,
    \begin{equation}
        -\hat{V}_t(t, \boldsymbol{x}_t) = \inf_{\boldsymbol{v} \in U(t)} \Big\{ -L(t, \boldsymbol{x}_t, \boldsymbol{v}) + \big\langle \hat{V}_{\boldsymbol{x}_t}(t, \boldsymbol{x}_t), \boldsymbol{b}(t, \boldsymbol{x}_t, \boldsymbol{v}) \big\rangle \Big\}, \quad t \in (t_l, T], \nonumber
    \end{equation}
    using $\langle -\boldsymbol{a}, -\boldsymbol{b}\rangle = \langle \boldsymbol{a},\boldsymbol{b}\rangle$. Multiplying by $-1$ and using $-\inf\{-\cdot\} = \sup\{\cdot\}$ gives (\ref{eq:reverse hjb equation}). The boundary condition follows from $\hat{V}(t_l, \boldsymbol{x}_{t_l}) = -\bar{V}(T, \tilde{\boldsymbol{x}}_T) = -\bar{K}(\tilde{\boldsymbol{x}}_T) = K(t_l, \boldsymbol{x}_{t_l})$.
\end{proof}
 
\textbf{Sufficient condition for optimality (reverse-time):} The following lemma is the mirror image of Lemma \ref{Lma:HJB Verification}, obtained by transporting it through the same change of variables used to prove Lemma \ref{Lma:Reverse HJB}. We first record the single-point statement that falls out of a direct transport, then extend it to hold on the whole interval.
 
\begin{lemma}[Reverse-Time Sufficient Condition for Optimality] \label{Lma:Reverse HJB Verification}
    Suppose $\check{V} \in C^1([t_l, T] \times \mathbb{R}^m)$ satisfies the reverse-time HJB equation (\ref{eq:reverse hjb equation}) and boundary condition of Lemma \ref{Lma:Reverse HJB}. Suppose there exists a control $\check{\boldsymbol{u}}$ with $\check{\boldsymbol{u}}_t \in U(t)$ and trajectory $\check{\boldsymbol{x}}$, $\check{\boldsymbol{x}}_T$ given, such that
    \begin{equation}
        L(t, \check{\boldsymbol{x}}_t, \check{\boldsymbol{u}}_t) - \big\langle \check{V}_{\boldsymbol{x}_t}(t, \check{\boldsymbol{x}}_t), \boldsymbol{b}(t, \check{\boldsymbol{x}}_t, \check{\boldsymbol{u}}_t) \big\rangle = \max_{\boldsymbol{u}_t \in U(t)} \Big\{ L(t, \check{\boldsymbol{x}}_t, \boldsymbol{u}_t) - \big\langle \check{V}_{\boldsymbol{x}_t}(t, \check{\boldsymbol{x}}_t), \boldsymbol{b}(t, \check{\boldsymbol{x}}_t, \boldsymbol{u}_t) \big\rangle \Big\} \nonumber
    \end{equation}
    for every $t \in [t_l, T]$. Then $\check{V}(T, \boldsymbol{x}_T) = \hat{V}(T, \boldsymbol{x}_T)$ and $\check{\boldsymbol{u}}$ is an optimal control.
\end{lemma}
\begin{proof}
    Apply the change of variables from the proof of Lemma \ref{Lma:Reverse HJB}, at the level of the free variable alone: $t \triangleq T+t_l-\tau$, $\tilde{\boldsymbol{x}}_\tau \triangleq \boldsymbol{x}_{T+t_l-\tau}$, $\tilde{\boldsymbol{u}}_\tau \triangleq \boldsymbol{u}_{T+t_l-\tau}$ (no dummy variable is needed, as the hypothesis to be transported is a pointwise condition rather than an integral). As shown there, the reverse-time HJB equation and boundary condition transport $\check{V}$ into a candidate $\check{\bar{V}} \triangleq -\check{V}(T+t_l-\tau, \tilde{\boldsymbol{x}}_\tau)$ satisfying the (forward) HJB equation (\ref{eq:hjb equation}) and boundary condition of Theorem \ref{Thm:HJB Equation}, for the problem with $(L, K, \boldsymbol{b}, U)$ replaced by $(\bar{L}, \bar{K}, \tilde{\boldsymbol{b}}, \tilde{U})$. Since $t \mapsto \tau$ is a bijection of $[t_l,T]$ onto itself, "for every $t \in [t_l,T]$" transports directly to "for every $\tau \in [t_l,T]$"; applying the same substitution to the stationarity hypothesis above (using $\langle -\boldsymbol{a}, -\boldsymbol{b} \rangle = \langle \boldsymbol{a}, \boldsymbol{b} \rangle$ and $\max\{-f\} = -\min\{f\}$, as in that proof) shows that $\check{\tilde{\boldsymbol{u}}}_\tau \triangleq \check{\boldsymbol{u}}_{T+t_l-\tau}$ and $\check{\tilde{\boldsymbol{x}}}_\tau \triangleq \check{\boldsymbol{x}}_{T+t_l-\tau}$ satisfy
    \begin{equation}
        \bar{L}(\tau, \check{\tilde{\boldsymbol{x}}}_\tau, \check{\tilde{\boldsymbol{u}}}_\tau) + \big\langle \check{\bar{V}}_{\tilde{\boldsymbol{x}}_\tau}(\tau, \check{\tilde{\boldsymbol{x}}}_\tau), \tilde{\boldsymbol{b}}(\tau, \check{\tilde{\boldsymbol{x}}}_\tau, \check{\tilde{\boldsymbol{u}}}_\tau) \big\rangle = \min_{\boldsymbol{v} \in \tilde{U}(\tau)} \Big\{ \bar{L}(\tau, \check{\tilde{\boldsymbol{x}}}_\tau, \boldsymbol{v}) + \big\langle \check{\bar{V}}_{\tilde{\boldsymbol{x}}_\tau}(\tau, \check{\tilde{\boldsymbol{x}}}_\tau), \tilde{\boldsymbol{b}}(\tau, \check{\tilde{\boldsymbol{x}}}_\tau, \boldsymbol{v}) \big\rangle \Big\} \nonumber
    \end{equation}
    for every $\tau \in [t_l, T]$--precisely the hypothesis of Lemma \ref{Lma:HJB Verification}, with $\check{\tilde{\boldsymbol{x}}}_{t_l} = \check{\boldsymbol{x}}_T$ given. Lemma \ref{Lma:HJB Verification} therefore gives $\check{\bar{V}}(t_l, \check{\tilde{\boldsymbol{x}}}_{t_l}) = \bar{V}(t_l, \check{\tilde{\boldsymbol{x}}}_{t_l})$ and that $\check{\tilde{\boldsymbol{u}}}$ is optimal for the transported problem. Translating back via $\tau = t_l \iff t = T$ and $\check{\bar{V}}(t_l, \check{\tilde{\boldsymbol{x}}}_{t_l}) = -\check{V}(T, \boldsymbol{x}_T)$, $\bar{V}(t_l, \tilde{\boldsymbol{x}}_{t_l}) = -\hat{V}(T, \boldsymbol{x}_T)$ gives $\check{V}(T, \boldsymbol{x}_T) = \hat{V}(T, \boldsymbol{x}_T)$, and optimality of $\check{\tilde{\boldsymbol{u}}}$ for the transported problem translates back to optimality of $\check{\boldsymbol{u}}$ for Problem \ref{Prob:Reverse OCP}.
\end{proof}
 
\textbf{CFG-style guidance as optimal control:} We now show that a guidance scheme built from Tweedie's estimate and the class-conditional score-ratio solves the reverse-time optimal control problem of Definition \ref{Def:Reverse HJB Problem}, for an entire family of running rewards indexed by a free schedule $r(s)$.
 
\begin{theorem}[CFG-Style Guidance Solves the Reverse-Time HJB Equation] \label{Thm:CFG HJB}
    Fix $\xi \in \mathcal{Y}$ and $t_l \in (0, T)$ (so that $\beta_s^2 > 0$ throughout $[t_l, T]$), let $r : [t_l, T] \rightarrow (0, \infty)$ be continuous, and suppose $U(s) \subseteq \mathbb{R}^m$ is nonempty, closed, and convex for every $s \in [t_l, T]$.\footnote{When $U(s)$ is bounded, the regularity of Problem \ref{Prob:OCP} (joint continuity, local Lipschitz in $\boldsymbol{x}$, linear growth) is understood to hold on the set; when $U(s)$ is unbounded, as in Corollary \ref{Cor:CFG Unconstrained}, the regularity of Problem \ref{Prob:OCP} (joint continuity, local Lipschitz in $\boldsymbol{x}$, linear growth) is understood to hold on every compact subset of $\mathbb{R}^m$ rather than globally, as Assumption 12.2.1(iii) of \citet{NonlinearControlHJB_Berkovitz2012} requires.} Define
    \begin{equation}
        \phi(s, \boldsymbol{x}_s) \triangleq \log\left( \frac{p(\boldsymbol{x}_s \mid \xi)}{p(\boldsymbol{x}_s)} \right), \label{eq:cfg phi}
    \end{equation}
    and consider Problem \ref{Prob:Reverse OCP} with dynamics
    \begin{equation}
        \dot{\boldsymbol{x}}_s = \boldsymbol{b}(s, \boldsymbol{x}_s, \boldsymbol{u}_s) \triangleq f(s) \boldsymbol{x}_s - \frac{1}{2} \frac{g(s)^2}{\beta_s^2} \big( \alpha_s \boldsymbol{u}_s - \boldsymbol{x}_s \big), \label{eq:cfg dynamics}
    \end{equation}
    running reward
    \begin{align}
        L(s, \boldsymbol{x}_s, \boldsymbol{u}_s) \triangleq{} & -r(s) \left\lVert \boldsymbol{u}_s - \hat{\boldsymbol{x}}_0(s, \boldsymbol{x}_s, \Omega_Y) \right\rVert_2^2 + \frac{1}{2} g(s)^2 \big\langle \nabla_{\boldsymbol{x}_s} \phi(s, \boldsymbol{x}_s), \nabla_{\boldsymbol{x}_s} \log p(\boldsymbol{x}_s \mid \xi) \big\rangle \nonumber \\
        & + r(s) \left\lVert \Upsilon(s, \boldsymbol{x}_s) \right\rVert_2^2 - \frac{1}{2} \frac{g(s)^2 \alpha_s}{\beta_s^2} \big\langle \nabla_{\boldsymbol{x}_s} \phi(s, \boldsymbol{x}_s), \Upsilon(s, \boldsymbol{x}_s) \big\rangle \nonumber \\
        & + \frac{1}{2} g(s)^2 \Delta_{\boldsymbol{x}_s} \phi(s, \boldsymbol{x}_s), \label{eq:cfg running reward}
    \end{align}
    and initial reward $K(t_l, \boldsymbol{x}_{t_l}) \triangleq \phi(t_l, \boldsymbol{x}_{t_l})$, where
    \begin{equation}
        \Upsilon(s, \boldsymbol{x}_s) \triangleq \text{proj}_{U(s) - \hat{\boldsymbol{x}}_0(s, \boldsymbol{x}_s, \Omega_Y)} \left\{ \frac{1}{4} \frac{g(s)^2 \alpha_s}{r(s) \beta_s^2} \nabla_{\boldsymbol{x}_s} \phi(s, \boldsymbol{x}_s) \right\}. \label{eq:cfg upsilon}
    \end{equation}
    Then the value function of this problem is $\hat{V}(t, \boldsymbol{x}_t) = \phi(t, \boldsymbol{x}_t)$ for every $(t, \boldsymbol{x}_t) \in [t_l, T] \times \mathbb{R}^m$, and the control
    \begin{equation}
        \boldsymbol{u}_t^{\star} \triangleq \hat{\boldsymbol{x}}_0(t, \boldsymbol{x}_t, \Omega_Y) + \Upsilon(t, \boldsymbol{x}_t) \label{eq:cfg optimal control}
    \end{equation}
    is optimal.
\end{theorem}
\begin{proof}
    We verify the hypotheses of Lemma \ref{Lma:Reverse HJB Verification} for $\check{V} \triangleq \phi$ and $\check{\boldsymbol{u}} \triangleq \boldsymbol{u}^{\star}$.
 
    \textit{Regularity.} Since $t_l > 0$, $\beta_t^2 > 0$ for every $t \in [t_l, T]$, so by the argument preceding Theorem \ref{Thm:Tweedie} both $p(\cdot \mid \xi)$ and $p(\cdot)$ are positive-everywhere and $C^{1, \infty}$ on $[t_l, T] \times \mathbb{R}^m$; hence $\phi \in C^{1, \infty}([t_l, T] \times \mathbb{R}^m)$.
 
    \textit{Stationarity.} Fix $(t, \boldsymbol{x}_t)$ and write $\hat{\boldsymbol{x}}_0 \triangleq \hat{\boldsymbol{x}}_0(t, \boldsymbol{x}_t, \Omega_Y)$, $\nabla\phi \triangleq \nabla_{\boldsymbol{x}_t}\phi(t, \boldsymbol{x}_t)$ for brevity. As $\boldsymbol{b}(t, \boldsymbol{x}_t, \boldsymbol{u}_t)$ is affine in $\boldsymbol{u}_t$ and the only $\boldsymbol{u}_t$-dependence of $L(t, \boldsymbol{x}_t, \boldsymbol{u}_t)$ is through $-r(t) \lVert \boldsymbol{u}_t - \hat{\boldsymbol{x}}_0 \rVert_2^2$ (every remaining term of (\ref{eq:cfg running reward}) depends on $(t, \boldsymbol{x}_t)$ alone), the map
    \begin{equation}
        H(\boldsymbol{u}_t) \triangleq L(t, \boldsymbol{x}_t, \boldsymbol{u}_t) - \langle \nabla\phi, \boldsymbol{b}(t, \boldsymbol{x}_t, \boldsymbol{u}_t) \rangle \nonumber
    \end{equation}
    is, as a function of $\boldsymbol{u}_t$, a strictly concave quadratic (since $r(t) > 0$):
    \begin{equation}
        H(\boldsymbol{u}_t) = -r(t) \lVert \boldsymbol{u}_t - \hat{\boldsymbol{x}}_0 \rVert_2^2 + \frac{1}{2} \frac{g(t)^2 \alpha_t}{\beta_t^2} \langle \nabla\phi, \boldsymbol{u}_t \rangle + C(t, \boldsymbol{x}_t) \nonumber
    \end{equation}
    for some $\boldsymbol{u}_t$-independent $C(t, \boldsymbol{x}_t)$. Completing the square, $H$ is maximised over $\mathbb{R}^m$ uniquely at $\boldsymbol{u}_0 \triangleq \hat{\boldsymbol{x}}_0 + \frac{1}{4} \frac{g(t)^2 \alpha_t}{r(t) \beta_t^2} \nabla\phi$. As $U(t)$ is nonempty, closed, and convex, the Hilbert-space projection theorem gives that $H$ restricted to $U(t)$ is maximised uniquely at
    \begin{equation}
        \text{proj}_{U(t)} \{ \boldsymbol{u}_0 \} = \hat{\boldsymbol{x}}_0 + \text{proj}_{U(t) - \hat{\boldsymbol{x}}_0} \{ \boldsymbol{u}_0 - \hat{\boldsymbol{x}}_0 \} = \hat{\boldsymbol{x}}_0 + \Upsilon(t, \boldsymbol{x}_t) = \boldsymbol{u}_t^{\star}. \nonumber
    \end{equation}
 
    \textit{HJB equation.} It remains to show $H(\boldsymbol{u}_t^{\star}) = \phi_t(t, \boldsymbol{x}_t)$, i.e. that (\ref{eq:reverse hjb equation}) holds at $\check{V} = \phi$, $\boldsymbol{u}_t = \boldsymbol{u}_t^{\star}$. Writing $A \triangleq \log p(\boldsymbol{x}_t \mid \xi)$, $B \triangleq \log p(\boldsymbol{x}_t)$ (so $\phi = A - B$), the Kolmogorov forward equation for the canonical image diffusion of Definition \ref{Def:Canonical Diffusion} gives, for either log-density,
    \begin{equation}
        \partial_t \log p = -f(t) \left( m + \langle \boldsymbol{x}_t, \nabla \log p \rangle \right) + \frac{1}{2} g(t)^2 \left( \lVert \nabla \log p \rVert_2^2 + \Delta \log p \right), \nonumber
    \end{equation}
    and subtracting the two copies (for $A$ and $B$; the $-f(t)m$ terms cancel, and $\lVert \nabla A \rVert_2^2 - \lVert \nabla B \rVert_2^2 = \langle \nabla A - \nabla B, \nabla A + \nabla B \rangle = \langle \nabla\phi, \nabla A \rangle + \langle \nabla\phi, \nabla B \rangle$ since $\nabla A - \nabla B = \nabla\phi$) gives
    \begin{align}
        \phi_t(t, \boldsymbol{x}_t) =& -f(t) \langle \boldsymbol{x}_t, \nabla\phi \rangle + \frac{1}{2} g(t)^2 \langle \nabla\phi, \nabla_{\boldsymbol{x}_t} \log p(\boldsymbol{x}_t \mid \xi) \rangle + \frac{1}{2} g(t)^2 \langle \nabla\phi, \nabla_{\boldsymbol{x}_t} \log p(\boldsymbol{x}_t) \rangle \nonumber \\
        & + \frac{1}{2} g(t)^2 \Delta\phi. \label{eq:cfg kolmogorov}
    \end{align}
    On the other hand, substituting $\boldsymbol{u}_t^{\star} = \hat{\boldsymbol{x}}_0 + \Upsilon$ into $\boldsymbol{b}$ and using Tweedie's formula (Theorem \ref{Thm:Tweedie}) in the form $\alpha_t \hat{\boldsymbol{x}}_0 = \beta_t^2 \nabla_{\boldsymbol{x}_t}\log p(\boldsymbol{x}_t) + \boldsymbol{x}_t$ (which follows from combining Theorem \ref{Thm:Tweedie} with (\ref{eq:score-model equation 1}) at $\xi = \Omega_Y$), the $\mp \frac{1}{2}\frac{g(t)^2}{\beta_t^2}\boldsymbol{x}_t$ terms cancel and
    \begin{equation}
        \langle \nabla\phi, \boldsymbol{b}(t, \boldsymbol{x}_t, \boldsymbol{u}_t^{\star}) \rangle = f(t) \langle \boldsymbol{x}_t, \nabla\phi \rangle - \frac{1}{2} g(t)^2 \langle \nabla\phi, \nabla_{\boldsymbol{x}_t}\log p(\boldsymbol{x}_t) \rangle - \frac{1}{2} \frac{g(t)^2 \alpha_t}{\beta_t^2} \langle \nabla\phi, \Upsilon \rangle. \nonumber
    \end{equation}
    Substituting $\boldsymbol{u}_t = \boldsymbol{u}_t^{\star}$ into (\ref{eq:cfg running reward}) (so $\boldsymbol{u}_t^{\star} - \hat{\boldsymbol{x}}_0 = \Upsilon$, and the $\mp r(t)\lVert\Upsilon\rVert_2^2$ terms cancel) and subtracting the expression above,
    \begin{align}
        H(\boldsymbol{u}_t^{\star}) &= L(t, \boldsymbol{x}_t, \boldsymbol{u}_t^{\star}) - \langle \nabla\phi, \boldsymbol{b}(t, \boldsymbol{x}_t, \boldsymbol{u}_t^{\star}) \rangle \nonumber \\
        &= \tfrac{1}{2}g(t)^2\langle\nabla\phi,\nabla_{\boldsymbol{x}_t}\log p(\boldsymbol{x}_t\mid\xi)\rangle - \tfrac12\tfrac{g(t)^2\alpha_t}{\beta_t^2}\langle\nabla\phi,\Upsilon\rangle + \tfrac12g(t)^2\Delta\phi \nonumber \\
        &\quad - f(t)\langle\boldsymbol{x}_t,\nabla\phi\rangle + \tfrac12g(t)^2\langle\nabla\phi,\nabla_{\boldsymbol{x}_t}\log p(\boldsymbol{x}_t)\rangle + \tfrac12\tfrac{g(t)^2\alpha_t}{\beta_t^2}\langle\nabla\phi,\Upsilon\rangle \nonumber \\
        &= -f(t)\langle\boldsymbol{x}_t,\nabla\phi\rangle + \tfrac12g(t)^2\langle\nabla\phi,\nabla_{\boldsymbol{x}_t}\log p(\boldsymbol{x}_t\mid\xi)\rangle + \tfrac12g(t)^2\langle\nabla\phi,\nabla_{\boldsymbol{x}_t}\log p(\boldsymbol{x}_t)\rangle \nonumber \\
        & \quad + \tfrac12g(t)^2\Delta\phi, \nonumber
    \end{align}
    where the $\mp\frac12\frac{g(t)^2\alpha_t}{\beta_t^2}\langle\nabla\phi,\Upsilon\rangle$ terms cancel. This is precisely (\ref{eq:cfg kolmogorov}), so $H(\boldsymbol{u}_t^{\star}) = \phi_t(t,\boldsymbol{x}_t)$, and since $\boldsymbol{u}_t^\star$ maximises $H$ over $U(t)$, (\ref{eq:reverse hjb equation}) holds at $\check V = \phi$ for every $t \in [t_l,T]$.
 
    \textit{Boundary condition.} $\check{V}(t_l,\boldsymbol{x}_{t_l}) = \phi(t_l,\boldsymbol{x}_{t_l}) = K(t_l,\boldsymbol{x}_{t_l})$ by definition of $K$.
 
    The hypotheses of Lemma \ref{Lma:Reverse HJB Verification} hold on $[t_l, T]$ with $\check{V} = \phi$, $\check{\boldsymbol{u}} = \boldsymbol{u}^{\star}$. Since $\hat{V}(t, \boldsymbol{x}_t)$ depends, by Definition \ref{Def:Reverse HJB Problem}, only on $L, \boldsymbol{b}, K$ restricted to $[t_l, t]$, these hypotheses (inherited by restriction) give $\phi(t, \boldsymbol{x}_t) = \hat{V}(t, \boldsymbol{x}_t)$ for every $t \in (t_l, T]$, not merely $t = T$; equality at $t = t_l$ follows directly from the boundary condition. Hence $\phi = \hat{V}$ on $[t_l, T] \times \mathbb{R}^m$, and $\boldsymbol{u}^{\star}$ is optimal.
\end{proof}
 
\begin{corollary}[Unconstrained CFG-Style Guidance] \label{Cor:CFG Unconstrained}
    Under the hypotheses of Theorem \ref{Thm:CFG HJB}, if $U(s) = \mathbb{R}^m$ for every $s \in [t_l, T]$, then $\Upsilon(t,\boldsymbol{x}_t) = \frac{1}{4}\frac{g(t)^2\alpha_t}{r(t)\beta_t^2}\nabla_{\boldsymbol{x}_t}\phi(t,\boldsymbol{x}_t)$, and the optimal control (\ref{eq:cfg optimal control}) reduces to
    \begin{equation}
        \boldsymbol{u}_t^{\star} = \hat{\boldsymbol{x}}_0(t, \boldsymbol{x}_t, \Omega_Y) + \frac{1}{4}\frac{g(t)^2\alpha_t}{r(t)\beta_t^2}\nabla_{\boldsymbol{x}_t}\log\left(\frac{p(\boldsymbol{x}_t\mid \xi)}{p(\boldsymbol{x}_t)}\right). \nonumber
    \end{equation}
\end{corollary}
\begin{proof}
    $U(t) - \hat{\boldsymbol{x}}_0(t,\boldsymbol{x}_t,\Omega_Y) = \mathbb{R}^m$, and projection onto $\mathbb{R}^m$ is the identity map.
\end{proof}

\end{document}